\documentclass[12pt,draftclsnofoot,onecolumn]{IEEEtran}

\usepackage{amsmath,amssymb}
\usepackage{pdflscape}

\usepackage{graphicx,cite,bm,amsthm,enumerate,cases,mathtools,accents}
\renewcommand{\(}{\left(}
\renewcommand{\)}{\right)}
\renewcommand{\[}{\left[}
\renewcommand{\]}{\right]}

\newcommand{\E}{\mathbf{E}}

\newcommand{\D}{\mathbf{D}}

\newcommand{\x}{\mathbf{x}}

\newcommand{\I}{\mathbf{I}}
\newcommand{\J}{\mathbf{J}}
\newcommand{\C}{\mathbf{C}}

\newcommand{\A}{\mathbf{A}}

\newcommand{\M}{\mathbf{M}}

\newcommand{\B}{\mathbf{B}}

\newcommand{\Tr}[1]{{\rm{Tr}}\left(#1\right)}

\newcommand{\End}[1]{{\rm{End}}}

\renewcommand{\log}[1]{{\rm{log}}#1}

\renewcommand{\det}[1]{\left|#1\right|}

\newcommand{\mean}[1]{{\rm{mean}}\(#1\)}

\newtheorem{lemma}{Lemma}
\newtheorem{definition}{Definition}
\newtheorem{theorem}{Theorem}

\newtheorem{corollary}{Corollary}

\newtheorem{rem}{Remark}

\newcommand{\norm}[1]{\left\lVert#1\right\rVert}

\usepackage[ruled,vlined,resetcount]{algorithm2e}
\usepackage{algorithmic,subfigure}
\usepackage{fontenc,amsfonts}
\usepackage{inputenc,relsize}
\usepackage[square,sort,compress,comma,numbers]{natbib}
\usepackage{epstopdf,pst-node}
\usepackage{siunitx}

\usepackage{ytableau,tikz,varwidth}
\usetikzlibrary{calc}

\usepackage{pgfplots}
\usepackage{verbatim}
\usepackage{smartdiagram}
\usesmartdiagramlibrary{additions}

\newcommand{\ubar}[1]{\underaccent{\bar}{#1}}

\usepackage{pgf}

\SetCommentSty{mycommfont}

\DeclarePairedDelimiter\floor{\lfloor}{\rfloor}

\begin{document}
\title{Region Detection in Markov Random Fields: \\ Gaussian Case}

\author{Ilya Soloveychik and Vahid Tarokh, \\ John A. Paulson School of Engineering and Applied Sciences, \\ Harvard University
\thanks{This work was supported by the Fulbright Foundation and Office of Navy Research grant N00014-17-1-2075.}
}
\maketitle

\begin{abstract}
We consider the problem of model selection in Gaussian Markov fields in the sample deficient scenario. The benchmark information-theoretic results in the case of $d$-regular graphs require the number of samples to be at least proportional to the logarithm of the number of vertices to allow consistent graph recovery. When the number of samples is less than this amount, reliable detection of all edges is impossible. In many applications, it is more important to learn the distribution of the edge (coupling) parameters over the network than the specific locations of the edges. Assuming that the entire graph can be partitioned into a number of spatial regions with similar edge parameters and reasonably regular boundaries, we develop new information-theoretic sample complexity bounds and show that a bounded number of samples can be sufficient to consistently recover these regions. Finally, we introduce and analyze an efficient region growing algorithm capable of recovering the regions with high accuracy. We show that it is consistent and demonstrate its performance benefits in synthetic simulations.
\end{abstract}


\begin{IEEEkeywords}
Model selection, Markov random fields, Gaussian graphical models, Fano's inequality, enumeration of polyominoes.
\end{IEEEkeywords}

\section{Introduction}

\subsection{Learning Markov Random Fields}
Markov random fields, or undirected probabilistic graphical models, provide a structured representation of the joint distributions of families of random variables. A Markov random field is an association of a set of random variables with the vertices of a graph, where the missing edges describe conditional independence properties among the variables \cite{lauritzen1996graphical}. It was shown by Hammersley and Clifford in their unpublished work \cite{lauritzen1996graphical} that the joint probability distribution specified by such a model factorizes according to the underlying graph. The practical importance of Markov random field is hard to overestimate. They have been applied to a large number of areas including bioinformatics, social science, control theory, civil engineering, political science, epidemiology, image processing, marketing analysis, and many others. For instance, a graphical model may be used to represent friendships between people in a social network \cite{farasat2015probabilistic} or links between organisms with the propensity to spread an infectious disease \cite{knorr1998modelling}. The increased availability of large-scale network data created every day by traditional and social media, sensors, mobile devices and social infrastructure provides rich opportunities and unique challenges for the analysis, prediction, and summarization making the development of novel large network analysis techniques absolutely necessary.

Given the graph structure, the most common computational tasks include calculating marginals, partition function, maximum a posteriori assignments, sampling from the distribution, and other questions of statistical inference. On the other hand, in many applications estimating the unknown edge structure of the underlying graph, also known as \textit{model selection} or \textit{inverse problem}, has attracted a great deal of attention. Naturally, both problems are essentially challenging especially in high dimensional scenarios and are known to be NP-hard for general models \cite{karger2001learning,bogdanov2008complexity}.

In model selection, the naive approach of searching exhaustively over the space of all graphs is computationally intractable, since there are as many as $2^{{p \choose 2}}$ distinct graphs over $p$ vertices, therefore, prior knowledge on the graph structure is used to make the size of the family of models tractable. A variety of methods have been proposed to address this problem. One of the first works in this direction was performed by Chow and Liu \cite{chow1968approximating}, who showed that if the underlying graph is known to be a tree, the model selection problem reduces to a maximum-weight spanning tree problem. Other models considered in the literature include sparse networks with bounded degrees of the vertices \cite{bresler2015efficiently, santhanam2012information}, walk-summable and locally separable graphs \cite{anandkumar2012high}, thresholding methods \cite{bresler2008reconstruction}, $\ell_1$-based relaxations \cite{meinshausen2006high, ravikumar2011high, friedman2008sparse, yuan2007model}, methods based on penalized pseudo-likelihood \cite{ji1996consistent}, and many others.


Most of the papers listed in the previous paragraph consider the graph selection problem in the high-dimensional setting, meaning that the number of samples $n$ is comparable to or even less than the dimension $p$ of the parameter space. One of the main focuses of these works consists in deriving tight information-theoretic lower bounds on the sample complexity, or in other words the necessary number of independent snapshots of a network that would allow reliable recovery of its connectivity structure. The benchmark results \cite{santhanam2012information, anandkumar2012high} for the Ising and Gaussian models claim that as a function of the model parameter $p$, the number of measurements
\begin{equation}
\label{eq:n_prop_log}
n = c\, \log\, p
\end{equation}
is required to make the learning possible. However, in many real wold scenarios even this moderate dependence may be prohibitively demanding. One of the reasons for that is the large value of the constant of proportionality $c$ between $n$ and $\log\, p$ in (\ref{eq:n_prop_log}). Quite often, this constant does not receive much attention during the analysis, however, in practice its value becomes critical. Moreover, for quite a wide range of dimensions $p$, this constant may be so large that the required number of sample $n$ will be essentially greater than the dimension. The aforementioned issues call for the development of new techniques that would allow to decrease the sample complexity beyond the logarithmic scaling.

In the sample deficient regime, all the parameters of the problem cannot be estimated reliably, therefore, the goal of the learning process must be reconsidered. This may be otherwise stated as a necessity to introduce more structure into the problem. Such structure can arise from the physical properties of the system, its spatial design, or can be reduced to an off-line precomputation of parameters. In this work, we mostly rely on the spatial structure of the real networks. Physical networks are naturally embedded into Euclidean spaces and the induced spatial structure dictates the sparsity pattern of the underlying graph as well as suggests that the parameters of the network change slowly for close vertices and can be even assumed to be constant inside small regions. Hence, quite often a graph can be viewed as a union of disjoint regions inside which the pattern of the interaction between vertices is approximately constant. This makes the detection of specific edges inside the regions redundant, while emphasizing the importance of quick adjustment and detection of the region boundaries.

One of the practically important examples is the brain activity analysis in animals and humans \textit{in vivo}. Whole-brain functional imaging at cellular resolution allows to investigate the functional distinctions between different regions of the brain and between healthy and damaged brain tissue based on the connectivity properties of neighboring neurons. Recently, a group of neuroscientists from HHMI's Janelia Research \cite{ahrens2013whole, dunn2016brain} developed a revolutionary technique capable of simultaneously capturing the activity of all single neurons of the entire brain of an alive fish. However, the number of available snapshots is way less than the necessary sample complexity conditions required by the classical graphical model selection approaches. This makes region detection techniques avoiding learning all graph edges extremely appealing.



\subsection{Two Dimensional Change Detection}
Another approach to the problem considered in this work can be formulated as the two dimensional change/region detection, which can be viewed as a generalization of the one dimensional change detection \cite{gustafsson2000adaptive}. Change points are abrupt variations in time series data that may represent transitions that occur between different states. Change point detection is the problem of finding such abrupt changes, description of their nature, and quantification given the data. Identification of change points is extremely useful in modeling and prediction of time series and is found in applications such as medical condition monitoring, meteorological analysis and prediction, speech and image analysis, human activity analysis, aerospace, finance, business, and entertainment; see \cite{aminikhanghahi2017survey} for a detailed survey. 

If the one dimensional change point detection is the problem of splitting a sequence of measurements into intervals representing different epochs, the two dimensional change \text{region} detection is the problem of splitting a spatial domain with data distributed over it into a disjoint union of open regions with similar properties. Specified to graphical models embedded into a metric space, that would mean partitioning a graph into a set of disjoint connected subgraphs based on their structural properties (under additional assumptions discussed later). Remarkably, already in the two dimensional case, the analysis is totally different due to the way the data is arranged. Because of the lack of the time axis, the measurements are not arranged sequentially anymore. The notions of cause and effect lose their simple meanings and make the analysis more involved. This situation may be compared to the difference between one and two dimensional Ising models \cite{mccoy2014two}.

\subsection{Our Contribution}
In this paper, we focus on Gaussian graphical models. Due to Hammersley-Clifford theorem, the problem of model selection in Gaussian scenario is equivalent to learning the sparsity structure of the precision matrix (interchangeably named inverse covariance, potential or information matrix) \cite{lauritzen1996graphical}. This fact combined with various types of statistical estimators suited to high dimensions has been exploited by many authors to recover the structure of Gaussian graphical models when the edge sets are sparse \cite{yuan2007model, anandkumar2012high}. 

Our study is motivated by the applications in which the network is embedded into a Euclidean space and the number of samples is scarce, significantly distinguishing it from the works listed above. The contribution of this work is four-fold. First, we introduce a novel framework for model selection in two dimensional change detection problem in sample deficient scenario. Essentially, we replace the common approach consisting in estimation of the entire graph by learning the critical structures, such as homogeneous regions of the network and their boundaries. Second, we develop information-theoretic bounds on the sample complexity of any algorithm addressing this problem and compare them with the standard model selection results for the full graph recovery. In particular we demonstrate that in our setup consistent recovery of the structure is possible for large $p$ even with bounded number of samples. To derive such an information-theoretic lower bound we rely on a common approach in model selection literature based on the application of Fano's inequality. Remarkably, this leads to a very challenging problem of enumeration of polygons and \textit{polyominoes} that represent a special type of lattice polygons defined below. In order to count them we resort to very recent and deep results emerged from the theory of random integer partitions and their large deviations. Third, we suggest a simple but efficient Greedy Region Detection ({\fontfamily{lmss}\selectfont GRID}) algorithm capable of reliably learning the regions of the network with finite number of samples. We rigorously analyze the {\fontfamily{lmss}\selectfont GRED} algorithm and examine its performance guaranties against the obtained bounds. Forth, we demonstrate its benefits in synthetic simulations.

\subsection{Image Segmentation and Community Detection}
It is important to emphasize a number of differences between our setup and seemingly similar problems: image segmentation and community detection.

In image segmentation \cite{haralick1985image}, one is also interested in splitting a two dimensional domain into regions based on the homogeneity properties of the latter (sometimes also referred to as edge detection). However, there exists a number of significant discrepancies making this task very different from the problem we address. First, in image segmentation one is usually interested in detection of the boundaries based on similarity of the pixel values and not on the statistical relations between them. Second, the ground truth picture is usually available, which totally changes the approach to the region detection. Third, technically images are always discretized using fixed grid (usually, regular square lattice), which also facilitates the segmentation process. Finally, quite often human interaction with the algorithm is required to perform the segmentation. This is also possible since the ground truth image is presented to the expert and the question is related to its perception.


The main goal of community detection is to partition the vertices of a graph into clusters based on relations between them (we refer the reader to the survey \cite{fortunato2010community} for more details). Here again, our scenario is quite different. First, in most community detection frameworks the clustering is based on the connectivity properties of the network, whereas in our setup the entire graph can be regular. Second, unlike our approach, the community detection techniques usually do not exploit the spatial information. Those works that do utilize the local structure and geometrical properties of the graph \cite{eckmann2002curvature, clauset2005finding} do not suggest a rigorous framework but rather empirical studies, and often develop essentially global optimization algorithms prohibitively demanding in terms of computational resources. The heuristic techniques suggested by a solid part of the community detection articles are usually not accompanied by performance guaranties and also do not derive any information-theoretic bounds that could serve as their performance benchmark.

The rest of the paper is organized as follows. We introduce the setup and notation in Section \ref{sec:notation}. In Section \ref{sec:assm}, we discuss the properties of the networks under consideration and the assumptions imposed on the model. We also formulate the problem addressed in this article and relate it to the existing works. In Section \ref{sec:inf_theor_bounds} we state the information-theoretic lower bounds on the sample complexity. Section \ref{sec:param_est} focuses on the parameter estimation issue and develops concentration bounds for it. Based on these results, in Section \ref{sec:algo} we present our algorithm for the two dimensional change detection and analyze its sample complexity. The results of numerical simulations illustrating our theoretical findings are given in Section \ref{sec:num_res}. We make our conclusions in Section \ref{sec:concl}. The details of the proofs can be found in Appendices \ref{app:graph_detect_inf_thoer_bound}-\ref{app:algo}.

\section{Setup and Notation}
\label{sec:notation}
\subsection{Gaussian Markov Fields} Let $G = (V,\E)$ be an undirected graph with the vertex set $V=[p]$, where $[p]=\{1,\dots,p\}$, and the binary adjacency matrix $\E$. For a vertex $i$, we denote by $\mathcal{N}(i) \subset V$ the set of its neighbors. In addition, to each vertex $i \in V$ we associate a real random variable $x_i$ and denote the probability density function of the joint distribution of $\x = \(x_1,\dots,x_p\)^\top$ by $f(x)$. We say that $f(x)$ satisfies local Markov property w.r.t. graph $G$ if
\begin{equation}
f(x_i|\x_{\mathcal{N}(i)}) = f(x_i|\x_{V\backslash i}),\quad \forall i \in V,
\end{equation}
where $\x_A = \{x_i|i\in A \subset V\}$. More generally, we say that $\x$ satisfies the global Markov property, if for all disjoint sets $A,B \subset V$, we have
\begin{equation}
f(\x_A, \x_B|\x_S) = f(\x_A|\x_S) f(\x_B|\x_S),
\end{equation}
where $S$ is a \textit{separator set} between $A$ and $B$, meaning that the removal of nodes in $S$ partitions $V$ in such a way that $A$ and $B$ belong to distinct components.

In this work we focus on a Gaussian graphical model (random Markov field) over this graph, meaning that the joint distribution of $\x = \(x_1,\dots,x_p\)^\top$ is normal,
\begin{equation}
p(\x;\J) = p(x_1,\dots,x_p;\J) = \frac{1}{\sqrt{2\pi |\J^{-1}|}}e^{-\frac{1}{2} \x^\top \J \x},
\end{equation}
where $\J =\{\J_{ij}\}_{i,j=1}^p = \bm\Sigma^{-1}$ is the precision (inverse covariance, potential, information) matrix of the population. It can be easily shown \cite{lauritzen1996graphical} that $\J$ has zeros in the entries $(i,j)$ corresponding to the missing edges in $\E$ and is non-zero otherwise. Both $\J$ and $\bm\Sigma$ are assumed to be positive definite, making the distribution non-degenerate. The off-diagonal non-zero elements $\J_{ij}$ are referred to as \textit{coupling} or \textit{edge} parameters between nodes $i$ and $j$ and in this work are assumed to be positive for simplicity\footnote{The generalization to both positive and negative coupling parameters is straightforward.},
\begin{equation}
\forall i \neq j \in V \colon \quad \J_{ij} \neq 0, \iff \E_{ij} = 1,\;\; \text{and}\;\; \J_{ij} > 0.
\end{equation}
Non-degenerate Gaussian graphical models satisfy both the local and global Markov properties, which are equivalent in this case \cite{lauritzen1996graphical}.

\subsection{Change Detection}
Consider a $d$-regular graph $G=G_p$ on $p$ vertices. We assume that the vertices of $G$ can be partitioned (based on their spatial proximity as discussed later) into a number $S$ of disjoint subsets $V = \bigcup_{s=1}^S V_s$ containing $p_s,\; s=1,\dots,S$ vertices each and referred to as \textit{regions}. For every two connected vertices from the same region $V_s$, the coupling parameter $\J_{ij}$ associated with them depends only on the class label $\J_{ij} = \theta_s,\; s=1,\dots,S$. For simplicity we assume that any edge connecting vertices $i \in V_{s_1}$ and $j \in V_{s_2}$ from two different regions $s_1 \neq s_2$ has the average $\J_{ij} = \theta_{s_1} + \theta_{s_2}$ coupling parameter between these regions.\footnote{As can be seen from Sections \ref{sec:inf_theor_bounds} and \ref{sec:algo}, the role played by these parameters vanishes in the limit $p \to \infty$ and the average values assumed here can be in fact replaced by any numbers in the allowed range defined below.} We also assume that the variances of all variables $\J_{ii}$ inside every region are constant. Therefore, (after reordering the vertices, if necessary) the precision matrix can be decomposed in the following way,
\begin{equation}
\J = \J_1 + \J_2,
\end{equation}
where
\begin{equation}
\J_1 = \bigoplus_s \left.\J\right|_{G_s} = \bigoplus_s \J_s = \bigoplus_s \(\kappa_s\I_{p_s} + \theta_s \E_s\),
\end{equation}
$\theta_s$-s are the coupling coefficients, $\kappa_s$-s are the variances in the subgraphs $G_s$, $\I_r$ is the $r \times r$ identity matrix, and the matrix $\J_2$ corresponding to the cross-region edges can have non-zeros only outside the $V_s \times V_s$ blocks (of sizes $p_s \times p_s$). For simplicity and without much loss of generality, it is common in the graphical model selection literature to assume that all $\kappa_s = 1$ \cite{anandkumar2012high}. Thus, we obtain the following model
\begin{equation}
\J_1 = \I_p + \bigoplus_s \theta_s \E_s.
\end{equation}
Below we sometimes omit the subscripts of the identity matrices if the dimensions are clear from the context. Denote
\begin{equation}
\ubar{\theta} = \min_s \theta_s,\quad \bar{\theta} = \max_s \theta_s.
\end{equation}
Intuitively, models with relatively too small or too large coupling parameters are harder to learn than those with comparable parameter values and they require more samples for consistent graph recovery. Indeed, very high couplings create long-ranging correlations which are hard to treat, as explained in the Introduction. On the other hand, the small ones are hard to distinguish from zeros. Therefore, the values $\ubar{\theta}$ and $\bar{\theta}$ will play significant role in the sample complexity bounds.

Below, whenever a set of quantities $\zeta_s$ indexed by $s=1,\dots,S$ is considered, we denote
\begin{equation}
\ubar{\zeta} = \min_s \zeta_s,\quad \bar{\zeta} = \max_s \zeta_s.
\end{equation}

To enable a rigorous study of high-dimensional distributions, it is customary to let the model parameter $p$ grow to infinity. Together with $p$, in our setting all $p_s$ increase such that
\begin{equation}
\label{eq:p_to_inf}
p_s = \nu_s p,\;\; s=1,\dots,S,
\end{equation}
where $\nu_s,\; s=1,\dots,S$ are constants\footnote{Formally, we have to take the integer part of the right-hand side in (\ref{eq:p_to_inf}). Here and below we omit the integer part brackets to simplify the notation.}. The model formulated in (\ref{eq:p_to_inf}) implies that the number of regions $S$ is constant. This condition is not restrictive and can be easily relaxed. Below we mention how to adjust (\ref{eq:p_to_inf}) to incorporate the case of a growing number of regions $S\to \infty$.

In this work, we address the question of recovering the structure of the regions given a small number of i.i.d.\ (independent and identically distributed) samples from the distribution. In order to formulate the problem precisely, we need to introduce additional assumptions.


\section{Structural Assumptions}
\label{sec:assm}
\subsection{State-of-the-Art}
As stated in the Introduction, due to the enormous size of the set of graphs on $p$ vertices, structural assumptions are usually made by researches to make the learning feasible, especially in the high-dimensional regime where the number of samples is not enough to consistently estimate all the degrees of freedom. Probably the earliest paper taking advantage of this approach was the seminal work of Chow and Liu \cite{chow1968approximating}, where the authors established that the structure estimation in tree models reduces to the maximum weight spanning tree problem. For graphs with loops the problem is much more challenging for two reasons: 1) a node and its neighbor can be marginally independent due to indirect path effects, and 2) this difficulty is amplified by the presence of long-range correlations meaning that distant vertices can be more correlated than the close ones. So far, there have not been proposed a complete description of graphs for which structure estimation is possible, however, a number of methods allowing model selection in graphs with structure richer than trees have been suggested. Among them are such models as polytrees \cite{dasgupta1999learning}, hypertrees \cite{srebro2001maximum}, graphs with few short cycles, \cite{anandkumar2012high}, general sparse Ising models \cite{bresler2015efficiently}, graphs with large girth and bounded degree \cite{netrapalli2010greedy}, and many others. 


Except for the sparsity, successful structure estimation also relies on certain assumptions on the parameters of the model, and these are often tied to the specific algorithms. Among various assumptions of this type, the Correlation Decay Property (CDP) stands out. Informally, a graphical model is said to have the CDP if any two variables $x_i$ and $x_j$ are asymptotically independent as the graph distance between $i$ and $j$ increases. Most of the existing model selection procedures require CDP explicitly \cite{montanari2009graphical}, the rest often do so indirectly through different assumptions on the model parameters and are also likely to require the CDP (we refer the reader to the survey \cite{gamarnik2013correlation} as well as e.g. \cite{dobrushin1970prescribing}). For example, the authors of \cite{anandkumar2012high} require a Gaussian Markov field to be $\alpha$-walk summable, as introduced and analyzed by \cite{malioutov2006walk}. The property of $\alpha$-walk summability essentially means that the spectral norm of the matrix consisting of the element-wise absolute values of the partial correlations is bounded by $\alpha$. Roughly speaking, this condition guarantees invertibility of the precision matrix (or in other words existence and non-degeneracy of the covariance matrix of the population). As can be traced from \cite{anandkumar2012high}, for example in ferromagnetic models with the vertex degrees tightly concentrated around a fixed value, the $\alpha$-walk summability is almost equivalent to an upper bound on the coupling parameters and, therefore, is a close relative of the CDP. Another example is the algorithm introduced by \cite{ravikumar2010high}, which is shown to work under certain incoherence conditions that seem distinct from the CDP, however, \cite{montanari2009graphical} established through a careful analysis that the algorithm fails for simple families of certain Markov random fields (ferromagnetic Ising models) without the CDP. In general, some assumptions that involve incoherence conditions are often hard to interpret as well as verify \cite{meinshausen2006high, ravikumar2011high}. It is also worth mentioning that usually, in addition to upper limits, the correlations between neighboring variables are supposed to be bounded away from zero (as is true for the ferromagnetic Ising model in the high temperature regime) to make the family of models identifiable.


\subsection{Two Dimensional Change Detection}
Following the above discussion, to guarantee efficient, reliable, and consistent recovery of the model, we make a number of assumptions. These can be roughly partitioned into the following two groups:
\begin{itemize}
\item Geometric or spatial assumptions induced by the geometry of the surrounding Euclidean space. This type of assumptions includes
\begin{itemize}
\item global structure, reflecting the way the whole graph is partitioned into areas and the properties of the boundaries between them, and
\item local connectivity properties for spatially close vertices; this assumption regulates the level of sparsity of the graph.
\end{itemize}
\item Parametric assumptions, that are mostly technical and allow reliable structure recovery.
\end{itemize}


\begin{figure}
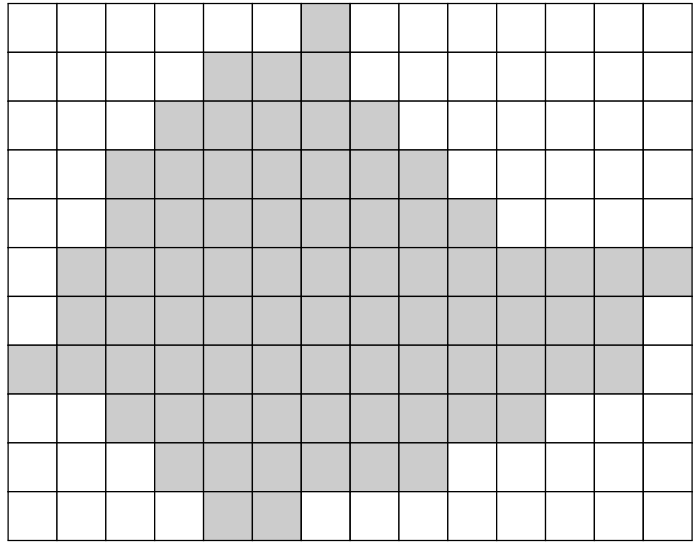

\centering
 \begin{equation*}
 \ytableausetup{textmode}
\begin{ytableau}
*(black!0) & *(black!0) & *(black!0) & *(black!0) & *(black!0) & *(black!0) & *(black!20) & *(black!0) & *(black!0) & *(black!0) & *(black!0) & *(black!0) & *(black!0) & *(black!0) \\
*(black!0) & *(black!0) & *(black!0) & *(black!0) & *(black!20) & *(black!20) & *(black!20) & *(black!0) & *(black!0) & *(black!0) & *(black!0) & *(black!0) & *(black!0) & *(black!0) \\
*(black!0) & *(black!0) & *(black!0) & *(black!20) & *(black!20) & *(black!20) & *(black!20) & *(black!20) & *(black!0) & *(black!0) & *(black!0) & *(black!0) & *(black!0) & *(black!0) \\
*(black!0) & *(black!0) & *(black!20) & *(black!20) & *(black!20) & *(black!20) & *(black!20) & *(black!20) & *(black!20) & *(black!0) & *(black!0) & *(black!0) & *(black!0) & *(black!0) \\
*(black!0) & *(black!0) & *(black!20) & *(black!20) & *(black!20) & *(black!20) & *(black!20) & *(black!20) & *(black!20) & *(black!20) & *(black!0) & *(black!0) & *(black!0) & *(black!0) \\
*(black!0) & *(black!20) & *(black!20) & *(black!20) & *(black!20) & *(black!20) & *(black!20) & *(black!20) & *(black!20) & *(black!20) & *(black!20) & *(black!20) & *(black!20) & *(black!20) \\
*(black!0) & *(black!20) & *(black!20) & *(black!20) & *(black!20) & *(black!20) & *(black!20) & *(black!20) & *(black!20) & *(black!20) & *(black!20) & *(black!20) & *(black!20) & *(black!0) \\
*(black!20) & *(black!20) & *(black!20) & *(black!20) & *(black!20) & *(black!20) & *(black!20) & *(black!20) & *(black!20) & *(black!20) & *(black!20) & *(black!20) & *(black!20) & *(black!0) \\
*(black!0) & *(black!0) & *(black!20) & *(black!20) & *(black!20) & *(black!20) & *(black!20) & *(black!20) & *(black!20) & *(black!20) & *(black!20) & *(black!0) & *(black!0) & *(black!0) \\
*(black!0) & *(black!0) & *(black!0) & *(black!20) & *(black!20) & *(black!20) & *(black!20) & *(black!20) & *(black!20) & *(black!0) & *(black!0) & *(black!0) & *(black!0) & *(black!0) \\
*(black!0) & *(black!0) & *(black!0) & *(black!0) & *(black!20) & *(black!20) & *(black!0) & *(black!0) & *(black!0) & *(black!0) & *(black!0) & *(black!0) & *(black!0) & *(black!0) \\
\end{ytableau}
\end{equation*}
\caption{\small A convex polyomino.}
\label{fig:conv_polyomino}
\end{figure}

\subsection{Model Classes}
\label{sec:model_classes}
When a network is embedded into a Euclidean space, its connectivity properties are determined by the ambient physical space, in particular its underlying graph and coupling parameters are influenced by the specific way it is deployed. In this work, we assume that the graph $G$ is embedded into a two dimensional Euclidean space $\mathbb{R}^2$ with a fixed orthonormal basis, standard scalar product $\langle \cdot,\cdot\rangle$ and the norm $\norm{\cdot}$ induced by it. To simplify the notation, we assume that the graph vertices come together with their coordinates and write this shortly as $V \subset \mathbb{R}^2$.

We deal with graphs which are discrete objects, and even when embedded into Euclidean spaces do not naturally possess boundaries. In the next section we rigorously introduce the notion of a region boundary and its properties in our setup. However, before doing that we need to introduce the family of admissible models we focus on, which is in our case the family of regions on the plane we want to detect.

In order to make our family of models finite for every $p$, we consider the graph on a square lattice and assume the discretized regions to be represented by convex \textit{polyominoes}. A polyomino on a square lattice is a union of elementary lattice cells which must be joined at their sides, and not just at nodes \cite{guttmann2009polygons}, such as e.g. the cells colored gray in Figure \ref{fig:conv_polyomino}. The boundary of a polyomino is a lattice polygon with only vertical or horizontal sides, and therefore there is a one-to one correspondence between such polygons and polyominoes. A polyomino is said to be \textit{column-convex} in a given lattice direction if all the cells along any line in that direction are connected through cells in the same line. A polyomino is \textit{convex} if it is column-convex in both horizontal and vertical lattice directions. For brevity, we will use the acronym CPMs for the Convex PolyoMinoes. In Appendix \ref{app:enum_poly} we discuss more details about the geometrical properties of CPMs. 
It is important to emphasize that the square lattice is chosen for concreteness and convenience of notation. In fact, any other tilling of the plane can be used instead, e.g. triangular, hexagonal (honeycomb) lattices, etc.

\begin{rem}
It is important to note that the convexity requirement is not necessary for the problem formulation and to guarantee a consistent model selection. However, this simple assumption implies a long list of useful properties which make the model selection possible. The convexity condition can be relaxed, but that would require introduction of specific technical conditions. Our main goal in this article is to expose the idea of the two dimensional detection and its rigorous treatment. Therefore, to keep the text concise and easily accessible we chose to stick to the convexity requirement. Moreover, among the two versions of the {\fontfamily{lmss}\selectfont GRED} algorithm provided below, the basic one does not require the convexity assumption.
\end{rem}

One of the main ingredients of Fano's inequality used later to derive the information-theoretic sample complexity bounds is the cardinality of the set of models at hand. Therefore, we need to be able to count the number of CPMs with different properties depending on the prior knowledge available in practice. For example, the number of CPMs with fixed perimeter, fixed area, or both fixed perimeter and area, or the number of CPMs in some vicinity of a specific curve, etc. As discussed in Appendix \ref{app:enum_poly} in more detail, enumeration of CPMs is a very involved task that became feasible only recently due to some deep breakthroughs in mathematics related to the theory of random integer partitions. 

It is important to mention that polyominoes is not the only possible way to discretize the regions. Another possibility is to consider e.g. the family of convex lattice polygons. Enumeration of convex polygons is also achieved through the application of LDP (see Appendix \ref{app:enum_cpgs} for more details). However, the detection algorithm requires more technical details. Due to limited space, we devote a separate article to the treatment of this setup \cite{soloveychik2018polygonal}.

\subsection{Spatial Structure}
\label{sec:graph_gener}
\noindent\textbf{[A1] Regularity of the Regions and their Boundaries.} To rigorously introduce the notion of a boundary and its properties in our setup, let us assume that the graph $G$ is generated in the following manner. First, a two dimensional lattice is constructed and a connected region $\mathcal{F}$ is chosen on it. After that, $\mathcal{F}$ is cut along piece-wise linear curves passing through the lattice nodes into $S$ connected subregions $\mathcal{F}_s,\; s=1,\dots,S$ of areas $A_s$ and having boundaries $\partial \mathcal{F}_s$ of lengths $l_s = |\partial \mathcal{F}_s|$, correspondingly. The boundaries are lattice polygons with \textit{nodes} and \textit{sides} connecting them (to distinguish from \textit{vertices} and \textit{edges} of the graph we are building). Later we will fill the obtained regions with the vertices of $G$, however, we start with the discussion on how the regions are constructed. For concreteness and without loss of generality of our approach, here we assume the sides of the polygons to be the lattice sides, and the polygons to be CPMs. Generalization to CPGs and other classes of polygons is achieved in a similar way.

To motivate the assumptions that we make below, let us first outline the geometric intuition behind them for the case of smooth boundaries. Our ultimate goal is to detect the regions from the measurements over graphs whose coupling parameters vary between the regions. To this end, we want these regions to have regular shapes.
\begin{itemize}
\item When the boundary of a two dimensional region is reasonably regular, it is natural to assume that its length is proportional to the square root of the embraced area. Intuitively, this condition is justified by the classical isoperimetric inequality.
\item One the other hand, to guarantee local regularity of the boundaries, we assume their radii of curvature to be bounded from below.
\end{itemize}

Next, we formulate the discrete analogs of these requirements for the lattice regions. The first assumption on the length of the boundary does not change when we pass to the discrete case (the only distinction would be in the constant of proportionality between the length of the boundary and the square root of the area, as the isoperimetric inequality, Lemma \ref{lem:isoper_ineq} from Appendix \ref{app:area_detect_proof}, suggests). Formally, assume that
\begin{equation}
\label{eq:def_length_area}
l_s = \beta_s A_s^{1/2},\;\; s=1,\dots,S,
\end{equation}
where $\beta_s$-s are constants. The discrete isoperimetric inequality for the square lattice, Lemma \ref{lem:isoper_ineq} from Appendix \ref{app:area_detect_proof}, shows that necessarily all $\beta_s \geqslant 4$ and in the case of equality, the polygon is a square. To avoid such scarce family of models, we require
\begin{equation}
\ubar{\beta} = \min_s \beta_s > 4.
\end{equation}



Formulation of the discrete analog of the second assumption requires more work. There is no universally accepted way of measuring curvature of piece-wise linear curves and different approaches exist. Probably the most popular and natural definition proposed by the authors of \cite{borrelli2003angular} measures the radius of curvature at any node $u$ of a polygon as
\begin{equation}
\label{eq:curv_assm}
\tilde{r}(u) = \frac{\pi_{l}(u)+\pi_{r}(u)}{2\alpha(u)},
\end{equation}
where $\pi_{l}(u)$ and $\pi_{r}(u)$ are the lengths of respectively the left and right sides of $\partial F_s$ incident to $u$ and $\alpha(u)$ is the angle between them. Note that in the case of CPMs, the angle between the sides can take only two values: $\frac{\pi}{2}$ or $\frac{3\pi}{2}$. Let us explain why this definition is not suitable in our setting. Assume $\mathcal{F}_s$ is a rectangle of perimeter $l_s$, then formula (\ref{eq:curv_assm}) suggests that the radius of curvature at any node is $\frac{l_s}{2\pi}$. This measure does not distinguish between the case of a square and of a narrow strip. However, from the point of view of detecting these two shapes, a square is easier to learn. This example motivated us to introduce a different notion of curvature for lattice polygons, namely
\begin{equation}
\label{eq:curv_assm_mod}
r(u) = \frac{\min\[\pi_{l}(u),\pi_{r}(u)\]}{\alpha(u)},
\end{equation}
which better captures the hardness of detection when the boundary contains short sides. Overall, we see that essentially this criterion of regularity boils down into restricting the minimal length of the sides of the polygons. For technical reasons appearing in the proofs, our formal assumption reads as follows. The lengths of the boundary sides of $\partial F_s,\; s=1,\dots,S$ should be divisible by (are integer multiplies of)
\begin{equation}
\label{eq:def_curv_rad_area}
r = \rho A^{\xi},
\end{equation}
where $A$ is the area of the entire domain, $\rho$ is a constants and $\xi \in \big(0,\frac{1}{2}\big]$. Note that due to (\ref{eq:curv_assm_mod}),
\begin{equation}
\min_{u \in \partial F_s}r(u) \geqslant r.
\end{equation}
In the case of CPM regions, the quantum $r$ of the boundary length suggest that the lattice width should be equal to $r$. This value will be used in the sequel.

After we have determined the shapes of the boundaries of the regions, we construct the graph $G$ starting with its vertices. Let us cover the region $\mathcal{F}$ uniformly with constant areal density\footnote{\label{ftn:1}In fact, any distribution of vertices satisfying (\ref{eq:approx_dens}) and Assumption [A3] stated below will work. Uniform distribution is used to reduce technical details and simplify the notation.} $\eta$ with $p$ vertices and denote those of them inside $\mathcal{F}_s$ by $V_s$ and their number by $p_s=|V_s|$ (vertices fall onto the boundaries with vanishing probability; if that happens the ties are broken arbitrarily). Thus, we can write
\begin{equation}
\label{eq:approx_dens}
p_s = \eta A_s.
\end{equation}
Strictly speaking, we must add to the right-hand side of (\ref{eq:approx_dens}) a term vanishing asymptotically and put rounding brackets, however, below we will omit this term and the brackets and still keep the equality sign to make the notation simpler. 
Figure \ref{fig:fig_intr} shows a piece capturing the segment of the boundary between two CPM regions. 

\begin{figure}[t]
\centering
\includegraphics[width=15cm]{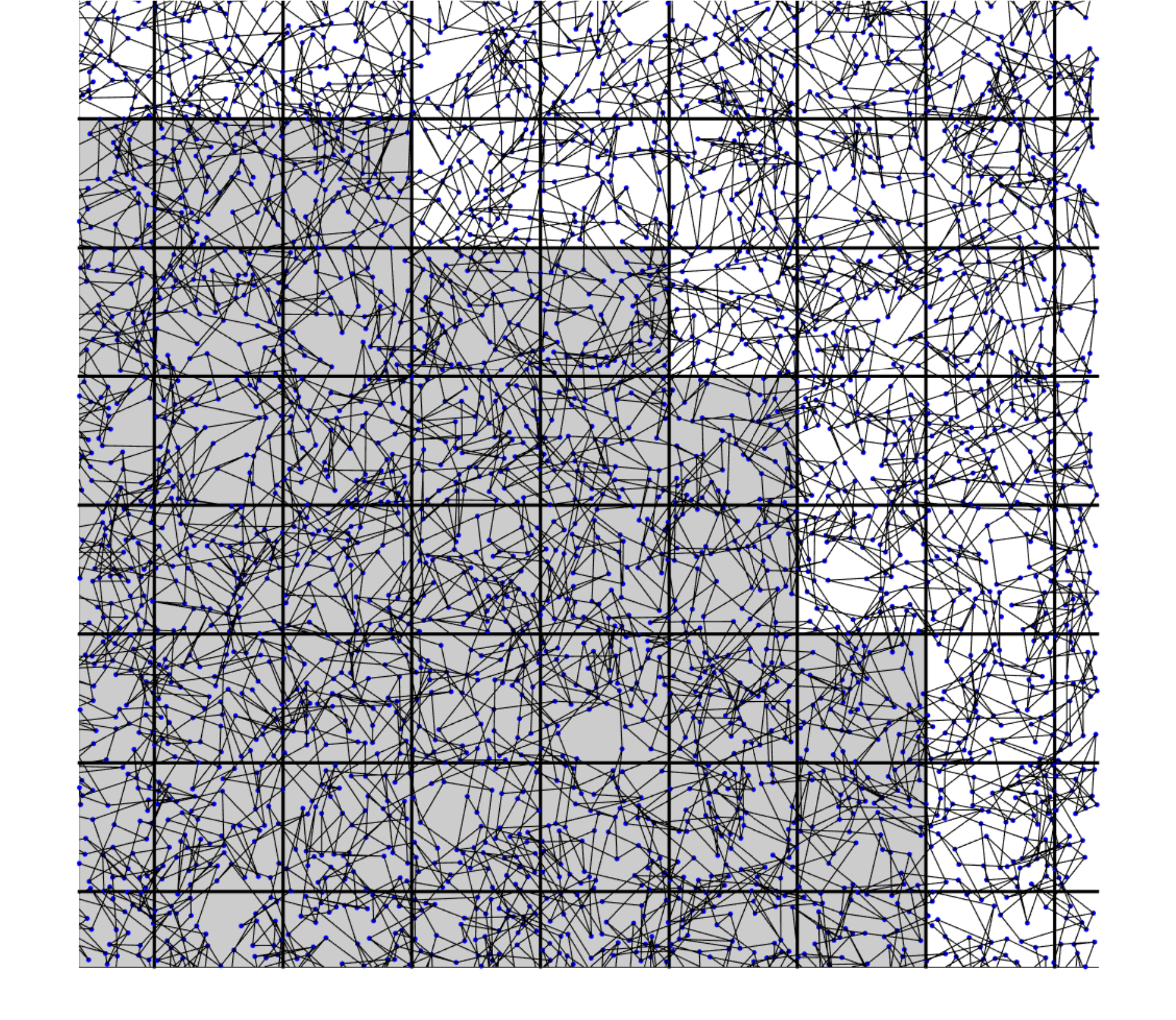}
\caption{\small A part of a graph showing segments of two regions with $4000$ vertices, $d=3$.}
\label{fig:fig_intr}
\end{figure}




As we have already mentioned, the parameters of the patches $\mathcal{F}_s$ (such as the area, boundary lengths, etc.) cannot be chosen independently. In particular, we must make sure that the radius of curvature of the circumscribed rectangle of $\mathcal{F}_s$ is at least $r$,
\begin{equation}
r(\mathcal{F}_s) \geqslant r,
\end{equation}
As explained in Lemma \ref{lem:min_inscr_sq} from Appendix \ref{app:area_detect_proof}, $r(\mathcal{F}_s) \leqslant \frac{2A_s}{l_s}$, therefore, we require
\begin{equation}
\label{eq:ar_l_cond}
\frac{2A_s}{l_s} \geqslant r.
\end{equation}
Condition (\ref{eq:ar_l_cond}) holds automatically when $\xi < \frac{1}{2}$; for $\xi=\frac{1}{2}$ it is equivalent to requiring $\beta_s\rho \leqslant 2$.

\medskip
\noindent\textbf{[A2] $d$-regularity.}
We have formulated the global properties of the graph $G$ related to the disposition of its vertices. Let us proceed to the assumptions on the edges of $G$. Recall that we deal with $d$-regular graphs,
\begin{equation}
|\mathcal{N}(i)| = d,\quad \forall i \in G.
\end{equation}
It is worth mentioning that the $d$-regularity assumption is technical and was made for simplicity of the derivations. In fact, as will be clear from the proofs below, this restriction can be easily relaxed to a condition similar to the one from \cite{anandkumar2012high}, where only the mean degree is fixed and the degrees of the vertices are allowed to vary slightly around that value.


\medskip
\noindent\textbf{[A3] Locality of Connections.} By construction of $G$, the amount of vertices inside every large enough square is proportional to its area. The vertex degrees are fixed and the number of vertices close to the cell boundary is proportional to its length, therefore, the amount of edges crossing the boundary is proportional to the cell perimeter. Denote the set of vertices inside a square with the side length at least $r$ by by $A \subset V$ and $B = V\backslash A$, then we can make the above intuitive reasoning precise by assuming that
\begin{equation}
\label{eq:assm2_eq}
\Tr{\E_{AB}\E_{AB}^\top} = \sum_{i \in A, j \in B} \E_{ij} \ll d\sqrt{|A|}.
\end{equation}

\begin{rem}
Interestingly, the construction described above resembles Random Geometric Graphs (RGGs) \cite{penrose2003random}. However unlike RGGs, we do not require all the vertices inside balls of a specific radius to be connected. In addition, to the best of our knowledge the works on RGGs are usually concerned with the combinatorial properties of the latter such as the sizes of the connected components, percolation effects and similar, and do not focus on probabilistic graphical models over such graphs.
\end{rem}


\subsection{Spectral Structure.} 

\medskip
\noindent\textbf{[A4] Correlation Decay Property.} 
As mentioned in the Introduction, to make the graph structure recovery possible, it is natural to assume the edge parameters to satisfy the CDP. This assumption regulates the influence of the variables on the far vertices and makes the consistent structure learning feasible. Specifically, we assume that
\begin{equation}
d\bar{\theta} < 1.
\end{equation}
Below we will see that in the Gaussian case this inequality ensures invertibility of the precision matrix. Similar approach was also exploited by the authors of \cite{anandkumar2012high}, who instead utilized the notion of $\alpha$-walk summability.

%
%

\subsection{Problem Formulation and Performance Measure}
\label{sec:pr_form}
Recall that unlike the classical model selection in graphical models aiming at learning the edges of the graph, our goal is to only detect the different regions of the graph. For this purpose, we need to define the family of models among which we choose one as the outcome of the detection procedure.

Let us denote by $\mathcal{R}_{p}$ the family of classes of models whose elements $R_{p}\(\mathcal{F}_1,\dots,\mathcal{F}_S\)$ are classes of graphs defined by the position of the regions $\mathcal{F}_1,\dots,\mathcal{F}_S$,
\begin{equation}
\mathcal{R}_p = \{R_{p}\(\mathcal{F}_1,\dots,\mathcal{F}_S\)\}.
\end{equation}
Each class $R_{p} = R_{p}\(\mathcal{F}_1,\dots,\mathcal{F}_S\)$ is equivalently defined by the boundaries of the regions $\partial\mathcal{F}_1,\dots,\partial\mathcal{F}_S$ and contains all the graphs $G_p$ on $p$ vertices that consist of $S$ regions with the same boundaries, coupling parameters $\theta_s$, constants $\beta_s$ and $\rho$. Basically, the graphs from $R_p$ can be characterized by having the same global structure, while differing in the specific local arrangements of the vertices and edges inside each region. In our region detection framework, graps belonging to the same class $R_{p}$ are indistinguishable. Graphs from different classes $R_p'$ and $R_p''$ have different boundaries but may still share the rest of the model parameters.

Given $n$ i.i.d.\ snapshots $\mathcal{X}^n = \{\x_1,\dots,\x_n\} \subset \mathbb{R}^p$ drawn from the product probability density function $\prod_{i=1}^n f(\x_i|G_p)$ with the underlying graph $G_p$, our goal is to detect the class $R_p$ of models containing $G_p$. This task can be equivalently formulated as detecting the boundaries of the regions $\partial F_s$ having different coupling parameters.

To assess the quality of the algorithm suggested below and to be able to compare it to the other existing methods, we need to specify a measure of performance. Most of the existing algorithms \cite{santhanam2012information, bresler2015efficiently, anandkumar2012high} consider the so-called zero-one loss for edges of the detected graph, meaning that they declare a failure once the estimated network differs from the ground truth graph by at least one edge. Since our main goal is the change or region detection, failure to find specific edges or addition of some amount of non-existent edges inside one of the regions is not a critical error. However, since our graphs have essential geometric structure, what is critical is the detection of the boundaries of the regions. Therefore, in this work we use the zero-one loss not for the graph itself but for the detection of the geometric structure of the region boundaries. Below we explain the details.

\section{Lower Sample Complexity Bounds}
\label{sec:inf_theor_bounds}

\subsection{The Main Result}
In this section, we develop lower bounds on the number of samples necessary for the region detection task formulated in Section \ref{sec:pr_form}. In other words, such bounds can be interpreted as the minimal possible number $n$ of i.i.d.\ snapshots (as a function of other parameters of the problem) such that with less than $n$ samples the error-less reconstruction is impossible. 

Assume a learning algorithm is chosen and its output is a class $\widehat{R}_p = \widehat{R}_p\(\mathcal{X}^n\) \in \mathcal{R}_p$, then the probability of error reads as
\begin{equation}
\label{eq:prob_eq}
\mathbb{P}_e^R = \mathbb{P}\[\widehat{R}_p \neq R_p\].
\end{equation}
Note that the probability measure in (\ref{eq:prob_eq}) is taken w.r.t.
\begin{itemize}
\item the measurements sampled from the graphical model with the underlying graph $G_p$,
\item the realization of the $G_p$ from the class $R_p$, and
\item the choice of the class $R_p$ from the family $\mathcal{R}_p$.
\end{itemize}
It is common to approach the lower bounds \cite{santhanam2012information, anandkumar2012high} from the information-theoretic perspective as the source coding (or compression) problem \cite{cover2012elements}. If we treat our problem as reconstruction of the source given the measurements $\mathcal{X}^n$, then the necessary conditions on the sample complexity follow from Fano's inequality (see Lemma \ref{lem:fano_ineq_lemma} from Appendix \ref{app:graph_detect_inf_thoer_bound}).

To compare and emphasize the difference between the entire graph recovery and the region detection, we derive two bounds. The first bound is for the standard entire network detection in the setup of Section \ref{sec:assm}, while the second is for the two dimensional region detection. Assume a full model selection algorithm (e.g. one of \cite{anandkumar2012high, yuan2007model} or any other) is chosen. Denote by $\widehat{G}_p = \widehat{G}_p\(\mathcal{X}^n\)$ the graph selected by it and let
\begin{equation}
\label{eq:prob_eq_g}
\mathbb{P}_e^G = \mathbb{P}\[\widehat{G}_p \neq G_p\]
\end{equation}
be the error of the detection of the true graph $G_p$ from $\mathcal{X}^n$.

\begin{theorem}[Necessary Sample Complexity for the Entire Graph Model Selection] 
\label{th:nec_cond_complete_graph}
Suppose that Assumptions [A1] - [A4] hold and that a graph $G_p \in \bigcup\limits_{\mathcal{F}_1,\dots,\mathcal{F}_S} R_{p}\(\mathcal{F}_1,\dots,\mathcal{F}_S\)$ is chosen uniformly. The number of i.i.d.\ samples $n$ from $G_p$ necessary for $\mathbb{P}_e^G$ to vanish asymptotically is
\begin{equation}
n \geqslant \frac{d\, \log \(\frac{p^{\xi}}{d}\)}{\log\(\frac{(2\pi e)^2}{1-d\bar{\theta}}\)},\quad p \to \infty.
\end{equation}
\end{theorem}
\begin{proof}
The proof can be found in Appendix \ref{app:graph_detect_inf_thoer_bound}.
\end{proof}

As we already mentioned above, we deal with the sample deficient case making the task of the estimation of the entire underlying graph structure unfeasible. Instead, we rely on the structural assumptions to enable consistent recovery of the regions, or equivalently of their boundaries, and therefore, the sample complexity is measured against this goal. 

\begin{theorem}[Necessary Sample Complexity for the Region Detection]
\label{th:nec_cond_region}
Suppose that Assumptions [A1] - [A4] hold and that a graph $G_p \in R_{p}$ is chosen uniformly from the class $R_p$ which is in turn chosen uniformly from the family $\mathcal{R}_p$. The number of i.i.d.\ samples $n$ from $G_p$ necessary for $\mathbb{P}_e^R$ to vanish asymptotically is
\begin{equation}
\label{eq:inf_th_bound_main_text}
n \geqslant \frac{1}{p^{2\xi}}\[\frac{1}{d}\min_s \min\limits_{\partial\mathcal{F}_s \cap \partial \mathcal{F}_t \neq \emptyset} \frac{C(\beta_s)}{\(\frac{\theta_s-\theta_t}{1-d\bar{\theta}}\)^2\beta_s \nu_s^{1/2}}\],\quad p \to \infty,
\end{equation}
where $C(\beta_s)$ is a constant\footnote{In the scenario discussed here the only constraint on the shape of the curves is given by the values of $\beta_s$. In fact, we can consider much more general families of CPMs with restrictions on their perimeter, area, both perimeter and area and many others. The various parameters will only affect the value of the constant $C$, but it will be free of dependence on $p$ and the statement of the theorem will remain the same. The exact value of the constant is given in Appendix \ref{app:area_detect_proof}. Appendix \ref{sec:const_calc} provides an example of calculation.} depending only on $\beta_s$.
\end{theorem}

\begin{proof}
The proof can be found in Appendix \ref{app:area_detect_proof}.
\end{proof}

A number of remarks are in place here. The expression in the square brackets in (\ref{eq:inf_th_bound_main_text}) does not depend on $p$. Therefore, unlike Theorem \ref{th:nec_cond_complete_graph} the result stated in Theorem \ref{th:nec_cond_region} suggests that the necessary number of samples needed for consistent detection decays with the growing dimension. More specifically, for the family of CPM models the rate of decay is proportional to $\frac{1}{p^{2\xi}}$. This in particular means that the necessary number of independent snapshots is bounded from above by a constant not depending on $p$. Another point is that the statement of Theorem \ref{th:nec_cond_region} relies on the model assumption (\ref{eq:p_to_inf}) and essentially considers the case of the fixed number of regions $S$. We should emphasize that the generalization to the growing number of regions is straightforward, as long as the model Assumptions [A1]-[A4] are satisfied.

\subsection{Related Works} 
\label{sec:rel_work}
It is instructive to compare our approach and information-theoretic bounds with other graph learning techniques proposed in the literature. We remind the reader that most of the existing algorithms consider the so-called zero-one loss over all the edges of the graph, and declare a failure once the estimate differs from the ground truth even by one edge. In our notation it is equivalent to demanding that $\mathbb{P}_e^G$ defined in (\ref{eq:prob_eq_g}) vanishes asymptotically. These techniques generally deal with abstract graphs not embedded into Euclidean spaces and lacking additional spatial structure. The benchmark result for the Gaussian Markov random field \cite{anandkumar2012high} claims that the sample complexity scales as
\begin{equation}
n \geqslant \frac{\log\, p}{\ubar{\theta}^2\log\(\frac{1}{1-d\bar{\theta}}\)},\quad p \to \infty.
\end{equation}

When the number of available samples is small, e.g. bounded with the growing dimension $p$ as in our case, even the moderate logarithmic dependence on the dimension is not affordable. Such a restriction comes from the fact that in many modern network applications immediate actions are required upon abrupt changes of the network. In addition, the locations of the sensors often alter too rapidly. These and the fact that the agents carry only limited memory and computational power do not allow to collect and process coherent data in amounts sufficient for precise structure recovery. In addition, the real word data usually can be partitioned into regions inside which the connectivity properties of the variables are similar and the exact structure inside the regions is not very important. Furthermore, the valuable information is attached to the boundaries of the regions, whose detection is the most crucial and challenging task.

\section{Parameter Estimation}
\label{sec:param_est}
Before providing our region detection algorithm in the next section, let us develop an efficient machinery that will allow us to locally estimate the coupling parameters with high precision. 



\subsection{Coupling Parameter Estimation}
The following auxiliary lemma will make the local estimation of the edge parameters possible in the sparse Gaussian setup.
\begin{lemma}[Schur's Complement, \cite{zhang2006schur}]
\label{lem:shur_lemma}
Let $\bm\Sigma = \J^{-1}$ be the population covariance matrix of the centered Gaussian distribution on the vertex set $V$. Denote by $A \subset V$ a subset of vertices and by $B = V \backslash A$ its complement. Partition the matrices as
\begin{equation}
\bm\Sigma = \begin{pmatrix} \bm\Sigma_{A} & \bm\Sigma_{AB} \\ \bm\Sigma_{AB}^\top & \bm\Sigma_{B} \end{pmatrix},\quad 
\J = \begin{pmatrix} \J_{A} & \J_{AB} \\ \J_{AB}^\top & \J_{B} \end{pmatrix},
\end{equation}
then
\begin{equation}
\label{eq:shur_formula}
\bm\Sigma_{A}^{-1} = \J_{A} - \J_{AB} \J_{B}^{-1} \J_{AB}^\top.
\end{equation}
\end{lemma}

Consider a square on the lattice $\mathcal{H} \subset \mathcal{F}_s$ and denote the graph vertices inside it by $A = A(\mathcal{H}) \subset V_s$ and those outside of it by $B = V \backslash A$, as in Lemma \ref{lem:shur_lemma}, then formula (\ref{eq:shur_formula}) applies and we can write
\begin{equation}
\label{eq:matr_inv}
\bm\Sigma_{A}^{-1} = \J_{A} - \J_{AB} \J_{B}^{-1} \J_{AB}^\top = \J_{A} - \bm\Omega,
\end{equation}
or
\begin{equation}
\label{eq:matr_inv_2}
\bm\Sigma_{A} = \(\J_{A} - \bm\Omega\)^{-1}.
\end{equation}
Next, we show that in our setup the right-hand side of (\ref{eq:matr_inv_2}) is a nice function of $\theta_s$. Since the left-hand side is easy to estimate from the data with good precision even when the number of samples is small, this will allow us to get an estimate of the unknown $\theta_s$.

Let us write $k=|A|$ and introduce the following quantity
\begin{equation}
\label{eq:param_est_f}
q(\theta) = \frac{\Tr{\(\J_{A} - \bm\Omega\)^{-1}}-k}{d k\theta^2}.
\end{equation}

\begin{lemma}
\label{lem:trace_est_lemma}
Under Assumptions [A1]-[A3], for a large enough square $\mathcal{H} \subset \mathcal{F}_s$ and vertices $A = A(\mathcal{H}) \subset G_s$ inside it,
\begin{equation}
\label{eq:matr_snv_21}
\left|q(\theta_s) - 1\right| = \left|\frac{\Tr{\(\J_{A} - \bm\Omega\)^{-1}}-k}{dk\theta_s^2} - 1\right| \leqslant \theta_sd + \frac{\theta_sd}{\sqrt{k}} + \frac{\bar{\theta}d}{\sqrt{k}} \leqslant 2\theta_sd.
\end{equation}
\end{lemma}
\begin{proof}
The proof can be found in Appendix \ref{app:algo}.
\end{proof}

The covariance matrix $\bm\Sigma_{A}$ can be directly estimated from the data through the sample covariance, which is the maximum likelihood estimate in the Gaussian populations. Such an estimate of $\Tr{\bm\Sigma_{A}}$ together with equation (\ref{eq:matr_inv_2}) and Lemma \ref{lem:trace_est_lemma} will produce an estimate of the coupling parameter
\begin{equation}
\label{eq:lim_err_est}
\hat{\theta}^2 = \frac{\Tr{\widehat{\bm\Sigma}_{A}}-k}{d k}.
\end{equation}

%

In this section, we rigorously analyze the performance of the obtained estimate. Denote
\begin{equation}
\x_{A,i} = \left. \x_i \right|_{A},
\end{equation}
then the empirical covariance matrix of the obtained measurements $\x_{A,1},\dots,\x_{A,n}$ reads as
\begin{equation}
\widehat{\bm\Sigma}_A = \frac{1}{n}\sum_{i=1}^n \x_{A,i}\x_{A,i}^\top.
\end{equation}
The estimator $\hat{\theta}^2$ from (\ref{eq:lim_err_est}) can be written as
\begin{equation}
\label{eq:lim_err_est_1}
\hat{\theta}^2 = \underbrace{\frac{\Tr{\bm\Sigma_{A}}-k}{d k}}_{\(\theta^*\)^2} + \underbrace{\frac{\Tr{\widehat{\bm\Sigma}_{A}}-\Tr{\bm\Sigma_{A}}}{d k}}_{\Delta\hat{\theta}^2},
\end{equation}
where $\(\theta^*\)^2$ is the target parameter and $\Delta\hat{\theta}^2$ is the error. The first term $\(\theta^*\)^2$ approaches a constant when the dimensions grow, thus the probabilistic error is introduced by the second summand. Recall that $\theta^* \geqslant \ubar{\theta}$, and therefore it is reasonable to cut the values of the estimate $\hat{\theta}$ at $\ubar{\theta}$ if they are less than this threshold.

\begin{lemma}
\label{lem:init_est_concentr}
\begin{equation}
\label{eq:init_est_concentr}
\mathbb{P}\[\left|\hat{\theta} - \theta^*\right| \geqslant t\] \leqslant 2\exp\(-\(2nkd\ubar{\theta}\)^2(1-d\bar{\theta})t^2\).
\end{equation}
\end{lemma}
\begin{proof}
The proof can be found in Appendix \ref{app:algo}.
\end{proof}

In the next section, we utilize this sub-Gaussian concentration inequality to analyze the performance of the model selection algorithm based on the local parameter estimation.

\section{Boundary Detection}
\label{sec:algo}
In this section, we introduce the Basic version of the Greedy Region Detection ({\fontfamily{lmss}\selectfont GRED}) algorithm, which is a relative of the so-called region growing family of techniques. Given a small number $n$ of measurements, its goal is to consistently recover the regions of the true graph using only locally available data and assuming the original regions to be CPMs. We analyze its sample complexity and compare it with the information-theoretic bounds derived in Section \ref{sec:inf_theor_bounds}.

Due to the insufficient number of measurements and limited computational power, we cannot estimate the unknown parameters $\theta_s$ by directly estimating the precision matrix e.g. through inversion of the sample covariance.
Instead, the {\fontfamily{lmss}\selectfont GRED} algorithm exploits a greedy strategy. It starts from building a coarse lattice and estimating the coupling parameters inside its cells. Based on the obtained values, it carefully chooses a number of \textit{seeds} that will eventually greedily grow into the detected regions by attaching to them neighboring cells with similar parameter values unless they reach saturation caused by the decreasing size of the lattice cells.


\subsection{Basic-{\fontfamily{lmss}\selectfont GRED} Algorithm for Polyominoes and Its Consistency}
\label{sec:algo_subs}
The basic version of the algorithm proposed in this section assumes that the underlying class of models consists of polyominoes on a square lattice. Even though Assumption [A1] suggests that we look at the family of CPMs, the algorithm presented here can be applied to non-convex polyominoes as well. In Section \ref{sec:algo_subs_conv} the convex version of the algorithm will be developed. 

The Basic-{\fontfamily{lmss}\selectfont GRED} Algorithm involves two major steps.

\begin{figure}
\hspace{1.5cm}
\begin{tikzpicture}[inner sep=0in,outer sep=0in,scale=1]
\node (n) {\begin{varwidth}{9.2cm}{
\begin{ytableau}
*(black!0) & *(black!0) &*(black!0) & *(black!0) & *(black!0) & *(black!0) & *(black!0) & *(black!0) & *(black!0) & *(black!0) & *(black!20) & *(black!20) & *(black!0) & *(black!0) & *(black!0) & *(black!0) & *(black!0) & *(black!0) & *(black!0) & *(black!0) \\
*(black!0) & *(black!0) &*(black!0) & *(black!0) & *(black!0) & *(black!0) & *(black!20) & *(black!20) & *(black!20) & *(black!20) & *(black!20) & *(black!20) & *(black!0) & *(black!0) & *(black!0) & *(black!0) & *(black!0) & *(black!0) & *(black!0) & *(black!0) \\
*(black!0) & *(black!0) &*(black!0) & *(black!0) & *(black!20) & *(black!20) & *(black!20) & *(black!20) & *(black!20) & *(black!20) & *(black!20) & *(black!20) & *(black!20) & *(black!20) & *(black!0) & *(black!0) & *(black!0) & *(black!0) & *(black!0) & *(black!0) \\
*(black!0) & *(black!0) &*(black!0) & *(black!0) & *(blue!20) & *(blue!20) & *(blue!20) & *(blue!20) & *(blue!20) & *(blue!20) & *(blue!20) & *(blue!20) & *(blue!20) & *(blue!20) & *(blue!20) & *(blue!20) & *(black!0) & *(black!0) & *(black!0) & *(black!0) \\
*(black!0) & *(black!0) &*(black!0) & *(black!0) & *(blue!20) & *(blue!20) & *(blue!20) & *(blue!20) & *(blue!20) & *(blue!20) & *(red!20) & *(red!20) & *(blue!20) & *(blue!20) & *(blue!20) & *(blue!20) & *(black!0) & *(black!0) & *(black!0) & *(black!0) \\
*(black!0) & *(black!20) &*(black!20) & *(black!20) & *(blue!20) & *(blue!20) & *(blue!20) & *(blue!20) & *(blue!20) & *(blue!20) & *(red!20) & *(red!20) & *(blue!20) & *(blue!20) & *(blue!20) & *(blue!20) & *(black!0) & *(black!2) & *(black!0) & *(black!0) \\
*(black!20) & *(black!20) &*(black!20) & *(black!20) & *(blue!20) & *(blue!20) & *(blue!20) & *(blue!20) & *(red!20) & *(red!20) & *(blue!20) & *(blue!20) & *(red!20) & *(red!20) & *(blue!20) & *(blue!20) & *(black!20) & *(black!0) & *(black!0) & *(black!0) \\
*(black!20) & *(black!20) &*(black!20) & *(black!20) & *(blue!20) & *(blue!20) & *(blue!20) & *(blue!20) & *(red!20) & *(red!20) & *(blue!20) & *(blue!20) & *(red!20) & *(red!20) & *(blue!20) & *(blue!20) & *(black!20) & *(black!20) & *(black!20) & *(black!0) \\
*(black!0) & *(black!0) &*(black!20) & *(black!20) & *(blue!20) & *(blue!20) & *(red!20) & *(red!20) & *(blue!20) & *(blue!20) & *(red!40) & *(red!40) & *(blue!20) & *(blue!20) & *(red!20) & *(red!20) & *(black!20) & *(black!20) & *(black!20) & *(black!20) \\
*(black!0) & *(black!0) &*(black!20) & *(black!20) & *(blue!20) & *(blue!20) & *(red!20) & *(red!20) & *(blue!20) & *(blue!20) & *(red!40) & *(red!40) & *(blue!20) & *(blue!20) & *(red!20) & *(red!20) & *(black!20) & *(black!20) & *(black!20) & *(black!22) \\
*(black!0) & *(black!0) &*(black!0) & *(black!20) & *(blue!20) & *(blue!20) & *(blue!20) & *(blue!20) & *(red!20) & *(red!20) & *(blue!20) & *(blue!20) & *(red!20) & *(red!20) & *(blue!20) & *(blue!20) & *(black!20) & *(black!20) & *(black!20) & *(black!0) \\
*(black!0) & *(black!0) &*(black!0) & *(black!0) & *(blue!20) & *(blue!20) & *(blue!20) & *(blue!20) & *(red!20) & *(red!20) & *(blue!20) & *(blue!20) & *(red!20) & *(red!20) & *(blue!20) & *(blue!20) & *(black!20) & *(black!20) & *(black!20) & *(black!0) \\
*(black!0) & *(black!0) &*(black!0) & *(black!0) & *(blue!20) & *(blue!20) & *(blue!20) & *(blue!20) & *(blue!20) & *(blue!20) & *(red!20) & *(red!20) & *(blue!20) & *(blue!20) & *(blue!20) & *(blue!20) & *(black!20) & *(black!0) & *(black!0) & *(black!0) \\
*(black!0) & *(black!0) &*(black!0) & *(black!0) & *(blue!20) & *(blue!20) & *(blue!20) & *(blue!20) & *(blue!20) & *(blue!20) & *(red!20) & *(red!20) & *(blue!20) & *(blue!20) & *(blue!20) & *(blue!20) & *(black!0) & *(black!0) & *(black!0) & *(black!0) \\
*(black!0) & *(black!0) &*(black!0) & *(black!0) & *(blue!20) & *(blue!20) & *(blue!20) & *(blue!20) & *(blue!20) & *(blue!20) & *(blue!20) & *(blue!20) & *(blue!20) & *(blue!20) & *(blue!20) & *(blue!20) & *(black!0) & *(black!0) & *(black!0) & *(black!0) \\
*(black!0) & *(black!0) &*(black!0) & *(black!0) & *(black!0) & *(black!0) & *(black!0) & *(black!20) & *(black!20) & *(black!20) & *(black!20) & *(black!20) & *(black!20) & *(black!0) & *(black!0) & *(black!0) & *(black!0) & *(black!0) & *(black!0) & *(black!0) \\
\end{ytableau}}\end{varwidth}};
\draw[very thick,purple] (-4.594,5.1954)--(8.3985,5.1954);
\draw[very thick,purple] (-4.594,3.8966)--(8.3985,3.8966);
\draw[very thick,blue] (-1.9958,3.2472)--(5.7996,3.2472);
\draw[very thick,purple] (-4.594,2.5978)--(8.3985,2.5978);
\draw[very thick,purple] (-4.594,1.2990)--(8.3985,1.2990);
\draw[very thick,purple] (-4.594,0.002)--(8.3985,0.002);
\draw[very thick,purple] (-4.594,-1.2986)--(8.3985,-1.2986);
\draw[very thick,purple] (-4.594,-2.5974)--(8.3985,-2.5974);
\draw[very thick,purple] (-4.594,-3.8962)--(8.3985,-3.8962);
\draw[very thick,blue] (-1.9958,-4.546)--(5.7996,-4.546);
\draw[very thick,purple] (-4.594,-5.1950)--(8.3985,-5.1950);

\draw[very thick,purple] (-4.594,-5.1954)--(-4.594,5.1954);
\draw[very thick,purple] (-3.2948,-5.1954)--(-3.2948,5.1954);
\draw[very thick,purple] (-1.9956,3.2472)--(-1.9956,5.1954);
\draw[very thick,purple] (-1.9956,-4.546)--(-1.9956,-5.1954);
\draw[very thick,blue] (-1.9956,3.2472)--(-1.9956,-4.546);
\draw[very thick,purple] (-0.6965,-5.1954)--(-0.6965,5.1954);
\draw[very thick,purple] (0.6028,-5.1954)--(0.6028,5.1954);
\draw[very thick,purple] (1.9020,-5.1954)--(1.9020,5.1954);
\draw[very thick,purple] (3.2012,-5.1954)--(3.2012,5.1954);
\draw[very thick,purple] (4.5004,-5.1954)--(4.5004,5.1954);
\draw[very thick,purple] (5.7996,3.2472)--(5.7996,5.1954);
\draw[very thick,purple] (5.7996,-5.1954)--(5.7996,-4.546);
\draw[very thick,blue] (5.7996,-4.546)--(5.7996,3.2472);
\draw[very thick,purple] (7.0988,-5.1954)--(7.0988,5.1954);
\draw[very thick,purple] (8.3980,-5.1954)--(8.3980,5.1954);

\node[] at (2.58,2.1494) {\footnotesize $\hat{\theta}^{(0)}(\mathcal{H}_1)$};
\node[] at (1.2884,0.8506) {\footnotesize $\hat{\theta}^{(0)}(\mathcal{H}_8)$};
\node[] at (3.9,0.8506) {\footnotesize $\hat{\theta}^{(0)}(\mathcal{H}_2)$};
\node[] at (-0.02,-0.4494) {\footnotesize $\hat{\theta}^{(0)}(\mathcal{H}_7)$};
\node[] at (2.58,-0.4494) {\footnotesize $\hat{\theta}^{(0)}(\mathcal{H}_0)$};
\node[] at (5.1672,-0.4494) {\footnotesize $\hat{\theta}^{(0)}(\mathcal{H}_3)$};
\node[] at (1.2884,-1.7482) {\footnotesize $\hat{\theta}^{(0)}(\mathcal{H}_6)$};
\node[] at (3.9,-1.7482) {\footnotesize $\hat{\theta}^{(0)}(\mathcal{H}_4)$};
\node[] at (2.58,-3.0482) {\footnotesize $\hat{\theta}^{(0)}(\mathcal{H}_5)$};
\end{tikzpicture}
\begin{flushleft}
\small \textbf{Legend:} The grid lines of the original lattice on which the graph was generated are black and the true region $G_s$ on it is colored gray. The square of the side length $6\tau^{(0)}$ inside it is blue, the coarse original lattice of width $\tau^{(0)}$ is purple, the detected seed is red and its $8$ neighboring squares are pink.
\end{flushleft}
\vspace{-0.2cm}
\caption{\small First stage of the {\fontfamily{lmss}\selectfont GRED} algorithm.}
\label{fig:algo_1}
\end{figure}

\begin{figure}
\centering
\begin{tikzpicture}[inner sep=0in,outer sep=0in,scale=1]
\node (n) {\begin{varwidth}{9.2cm}{
\begin{ytableau}
*(black!0) & *(black!0) &*(black!0) & *(black!0) & *(black!0) & *(black!0) & *(black!0) & *(black!0) & *(black!0) & *(black!0) & *(black!20) & *(black!20) & *(black!0) & *(black!0) & *(black!0) & *(black!0) & *(black!0) & *(black!0) & *(black!0) & *(black!0) \\
*(black!0) & *(black!0) &*(black!0) & *(black!0) & *(black!0) & *(black!0) & *(black!20) & *(black!20) & *(black!20) & *(black!20) & *(black!20) & *(black!20) & *(black!0) & *(black!0) & *(black!0) & *(black!0) & *(black!0) & *(black!0) & *(black!0) & *(black!0) \\
*(black!0) & *(black!0) &*(black!0) & *(black!0) & *(black!20) & *(black!20) & *(black!20) & *(black!20) & *(black!20) & *(black!20) & *(black!20) & *(black!20) & *(black!20) & *(black!20) & *(black!0) & *(black!0) & *(black!0) & *(black!0) & *(black!0) & *(black!0) \\
*(black!0) & *(black!0) &*(black!0) & *(black!0) & *(blue!20) & *(blue!20) & *(blue!20) & *(blue!20) & *(blue!20) & *(blue!20) & *(blue!20) & *(blue!20) & *(blue!20) & *(blue!20) & *(blue!20) & *(blue!20) & *(black!0) & *(black!0) & *(black!0) & *(black!0) \\
*(black!0) & *(black!0) &*(black!0) & *(black!0) & *(blue!20) & *(blue!20) & *(blue!20) & *(blue!20) & *(blue!20) & *(blue!20) & *(blue!20) & *(blue!20) & *(blue!20) & *(blue!20) & *(blue!20) & *(blue!20) & *(black!0) & *(black!0) & *(black!0) & *(black!0) \\
*(black!0) & *(black!20) &*(black!20) & *(black!20) & *(blue!20) & *(blue!20) & *(blue!20) & *(blue!20) & *(blue!20) & *(blue!20) & *(blue!20) & *(blue!20) & *(blue!20) & *(blue!20) & *(blue!20) & *(blue!20) & *(black!0) & *(black!2) & *(black!0) & *(black!0) \\
*(black!20) & *(black!20) &*(black!20) & *(black!20) & *(blue!20) & *(blue!20) & *(blue!20) & *(blue!20) & *(blue!20) & *(blue!20) & *(blue!20) & *(blue!20) & *(blue!20) & *(blue!20) & *(blue!20) & *(blue!20) & *(black!20) & *(black!0) & *(black!0) & *(black!0) \\
*(black!20) & *(black!20) &*(black!20) & *(black!20) & *(blue!20) & *(blue!20) & *(blue!20) & *(blue!20) & *(blue!20) & *(red!40) & *(blue!20) & *(red!40) & *(blue!20) & *(blue!20) & *(blue!20) & *(blue!20) & *(black!20) & *(black!20) & *(black!20) & *(black!0) \\
*(black!0) & *(black!0) &*(black!20) & *(black!20) & *(blue!20) & *(blue!20) & *(blue!20) & *(blue!20) & *(red!40) & *(red!40) & *(red!40) $\hat{\mathcal{F}}$ & *(red!40) & *(blue!20) & *(blue!20) & *(blue!20) & *(blue!20) & *(black!20) & *(black!20) & *(black!20) & *(black!20) \\
*(black!0) & *(black!0) &*(black!20) & *(black!20) & *(blue!20) & *(blue!20) & *(blue!20) & *(blue!20) & *(blue!20) & *(blue!20) & *(red!40) & *(red!40) & *(blue!20) & *(blue!20) & *(blue!20) & *(blue!20) & *(black!20) & *(black!20) & *(black!20) & *(black!22) \\
*(black!0) & *(black!0) &*(black!0) & *(black!20) & *(blue!20) & *(blue!20) & *(blue!20) & *(blue!20) & *(blue!20) & *(blue!20) & *(red!40) & *(red!40) & *(blue!20) & *(blue!20) & *(blue!20) & *(blue!20) & *(black!20) & *(black!20) & *(black!20) & *(black!0) \\
*(black!0) & *(black!0) &*(black!0) & *(black!0) & *(blue!20) & *(blue!20) & *(blue!20) & *(blue!20) & *(blue!20) & *(blue!20) & *(blue!20) & *(blue!20) & *(blue!20) & *(blue!20) & *(blue!20) & *(blue!20) & *(black!20) & *(black!20) & *(black!20) & *(black!0) \\
*(black!0) & *(black!0) &*(black!0) & *(black!0) & *(blue!20) & *(blue!20) & *(blue!20) & *(blue!20) & *(blue!20) & *(blue!20) & *(blue!20) & *(blue!20) & *(blue!20) & *(blue!20) & *(blue!20) & *(blue!20) & *(black!20) & *(black!0) & *(black!0) & *(black!0) \\
*(black!0) & *(black!0) &*(black!0) & *(black!0) & *(blue!20) & *(blue!20) & *(blue!20) & *(blue!20) & *(blue!20) & *(blue!20) & *(blue!20) & *(blue!20) & *(blue!20) & *(blue!20) & *(blue!20) & *(blue!20) & *(black!0) & *(black!0) & *(black!0) & *(black!0) \\
*(black!0) & *(black!0) &*(black!0) & *(black!0) & *(blue!20) & *(blue!20) & *(blue!20) & *(blue!20) & *(blue!20) & *(blue!20) & *(blue!20) & *(blue!20) & *(blue!20) & *(blue!20) & *(blue!20) & *(blue!20) & *(black!0) & *(black!0) & *(black!0) & *(black!0) \\
*(black!0) & *(black!0) &*(black!0) & *(black!0) & *(black!0) & *(black!0) & *(black!0) & *(black!20) & *(black!20) & *(black!20) & *(black!20) & *(black!20) & *(black!20) & *(black!0) & *(black!0) & *(black!0) & *(black!0) & *(black!0) & *(black!0) & *(black!0) \\
\end{ytableau}}\end{varwidth}};
\draw[very thick,purple] (-4.594,5.1954)--(8.3985,5.1954);
\draw[very thick,purple] (-4.594,4.5460)--(8.3985,4.5460);
\draw[very thick,purple] (-4.594,3.8966)--(8.3985,3.8966);
\draw[very thick,purple] (-4.594,3.2472)--(-1.9958,3.2472);
\draw[very thick,blue] (-1.9958,3.2472)--(5.7996,3.2472);
\draw[very thick,purple] (5.7996,3.2472)--(8.3985,3.2472);
\draw[very thick,purple] (-4.594,2.5978)--(8.3985,2.5978);
\draw[very thick,purple] (-4.594,1.9484)--(8.3985,1.9484);
\draw[very thick,purple] (-4.594,1.2990)--(8.3985,1.2990);
\draw[very thick,purple] (-4.594,0.6496)--(8.3985,0.6496);
\draw[very thick,purple] (-4.594,0.002)--(8.3985,0.002);
\draw[very thick,purple] (-4.594,-0.6492)--(8.3985,-0.6492);
\draw[very thick,purple] (-4.594,-1.2986)--(8.3985,-1.2986);
\draw[very thick,purple] (-4.594,-1.9480)--(8.3985,-1.9480);
\draw[very thick,purple] (-4.594,-2.5974)--(8.3985,-2.5974);
\draw[very thick,purple] (-4.594,-3.2468)--(8.3985,-3.2468);
\draw[very thick,purple] (-4.594,-3.8962)--(8.3985,-3.8962);
\draw[very thick,purple] (-4.594,-4.546)--(-1.9958,-4.546);
\draw[very thick,blue] (-1.9958,-4.546)--(5.7996,-4.546);
\draw[very thick,purple] (5.7996,-4.546)--(8.3985,-4.546);
\draw[very thick,purple] (-4.594,-5.1950)--(8.3985,-5.1950);

\draw[very thick,purple] (-4.594,-5.1954)--(-4.594,5.1954);
\draw[very thick,purple] (-3.9444,-5.1954)--(-3.9444,5.1954);
\draw[very thick,purple] (-3.2948,-5.1954)--(-3.2948,5.1954);
\draw[very thick,purple] (-2.6452,-5.1954)--(-2.6452,5.1954);
\draw[very thick,purple] (-1.9956,3.2472)--(-1.9956,5.1954);
\draw[very thick,purple] (-1.9956,-4.546)--(-1.9956,-5.1954);
\draw[very thick,blue] (-1.9956,3.2472)--(-1.9956,-4.546);
\draw[very thick,purple] (-1.3460,-5.1954)--(-1.3460,5.1954);
\draw[very thick,purple] (-0.6965,-5.1954)--(-0.6965,5.1954);
\draw[very thick,purple] (-0.0468,-5.1954)--(-0.0468,5.1954);
\draw[very thick,purple] (0.6028,-5.1954)--(0.6028,5.1954);
\draw[very thick,purple] (1.2524,-5.1954)--(1.2524,5.1954);
\draw[very thick,purple] (1.9020,-5.1954)--(1.9020,5.1954);
\draw[very thick,purple] (2.5516,-5.1954)--(2.5516,5.1954);
\draw[very thick,purple] (3.2012,-5.1954)--(3.2012,5.1954);
\draw[very thick,purple] (3.8508,-5.1954)--(3.8508,5.1954);
\draw[very thick,purple] (4.5004,-5.1954)--(4.5004,5.1954);
\draw[very thick,purple] (5.1500,-5.1954)--(5.1500,5.1954);
\draw[very thick,purple] (5.7996,3.2472)--(5.7996,5.1954);
\draw[very thick,purple] (5.7996,-5.1954)--(5.7996,-4.546);
\draw[very thick,blue] (5.7996,-4.546)--(5.7996,3.2472);
\draw[very thick,purple] (6.4492,-5.1954)--(6.4492,5.1954);
\draw[very thick,purple] (7.0988,-5.1954)--(7.0988,5.1954);
\draw[very thick,purple] (7.7484,-5.1954)--(7.7484,5.1954);
\draw[very thick,purple] (8.3980,-5.1954)--(8.3980,5.1954);

\end{tikzpicture}
\caption{\small Second stage of the {\fontfamily{lmss}\selectfont GRED} algorithm; refined mesh.}
\label{fig:algo_2}
\end{figure}
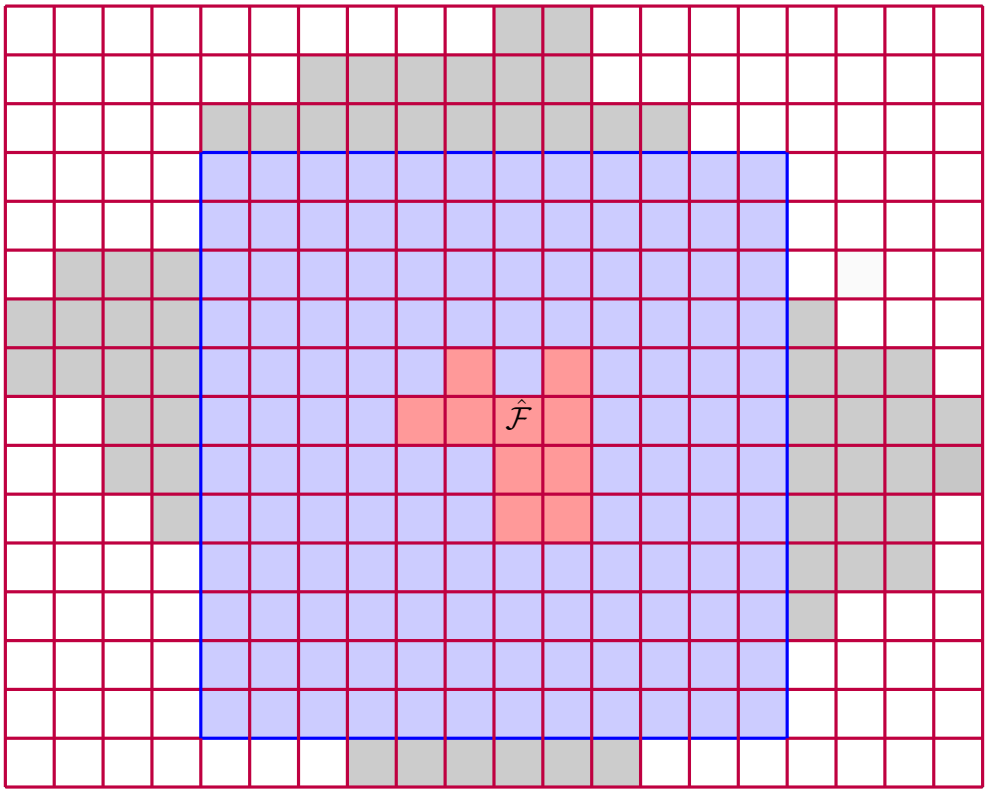

\noindent \textbf{Step 1.} We start by creating a coarse lattice covering the graph at hand, which will be later refined. Here we need to make the following assumption. A point $\mathfrak{p}$ representing a node of the original lattice used to generate the graph is given together with a unit vector $\mathfrak{v}$ showing the direction of one of the lattice axes. This assumption is technical and is required to enable consistency of the estimation measured by the number of misclassified lattice cells. See remark after the proof of Theorem \ref{thm:main_empir} below regarding the setting in which the point and/or the direction are not known. Also we assume that the width of the original lattice $r$ is known. We set the initial side length of the lattice on the first iteration of the algorithm to be
\begin{equation}
\label{eq:init_side_alg}
\tau^{(0)} = \frac{1}{3} \min_s \frac{A_s}{P_s} = \frac{1}{3} \min_s \frac{\sqrt{A_s}}{\beta_s},
\end{equation}
and, as we also did earlier, assume without loss of generality that $\tau^{(0)}$ is a multiple of $r$. The justification for using the value in (\ref{eq:init_side_alg}) for $\tau^{(0)}$ comes from Lemma \ref{lem:min_inscr_sq}, claiming that every convex polygon of area $A$ and perimeter $P$ must contain a $\frac{2A}{P} \times \frac{2A}{P}$ square. In the next paragraph we explain why $\tau^{(0)}$ is taken to be $\frac{1}{6}$ of this value.

The constructed lattice cuts the graph into squares denoted by $\mathcal{H}_j$, inside which we estimate the values $\hat{\theta}^{(0)}(\mathcal{H}_j)$ of the coupling parameter using formula (\ref{eq:lim_err_est}). Fix any cell $\mathcal{H}_0$. Our goal is to determine whether it completely belongs to any region of the original graph or not. For that purpose, take $8$ $\tau^{(0)} \times \tau^{(0)}$ lattice cells surrounding $\mathcal{H}_0$ in the pattern shown in Figure \ref{fig:algo_1}. 
The definition of $\tau^{(0)}$ suggests that for every region there would exist such a configuration of $9$ ($1$ red + $8$ pink) squares that lies inside it. Based on the estimates $\hat{\theta}^{(0)}(\mathcal{H}_i),\; i=0,\dots,8$ we decide whether $\mathcal{H}_0$ is a seed of a new detected region using the test described below. The choice of such pattern of cells can be explained in the following way. On the one hand, we want the test cells to be spatially close to each other since we need to determine whether they belong to the same region or not. On the other hand, if they share sides, the values of their parameter estimates will be statistically dependent and will decrease the power of our hypotheses test. The arrangement in which the squares only touch at the corner nodes is beneficial from both perspectives. Indeed, since the number of edges passing from one of the touching squares to another is bounded by a constant (and therefore negligible in the asymptotics), their respective coupling parameter estimates can be assumed independent.

Denote the number of vertices inside one lattice square of width $\tau^{(0)}$ by
\begin{equation}
k(\tau^{(0)}) = \(\tau^{(0)}\)^2\eta.
\end{equation}
Let
\begin{equation}
\ubar{\delta} = \min_{s \neq s'} |\theta_s - \theta_{s'}|,
\end{equation}
and set the threshold
\begin{equation}
\label{eq:threshold_val}
\zeta = \frac{\ubar{\delta}}{2}.
\end{equation}
We declare that the (red) central cell $\mathcal{H}_0$ is a seed of a new region if
\begin{equation}
\label{eq:init_st_test}
\max_i \hat{\theta}^{(0)}(\mathcal{H}_i) - \min_i \hat{\theta}^{(0)}(\mathcal{H}_i) \leqslant \zeta.
\end{equation}
False detection may happen if we label the seed $\mathcal{H}_0$ as belonging to $G_{s}$, while it partially (or completely) belongs to different regions. Due to the convexity of the regions, this can happen only if one of the cells $\mathcal{H}_{j},\; j=1,\dots,8$ lies in the compliment of $\mathcal{F}_s$, and therefore the estimate $\hat{\theta}^{(0)}(\mathcal{H}_{j})$ deviates at least $\frac{\ubar{\delta}}{2}$ from its expected value. By Lemma \ref{lem:init_est_concentr}, the probability of this event is bounded as
\begin{equation}
\mathbb{P}\[\left|\hat{\theta}^{(0)}(\mathcal{H}_i) - \theta_{s'}\right| \geqslant \frac{\ubar{\delta}-\zeta}{2}\] \leqslant 2\exp\(-\(2ndk(\tau^{(0)})\ubar{\theta}\)^2(1-d\bar{\theta})\frac{\(\ubar{\delta}-\zeta\)^2}{4}\).
\end{equation}
Since we have $8$ cells around the potential seed, the probability of false detection can be bounded from above by
\begin{equation}
\mathbb{P}\[\max_i \hat{\theta}^{(0)}(\mathcal{H}_i) - \min_i \hat{\theta}^{(0)}(\mathcal{H}_i) \geqslant \frac{\ubar{\delta}-\zeta}{2}\] \leqslant 16\exp\(-\(2ndk(\tau^{(0)})\ubar{\theta}\)^2(1-d\bar{\theta})\frac{\(\ubar{\delta}-\zeta\)^2}{4}\).
\end{equation}

\noindent\textbf{Step 2.} At the second stage, the initially chosen seeds start to grow greedily by incorporating new cells that share boundaries with them and have similar estimated coupling parameters. This is done in a loop and on each iteration the lattice is refined, until the lattice cells become so small that the false detection error becomes significant. 

At time (iteration) $t$, each cell $\mathcal{H}$ that has not yet been assigned to a region but has neighboring cells already attached to the same region $\hat{\mathcal{F}}$ with the current parameter estimate $\hat{\theta}^{(t-1)}(\hat{\mathcal{F}})$ is tested for belonging to that same region using the test
\begin{equation}
\left|\hat{\theta}^{(t)}(\mathcal{H}) - \hat{\theta}^{(t-1)}(\hat{\mathcal{F}})\right| \leqslant \zeta,
\end{equation}
and if it passes the test, it is attached. If at some point a cell (multiple cells) is surrounded by the cells already assigned to a region, the former is also added. When no new cells are attached to any region at an iteration, the algorithm halts and the detected regions are declared as the final ones. The area not assigned to any one of the regions is labeled as \textit{gray area}.

We need to determine the minimal lattice width that will allow consistent region recovery. The probability of false detection for a single cell is given by Lemma \ref{lem:init_est_concentr} above. Denote by $\tau^{(t)}$ the lattice width at time $t$, then the probability of false detection on a cell is bounded by
\begin{equation}
\label{eq:prob_bound_t}
\mathbb{P}\[\left|\hat{\theta}^{(t)}(\mathcal{H}) - \hat{\theta}^{(t-1)}(\hat{\mathcal{F}})\right| \geqslant \frac{\ubar{\delta}-\zeta}{2}\] \leqslant 2\exp\(-\(2ndk(\tau^{(t)})\ubar{\theta}\)^2(1-d\bar{\theta})\frac{(\ubar{\delta}-\zeta)^2}{4}\).
\end{equation}

Now we need to calculate how many hypotheses are simultaneously tested, in order to apply the union bound and upper bound the total probability of false detection. Due to the decay of the lattice width $\tau^{(t)}$ in the loop, the exponent in the right-hand side of (\ref{eq:prob_bound_t}) significantly increases on every successive iteration. The number of tested hypotheses only increases on every iteration since the cells become smaller and the boundaries of the growing regions become longer. Therefore, we can neglect the false detection errors on iterations up to $t-1$ compared to the error on iteration $t$.

\begin{algorithm}[t]
 \caption{Change Manifold Detection in Markov Fields}
 \begin{algorithmic}[1]
 \label{algo}
 \renewcommand{\algorithmicrequire}{\textbf{Input:}}
 \renewcommand{\algorithmicensure}{\textbf{Output:}}
 \renewcommand{\algorithmicfor}{\textbf{parforeach}}
 \REQUIRE $V \subset \mathbb{R}^2,\; \mathfrak{p},\; \mathfrak{v},\; \x_1,\dots,\x_n \in \mathbb{R}^p,\; d, \rho$;
 \ENSURE  $\{V_k\}$ such that $V = \bigcup_k V_k \cup V_g$; \\
  \tcc{--------------------FIRST STAGE--------------------}
  \STATE $\eta = \frac{p}{A}$;
  \STATE cover the graph by a square lattice $H_{\tau^{(0)}}$ with the cell size $\tau^{(0)}$;
  \FOR{$\mathcal{H}_0 \in H_{\tau^{(0)}}$}
  \STATE calculate $\hat{\theta}^{(0)}_i = \hat{\theta}^{(0)}(\mathcal{H}_i),\;i=1,\dots,8$ for the corner-neighbors of $\mathcal{H}_0$ using (\ref{eq:lim_err_est});
  \IF{$\max_i \hat{\theta}^{(0)}(\mathcal{H}_i) - \min_i \hat{\theta}^{(0)}(\mathcal{H}_i) \leqslant \zeta$}
 \STATE declare $\mathcal{H}_0$ to be a seed of the new region $\hat{\mathcal{F}}$;
 \ENDIF
 \ENDFOR \\
 \tcc{--------------------SECOND STAGE--------------------}
 \WHILE{$\tau^{(t)} \geqslant \rho \(\frac{p}{\eta}\)^{\xi}$} 
 \REPEAT
 \IF{$\left|\hat{\theta}^{(t)}(\mathcal{H}) - \hat{\theta}^{(t-1)}(\hat{\mathcal{F}})\right| \leqslant \zeta$ \OR \\ cell is surrounded by the cells with the same parameter}
 \STATE glue them into a single region; 
 \ENDIF
 \UNTIL{no cells were added}
 \STATE decrease $\tau^{(t)}$\tcc*{refine the lattice}
 \ENDWHILE
 \STATE $V_g = $ vertices not inside regions\tcc*{gray vertices}
 \end{algorithmic} 
 \end{algorithm}

Since the regions grow on every step, it is reasonable to bound the lengths of their boundaries by their true values $l_s$. The detection of the new cells occurs only on the boundary of the growing regions, thus the number of cells tested on iteration $t$ is bounded from above by a constant multiple of $\frac{\sum_s l_s}{r}$, since the smallest side of the region must be a multiple of $r$ due to (\ref{eq:def_curv_rad_area}). Overall, the probability of false detection on the $t$-th iteration is bounded by
\begin{equation}
\label{eq:union_err_b}
\mathbb{P}\[N_{\text{err}}\] \leqslant \frac{\sum_s l_s}{\tau^{(t)}}\cdot 2\exp\(-\(nd\ubar{\theta}\ubar{\delta}\eta \(\tau^{(t)}\)^2\)^2(1-d\bar{\theta})/4\).
\end{equation}
Due to Assumption [A1], the width $\tau^{(t)}$ of the lattice on the last iteration must be at least
\begin{equation}
\tau^{(t)} \geqslant r = \rho \(\frac{p}{\eta}\)^{\xi},
\end{equation}
therefore, to make the left-hand side of (\ref{eq:union_err_b}) vanish asymptotically, we require
\begin{equation}
\(nd\ubar{\theta}\ubar{\delta}\eta^{1-2\xi} \rho^2 p^{2\xi}\)^2(1-d\bar{\theta}) = \omega\(\log\(\frac{\eta^\xi\sum_s l_s}{\rho p^{\xi}}\)\),
\end{equation}
where $a(x) = \omega(b(x)),\; x \to \infty$ means that $b(x)/a(x) \to 0,\; x \to \infty$. We obtain,
\begin{equation}
\label{eq:suff_sc}
n = \omega\(\frac{\log^{1/2}\(S \bar{\beta}\sqrt{\bar{\nu} p}/(\rho p^{\xi})\)}{p^{2\xi}\rho^2d\ubar{\theta}\ubar{\delta}\eta^{1-2\xi}(1-d\bar{\theta})}\) = 
\begin{cases} \omega\(\frac{\log^{1/2}p}{p^{2\xi}}\frac{1}{\rho^2 d\ubar{\theta}\ubar{\delta}\eta^{1-2\xi}(1-d\bar{\theta})}\), & \xi < \frac{1}{2}, \\  
\omega\(\frac{1}{p}\frac{\log^{1/2}\(S\eta^{1/2} \bar{\beta}\sqrt{\bar{\nu} }/\rho)\)}{\rho^2 d\ubar{\theta}\ubar{\delta}(1-d\bar{\theta})}\), & \xi = \frac{1}{2}, \end{cases} \quad p \to \infty.
\end{equation}
So far we have discussed the issue of false detection (false positives or type I error). By applying similar reasoning to the false negatives (type II error) we will prove consistency. Indeed, since when $p \to \infty$ the estimates $\hat{\theta}^{(t)}(\mathcal{H})$ over all test squares exhibit sub-Gaussian concentration, the detected boundaries will approach the true boundaries of the regions if the number of samples satisfies (\ref{eq:suff_sc}). We have proven the following result.
\begin{theorem}[Structural consistency of \text{\fontfamily{lmss}\selectfont GRED}]
\label{thm:main_empir}
Suppose that Assumptions [A1] - [A4] hold and assume that a graph $G_p \in R_{p}$ is chosen uniformly from the class $R_{p}$ which is in turn chosen uniformly from the family $\mathcal{R}_{p}$ of models. Let the number of i.i.d.\ samples from $G_p$ grow as
\begin{equation}
\label{eq:main_empir_thm}
n = 
\begin{cases} \omega\(\frac{\log^{1/2}p}{p^{2\xi}}\frac{1}{\rho^2 d\ubar{\theta}\ubar{\delta}\eta^{1-2\xi}(1-d\bar{\theta})}\), & \xi < \frac{1}{2}, \\  
\omega\(\frac{1}{p}\frac{\log^{1/2}\(S\eta^{1/2} \bar{\beta}\sqrt{\bar{\nu} }/\rho)\)}{\rho^2 d\ubar{\theta}\ubar{\delta}\eta(1-d\bar{\theta})}\), & \xi = \frac{1}{2}, \end{cases} \quad p \to \infty.
\end{equation}
then $\SI{}{\text{\fontfamily{lmss}\selectfont GRED}}(\mathcal{X}^n)$ succeeds almost surely w.r.t. the model choice, i.e.
\begin{equation}
\lim_{p \to \infty} \mathbb{P}\[\SI{}{\text{\fontfamily{lmss}\selectfont GRED}}(\mathcal{X}^n) \neq R_{p}\] = 0.
\end{equation}
\end{theorem}

Interestingly, for $\xi<\frac{1}{2}$ the rate of the decay of the left-hand side of (\ref{eq:main_empir_thm}) with $p$ matches that predicted by Theorem \ref{th:nec_cond_region} up to a $\sqrt{\log\,p}$ multiplier. This essentially means that despite being a greedy algorithm, the proposed detection technique is efficient in terms of its sample complexity.

Let us make a number of remarks about the operation of Algorithm \ref{algo}. First, if the size of a region is increased during the iteration via incorporation of new cells, the estimated values of its coupling parameter $\hat{\theta}^{(t)}(\hat{\mathcal{F}})$ may be updated to improve the precision of the estimates. Second, the algorithm will still work even if the values of the edge parameters vary slightly inside the regions. The critical condition being that the amplitude of the variation inside the regions is negligible compared to $\ubar{\delta}$. Finally, in the real world applications, the reference point $\mathfrak{p}$ and the direction $\mathfrak{v}$ are not known. It is important to emphasize that we only need to specify them in order to guarantee consistency of the {\fontfamily{lmss}\selectfont GRED} algorithm since we must make sure that the recovery is feasible, or equivalently, the region boundaries can be found exactly. However, the algorithm may be executed without the reference point and direction, in which case the lattice of the resulting graph may not be aligned with the original one used to generate it.


\subsection{Convex-{\fontfamily{lmss}\selectfont GRED} Algorithm}
\label{sec:algo_subs_conv}

\begin{figure}
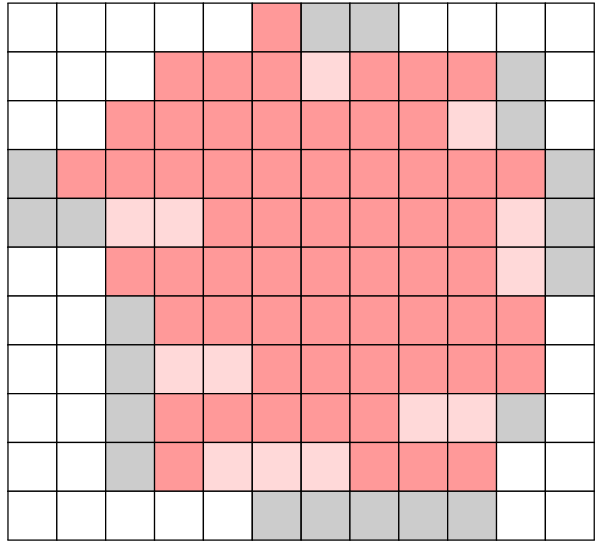

\centering
\begin{ytableau}
*(black!0) & *(black!0) & *(black!0) & *(black!0) & *(black!0) & *(red!40) & *(black!20) & *(black!20) & *(black!0) & *(black!0) & *(black!0) & *(black!0) \\
*(black!0) & *(black!0) & *(black!0) & *(red!40) & *(red!40) & *(red!40) & *(red!15) & *(red!40) & *(red!40) & *(red!40) & *(black!20) & *(black!0) \\
*(black!0) & *(black!0) & *(red!40) & *(red!40) & *(red!40) & *(red!40) & *(red!40) & *(red!40) & *(red!40) & *(red!15) & *(black!20) & *(black!0) \\
*(black!20) & *(red!40) & *(red!40) & *(red!40) & *(red!40) & *(red!40) & *(red!40) & *(red!40) & *(red!40) & *(red!40) & *(red!40) & *(black!20) \\
*(black!20) & *(black!20) & *(red!15) & *(red!15) & *(red!40) & *(red!40) & *(red!40) & *(red!40) & *(red!40) & *(red!40) & *(red!15) & *(black!20) \\
*(black!0) & *(black!0) & *(red!40) & *(red!40) & *(red!40) & *(red!40) & *(red!40) & *(red!40) & *(red!40) & *(red!40) & *(red!15) & *(black!20) \\
*(black!0) & *(black!0) & *(black!20) & *(red!40) & *(red!40) & *(red!40) & *(red!40) & *(red!40) & *(red!40) & *(red!40) & *(red!40) & *(black!0) \\
*(black!0) & *(black!0) & *(black!20) & *(red!15) & *(red!15) & *(red!40) & *(red!40) & *(red!40) & *(red!40) & *(red!40) & *(red!40) & *(black!0) \\
*(black!0) & *(black!0) & *(black!20) & *(red!40) & *(red!40) & *(red!40) & *(red!40) & *(red!40) & *(red!15) & *(red!15) & *(black!20) & *(black!0) \\
*(black!0) & *(black!0) & *(black!20) & *(red!40) & *(red!15) & *(red!15) & *(red!15) & *(red!40) & *(red!40) & *(red!40) & *(black!0) & *(black!0) \\
*(black!0) & *(black!0) & *(black!0) & *(black!0) & *(black!0) & *(black!20) & *(black!20) & *(black!20) & *(black!20) & *(black!20) & *(black!0) & *(black!0) \\
\end{ytableau}
\caption{\small Convexification of the region detected by the Basic-{\fontfamily{lmss}\selectfont GRED} algorithm.}
\label{fig:algo_conv}
\end{figure}

\begin{figure}
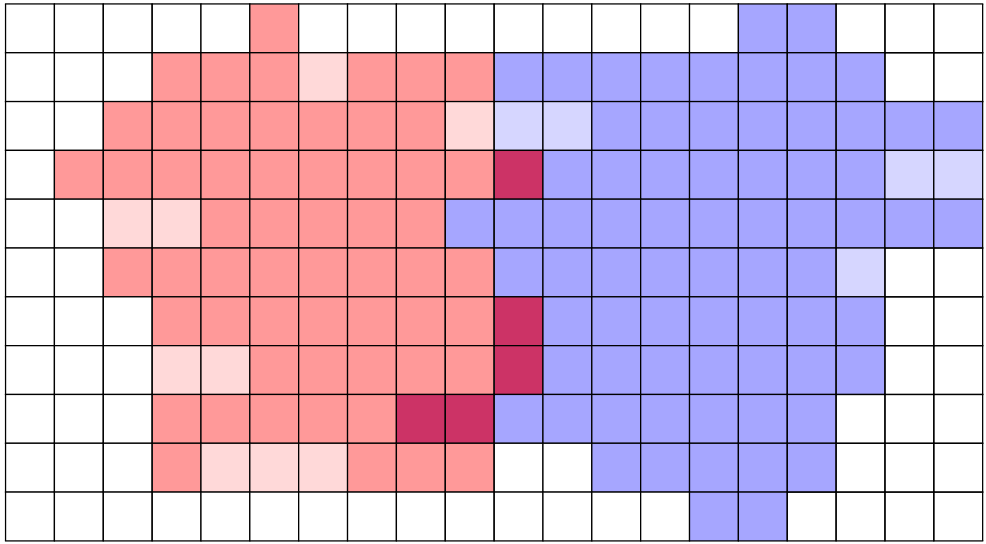

\centering
\begin{ytableau}
*(black!0) & *(black!0) & *(black!0) & *(black!0) & *(black!0) & *(red!40) & *(black!0) & *(black!0) & *(black!0) & *(black!0) & *(black!0) & *(black!0) & *(black!0) & *(black!0) & *(black!0) & *(blue!35) & *(blue!35) & & &\\
*(black!0) & *(black!0) & *(black!0) & *(red!40) & *(red!40) & *(red!40) & *(red!15) & *(red!40) & *(red!40) & *(red!40) & *(blue!35) & *(blue!35) & *(blue!35) & *(blue!35) & *(blue!35) & *(blue!35) & *(blue!35) & *(blue!35) & & \\
*(black!0) & *(black!0) & *(red!40) & *(red!40) & *(red!40) & *(red!40) & *(red!40) & *(red!40) & *(red!40) & *(red!15) & *(blue!16) & *(blue!16) & *(blue!35) & *(blue!35) & *(blue!35) & *(blue!35) & *(blue!35) & *(blue!35) & *(blue!35) & *(blue!35)\\
*(black!0) & *(red!40) & *(red!40) & *(red!40) & *(red!40) & *(red!40) & *(red!40) & *(red!40) & *(red!40) & *(red!40) & *(purple!80) & *(blue!35) & *(blue!35) & *(blue!35) & *(blue!35) & *(blue!35) & *(blue!35) & *(blue!35) & *(blue!16) & *(blue!16) \\
*(black!0) & *(black!0) & *(red!15) & *(red!15) & *(red!40) & *(red!40) & *(red!40) & *(red!40) & *(red!40) & *(blue!35) & *(blue!35) & *(blue!35) & *(blue!35) & *(blue!35) & *(blue!35) & *(blue!35) & *(blue!35) & *(blue!35) & *(blue!35) & *(blue!35) \\
*(black!0) & *(black!0) & *(red!40) & *(red!40) & *(red!40) & *(red!40) & *(red!40) & *(red!40) & *(red!40) & *(red!40) & *(blue!35) & *(blue!35) & *(blue!35) & *(blue!35) & *(blue!35) & *(blue!35) & *(blue!35) & *(blue!16) & & \\
*(black!0) & *(black!0) & *(black!0) & *(red!40) & *(red!40) & *(red!40) & *(red!40) & *(red!40) & *(red!40) & *(red!40) & *(purple!80) & *(blue!35) & *(blue!35) & *(blue!35) & *(blue!35) & *(blue!35) & *(blue!35) & *(blue!35) & & \\
*(black!0) & *(black!0) & *(black!0) & *(red!15) & *(red!15) & *(red!40) & *(red!40) & *(red!40) & *(red!40) & *(red!40) & *(purple!80) & *(blue!35) & *(blue!35) & *(blue!35) & *(blue!35) & *(blue!35) & *(blue!35) & *(blue!35) & & \\
*(black!0) & *(black!0) & *(black!0) & *(red!40) & *(red!40) & *(red!40) & *(red!40) & *(red!40) & *(purple!80) & *(purple!80) & *(blue!35) & *(blue!35) & *(blue!35) & *(blue!35) & *(blue!35) & *(blue!35) & *(blue!35) & & & \\
*(black!0) & *(black!0) & *(black!0) & *(red!40) & *(red!15) & *(red!15) & *(red!15) & *(red!40) & *(red!40) & *(red!40) & *(black!0) & *(black!0) & *(blue!35) & *(blue!35) & *(blue!35) & *(blue!35) & *(blue!35) & & & \\
*(black!0) & *(black!0) & *(black!0) & *(black!0) & *(black!0) & *(black!0) & *(black!0) & *(black!0) & *(black!0) & *(black!0) & *(black!0) & *(black!0) & *(black!0) & *(black!0) & *(blue!35) & *(blue!35) & & & & \\
\end{ytableau}
\caption{\small Ties breaking in Convex-{\fontfamily{lmss}\selectfont GRED}.}
\label{fig:algo_conv_ties}
\end{figure}

As can be easily noted, the Basic-{\fontfamily{lmss}\selectfont GRED} algorithm does not make use of the convexity assumption of the detected regions. Indeed, as mentioned earlier, the convexity prior was assumed for the ease of theoretical treatment. Despite being quite a natural condition in many scenarios, in some cases this assumption may not hold or such information on the global geometrical properties of the region may be not available to the engineer.

Let us describe a simple way of obtaining a Convex version of the {\fontfamily{lmss}\selectfont GRED} algorithm. Basically, we are going to \textit{convexify} the shapes obtained by the Basic-{\fontfamily{lmss}\selectfont GRED} after it halts. Recall that we call a polyomino convex if it is column- and row-convex, or in other words if every vertical or horizontal line crossing it does not see \textit{gaps} in it. Therefore, the natural approach to convexification is just to fill all such gaps with cells belonging to the area, see Figure \ref{fig:algo_conv} for an illustration. Using the same reasoning as in the previous section, we can justify that this Convex-{\fontfamily{lmss}\selectfont GRED} algorithm is also consistent. The convexification procedure can be performed on every iteration or at the end. In Section \ref{sec:num_test_convex}, we compare Convex-{\fontfamily{lmss}\selectfont GRED} with convexification on every iteration with Basic-{\fontfamily{lmss}\selectfont GRED} using numerical simulations.

In some cases, we may get ties as in Figure \ref{fig:algo_conv_ties} meaning that the detected cells are arranged in such a way that if we convexify both polyominoes, the latter will overlap. Such ties may be broken randomly or by dividing the overlapping area between the regions arbitrarily. This will not affect the consistency of the algorithm.

\section{Numerical Simulations}
\label{sec:num_res}

\subsection{Basic-{\fontfamily{lmss}\selectfont GRED}}
In this section, we illustrate the power of the proposed algorithm using synthetic simulations. To this end, we have created a graph on $p=20000$ vertices in the following manner. We took a $2 \times 2$ square on the plane and sequentially generated realizations of a random variable uniformly distributed over the area of the square. Each time the point was generated, we checked that the closest (out of already accept) vertices is not closer to it than $w_{\min} = 0.003$ (see footnote \ref{ftn:1} above). After $p$ points were obtained, this process was terminated and the points were listed in the order of their appearance on the map. At the next stage, every point was connected sequentially to at most $d = 4$ of its neighbors located not further than $w_{\max} = 5w_{\min}$ of it. This way we obtained a graph where the number of vertices with the degree $4$ was $19186$, $3$ - $634$, $2$ - $147$, $1$ - $26$, and the remaining $7$ vertices were isolated. Then the original $2 \times 2$ square was split into $4$ unit squares as illustrated in the upper left box of Figure \ref{fig:gred_detection_iter}. The $4$ regions were assigned different coupling parameters at random from the set $\[0.04\,\; 0.056\,\; 0.069\,\; 0.08\,\]$. Any edge connecting vertices from the same region obtained the weight corresponding to that region, the edge connecting vertices from different regions (crossing the region boundary) received the average coupling parameter. After the underlying network was created, we generated $n=5000$ i.i.d.\ samples from the joint distribution and fed the {\fontfamily{lmss}\selectfont GRED} algorithm with obtained data as the input. In this experiment we assumed the reference point and the lattice direction to be known to the algorithm. The sizes of the test squares $\tau^{(t)}$ were chosen in the following manner. We set $\tau^{(0)} = 0.16$ and then halved it at every successive iteration. The letters in the brackets on every iteration count the subiterations, where the size of the test squares $\tau^{(t)}$ is kept fixed but we allow the regions to grow. This happens until no more cells can be added to the regions, see Algorithm \ref{algo}. Figure \ref{fig:gred_detection_iter} shows the first $2$ iterations of the the algorithm.

\renewcommand\thesubfigure{}

\begin{figure}[ht!]
\setcounter{subfigure}{-1}
        \subfigure[Original network]{%
            \includegraphics[width=0.33\textwidth]{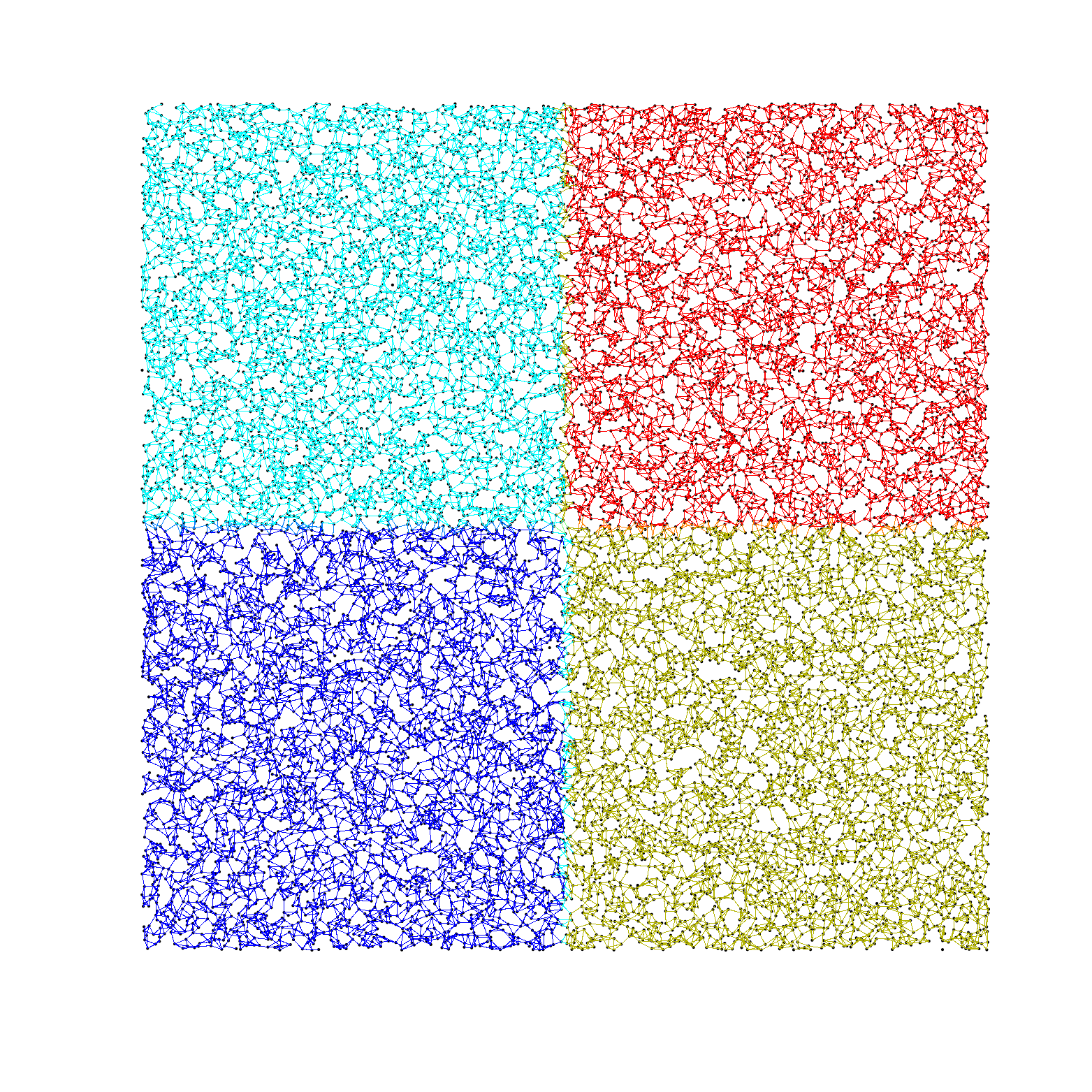}
        }%
        \subfigure[Initial seed detection]{%
            \includegraphics[width=0.33\textwidth]{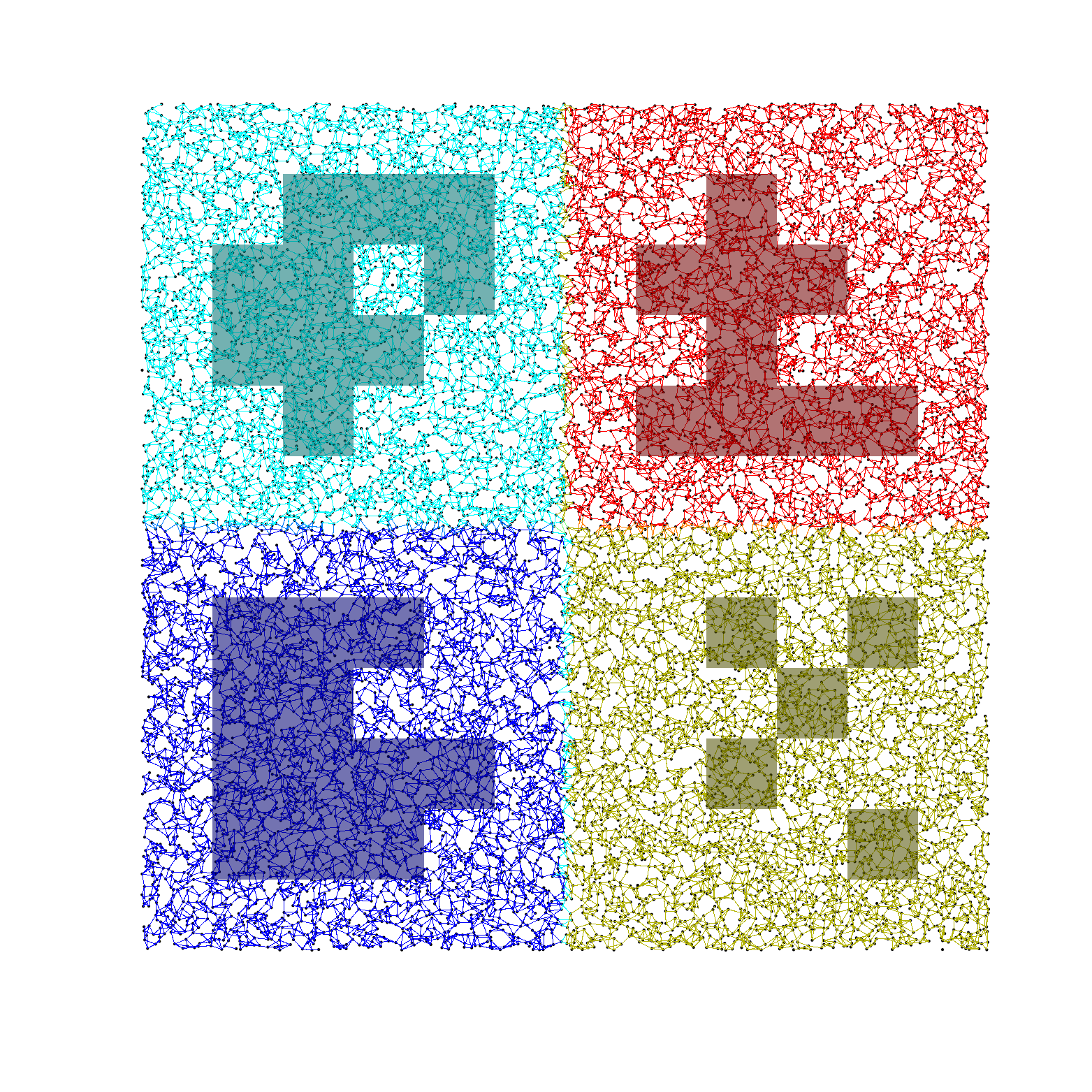}
        }%
        \subfigure[Iteration 1 (a)]{%
            \includegraphics[width=0.33\textwidth]{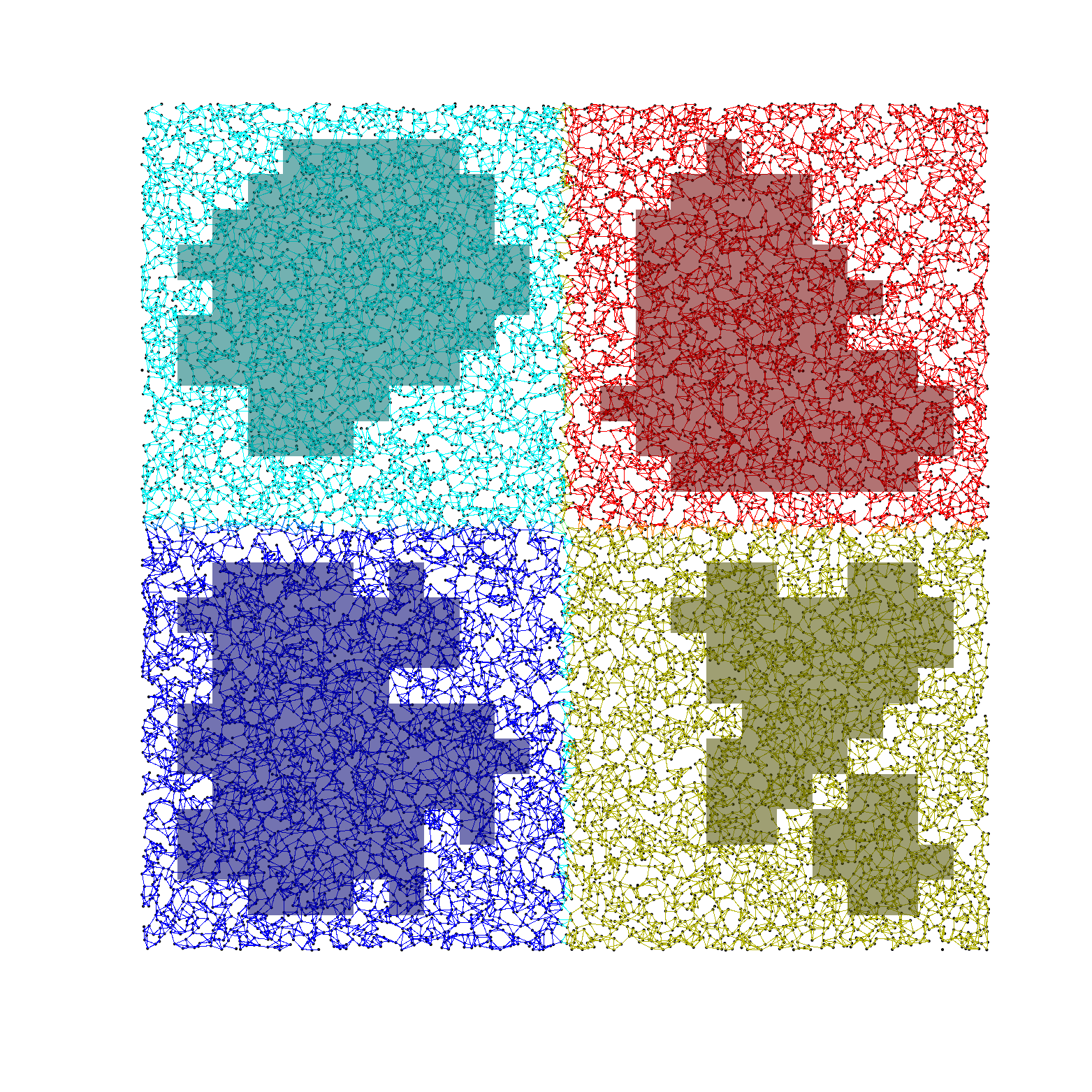}
        }\\ 
        \subfigure[Iteration 1 (b)]{%
            \includegraphics[width=0.33\textwidth]{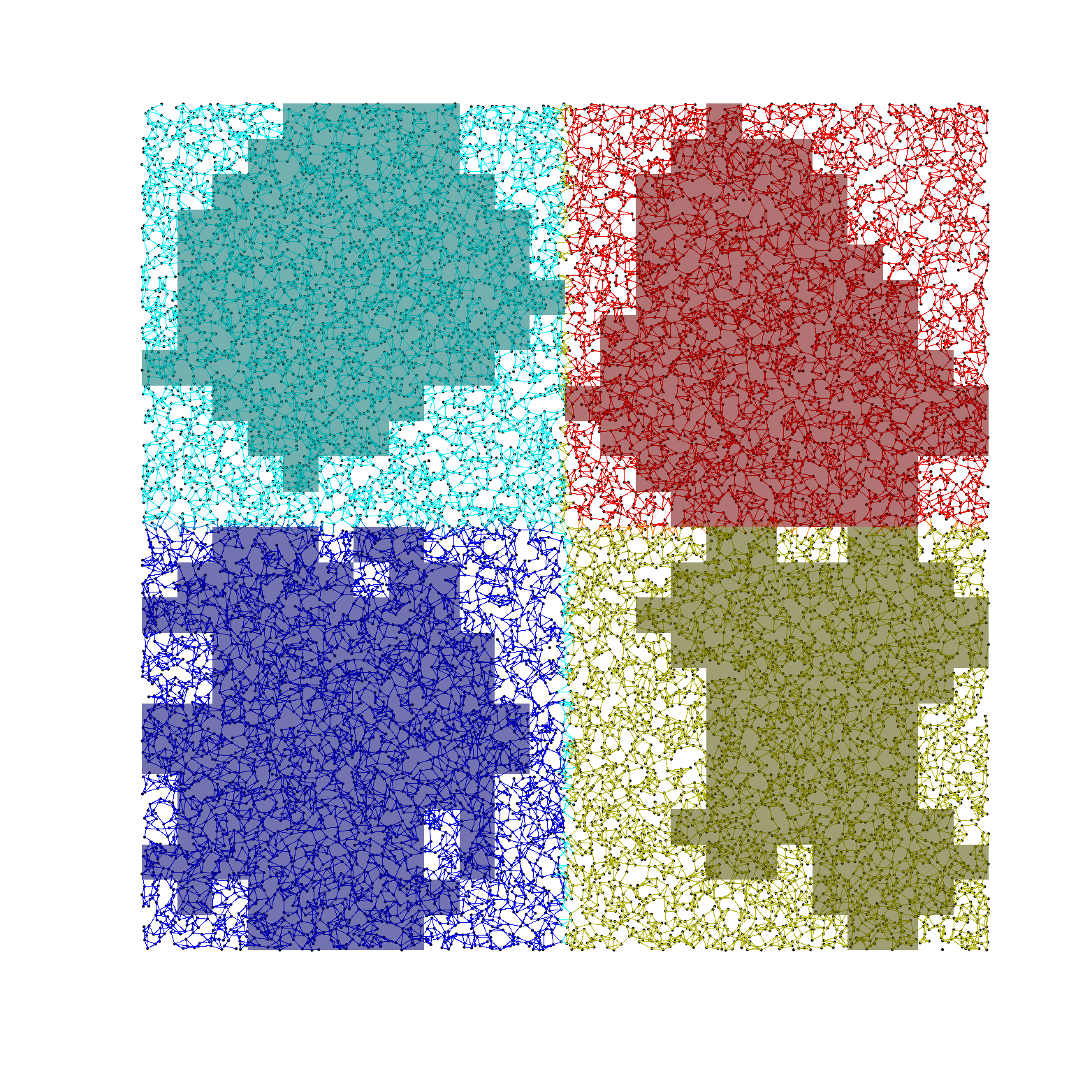}
        }%
        \subfigure[Iteration 1 (c)]{%
            \includegraphics[width=0.33\textwidth]{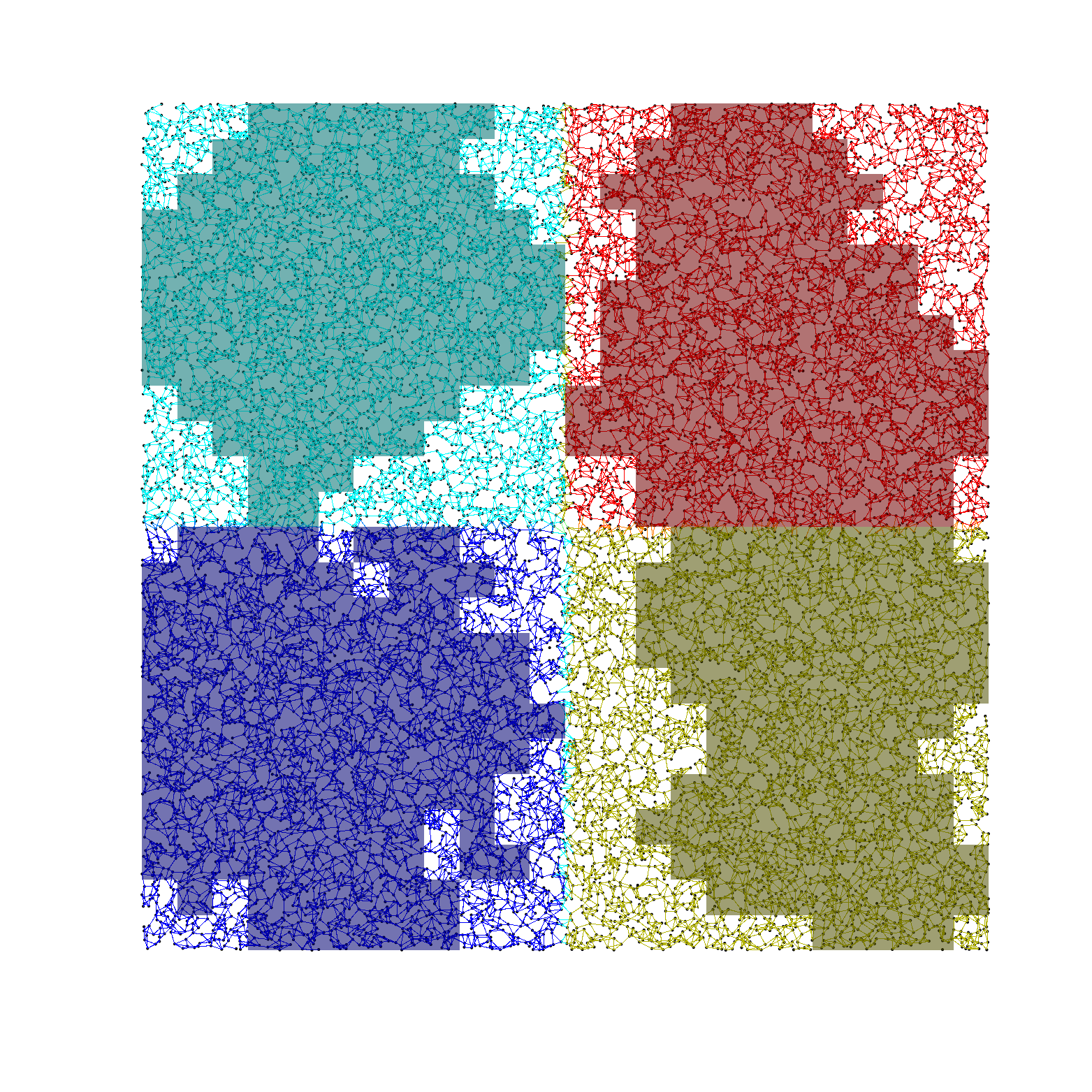}
        }%
        \subfigure[Iteration 2 (a)]{%
            \includegraphics[width=0.33\textwidth]{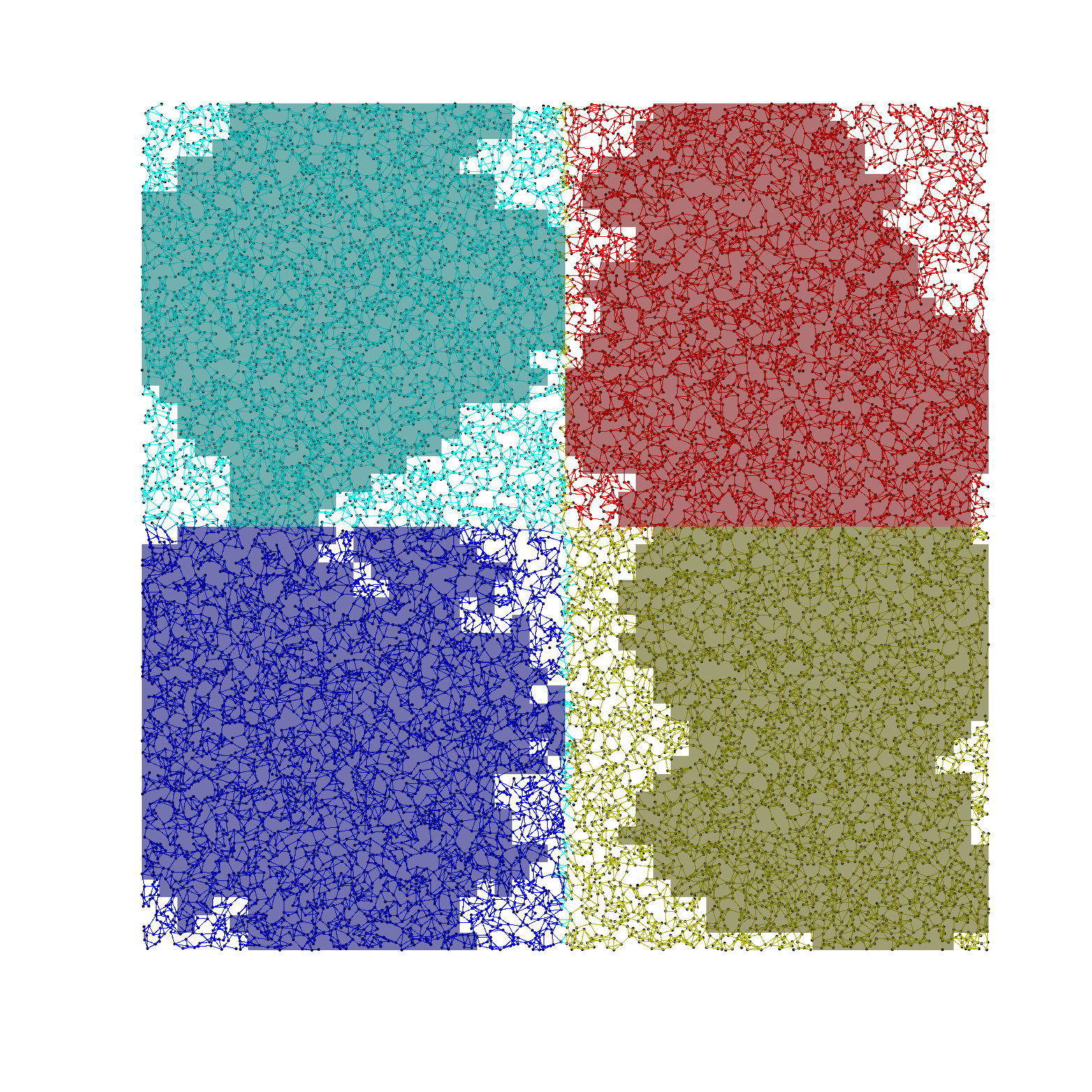}
        }\\ 
        \subfigure[Iteration 2 (b)]{%
            \includegraphics[width=0.33\textwidth]{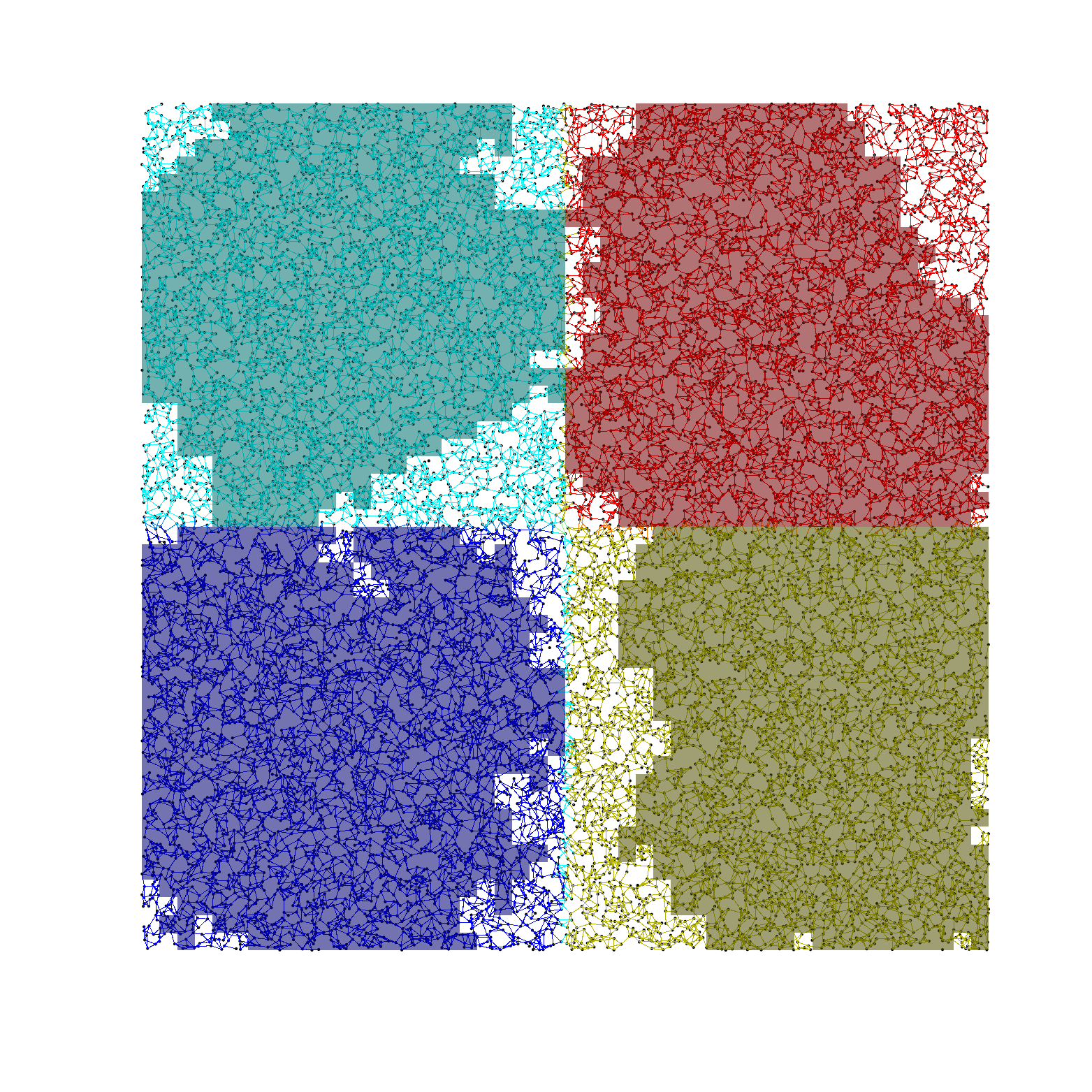}
        }%
        \subfigure[Iteration 2 (c)]{%
            \includegraphics[width=0.33\textwidth]{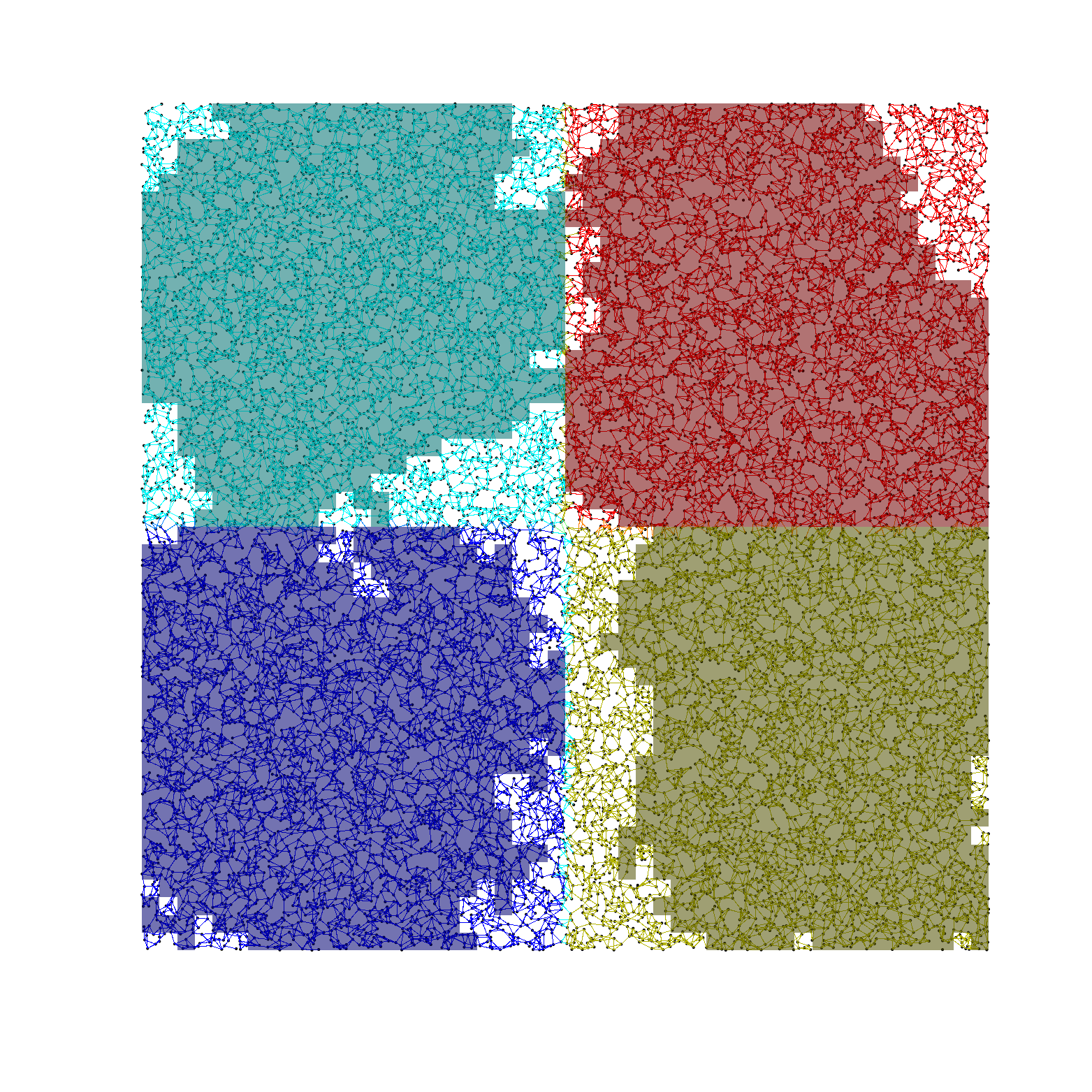}
        }%
        \subfigure[Iteration 2 (d)]{%
            \includegraphics[width=0.33\textwidth]{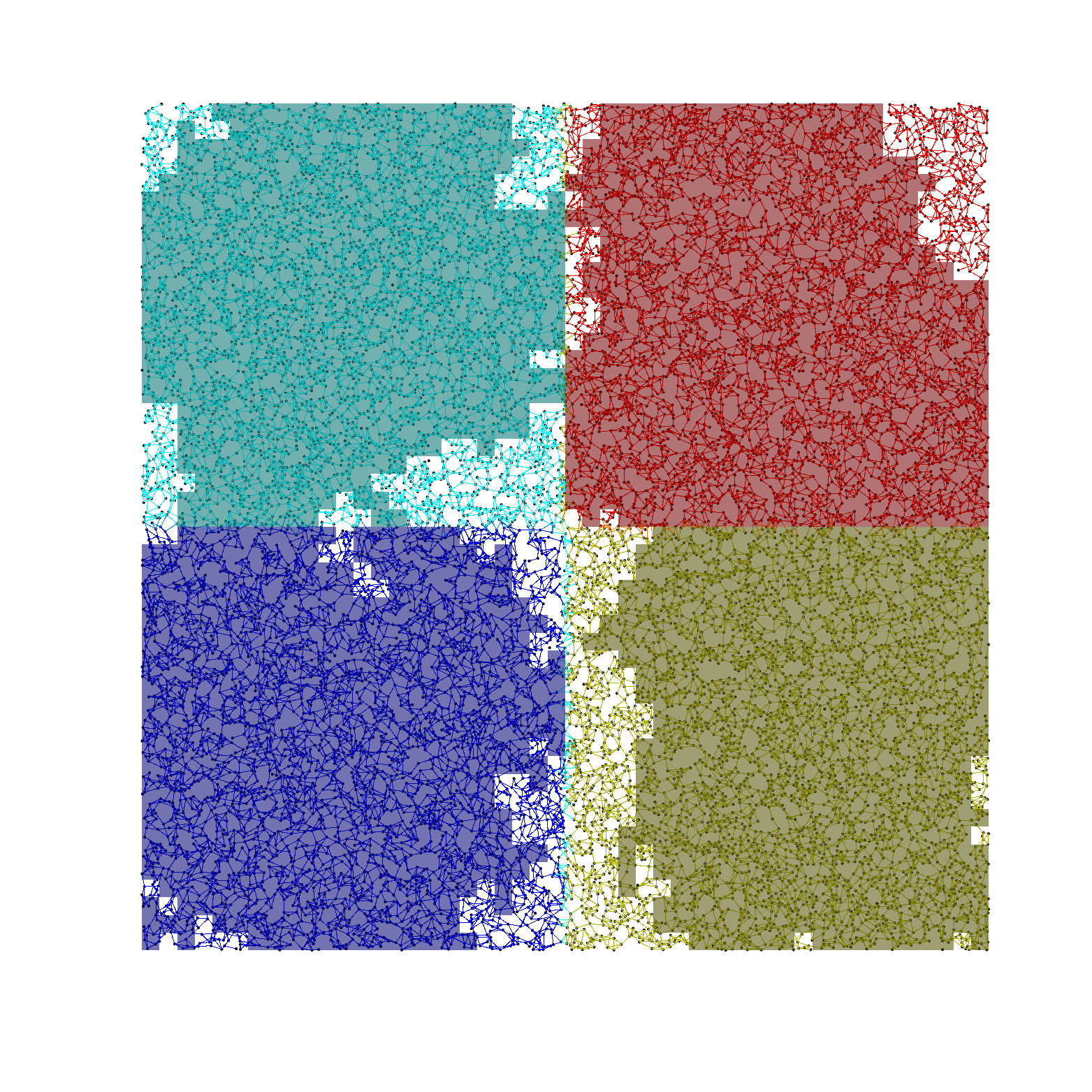}
        }%
    \caption{%
        The Basic-{\fontfamily{lmss}\selectfont GRED} algorithm iterative detection process.
     }%
   \label{fig:gred_detection_iter}
\end{figure}

\begin{figure}[ht!]
\setcounter{subfigure}{-1}
        \subfigure[Original network]{%
            \includegraphics[width=0.33\textwidth]{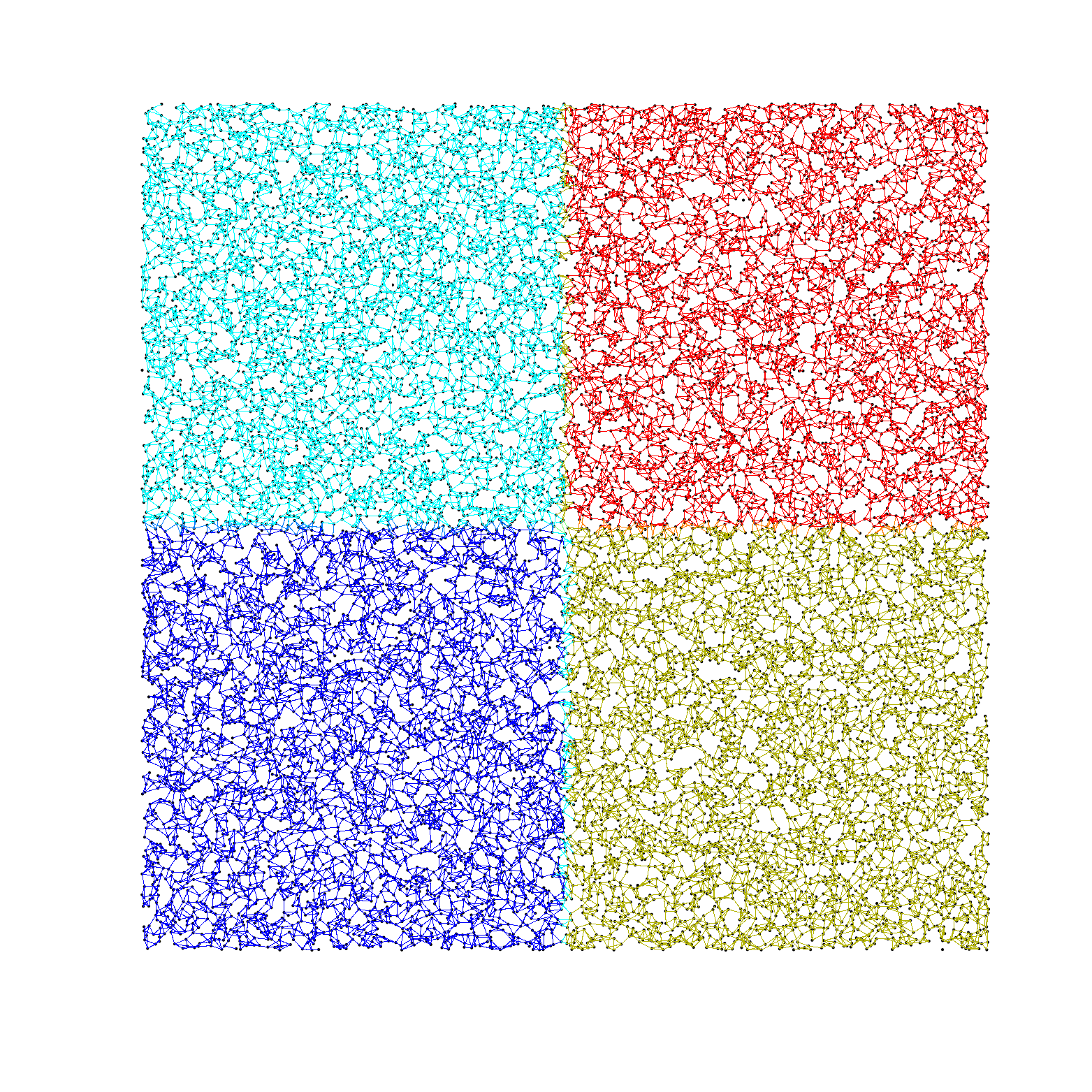}
        }%
        \subfigure[Initial seed detection]{%
            \includegraphics[width=0.33\textwidth]{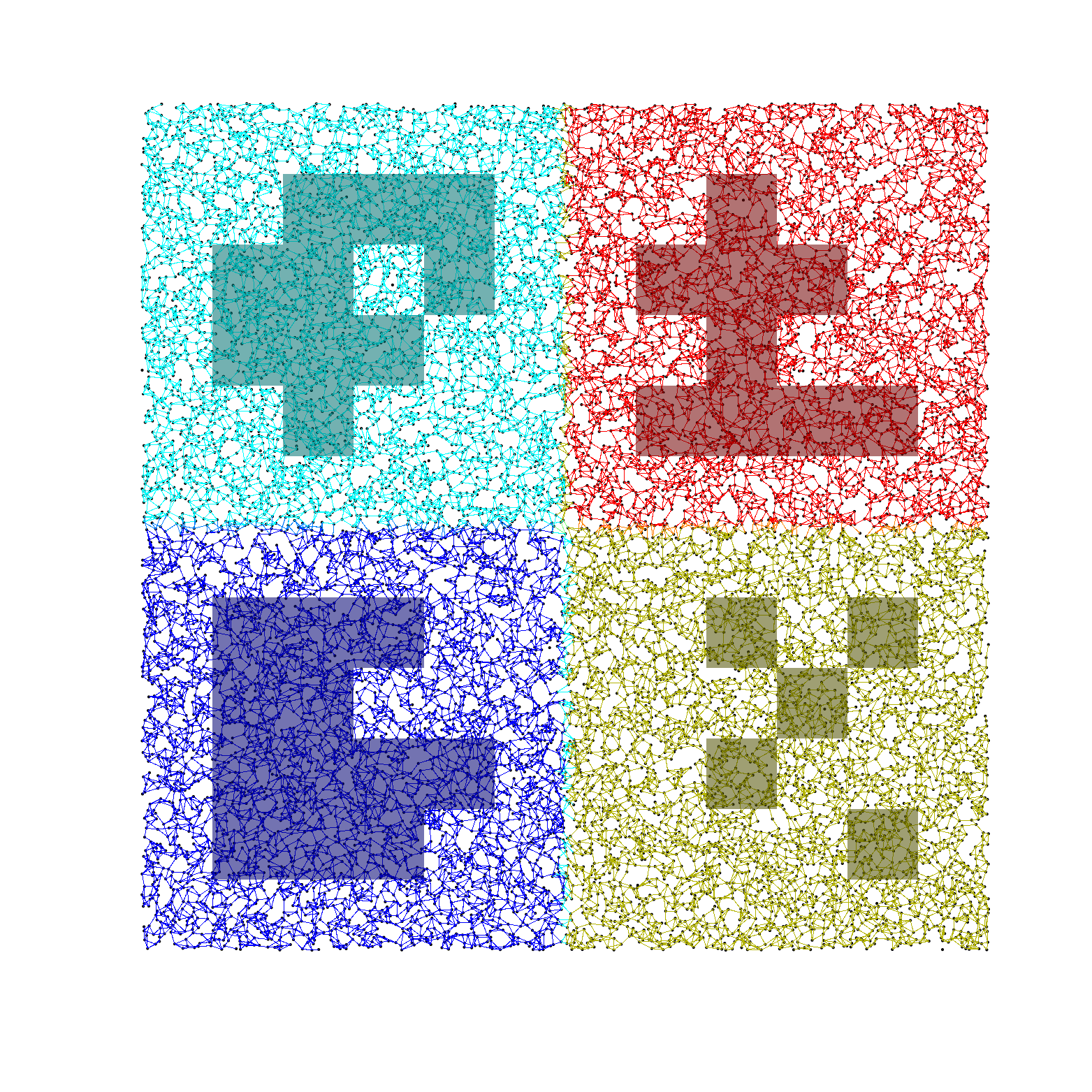}
        }%
        \subfigure[Iteration 1 (a)]{%
            \includegraphics[width=0.33\textwidth]{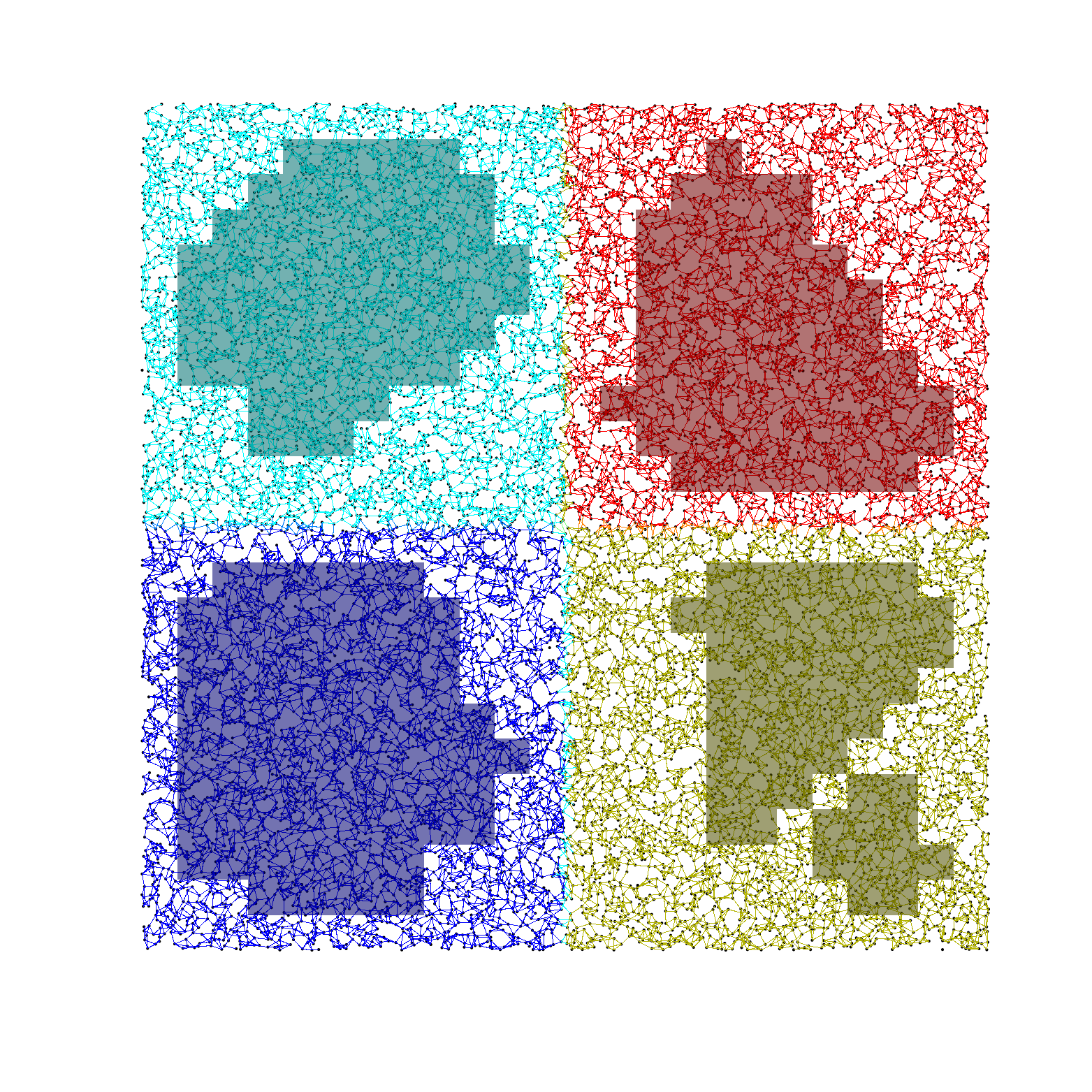}
        }\\ 
        \subfigure[Iteration 1 (b)]{%
            \includegraphics[width=0.33\textwidth]{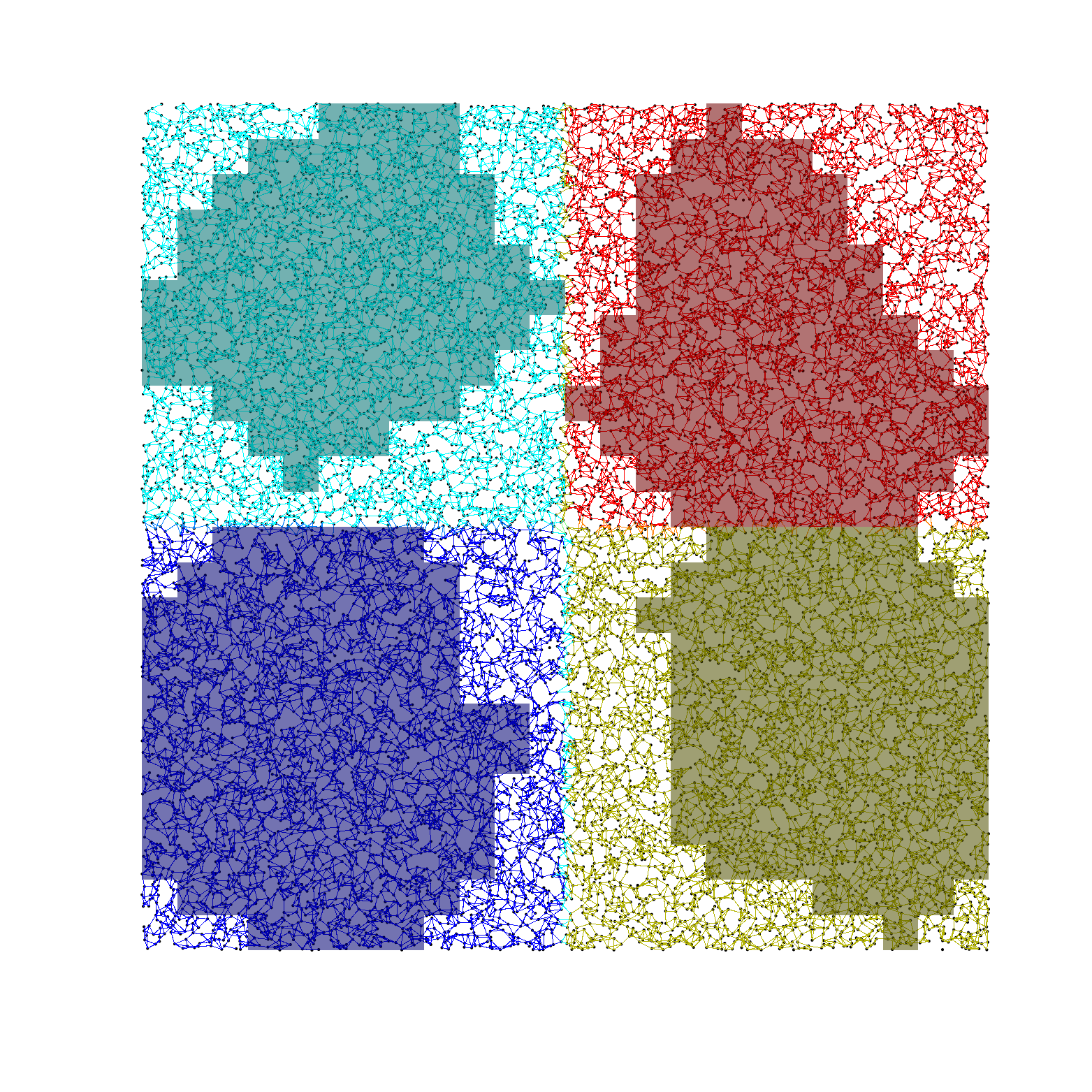}
        }%
        \subfigure[Iteration 1 (c)]{%
            \includegraphics[width=0.33\textwidth]{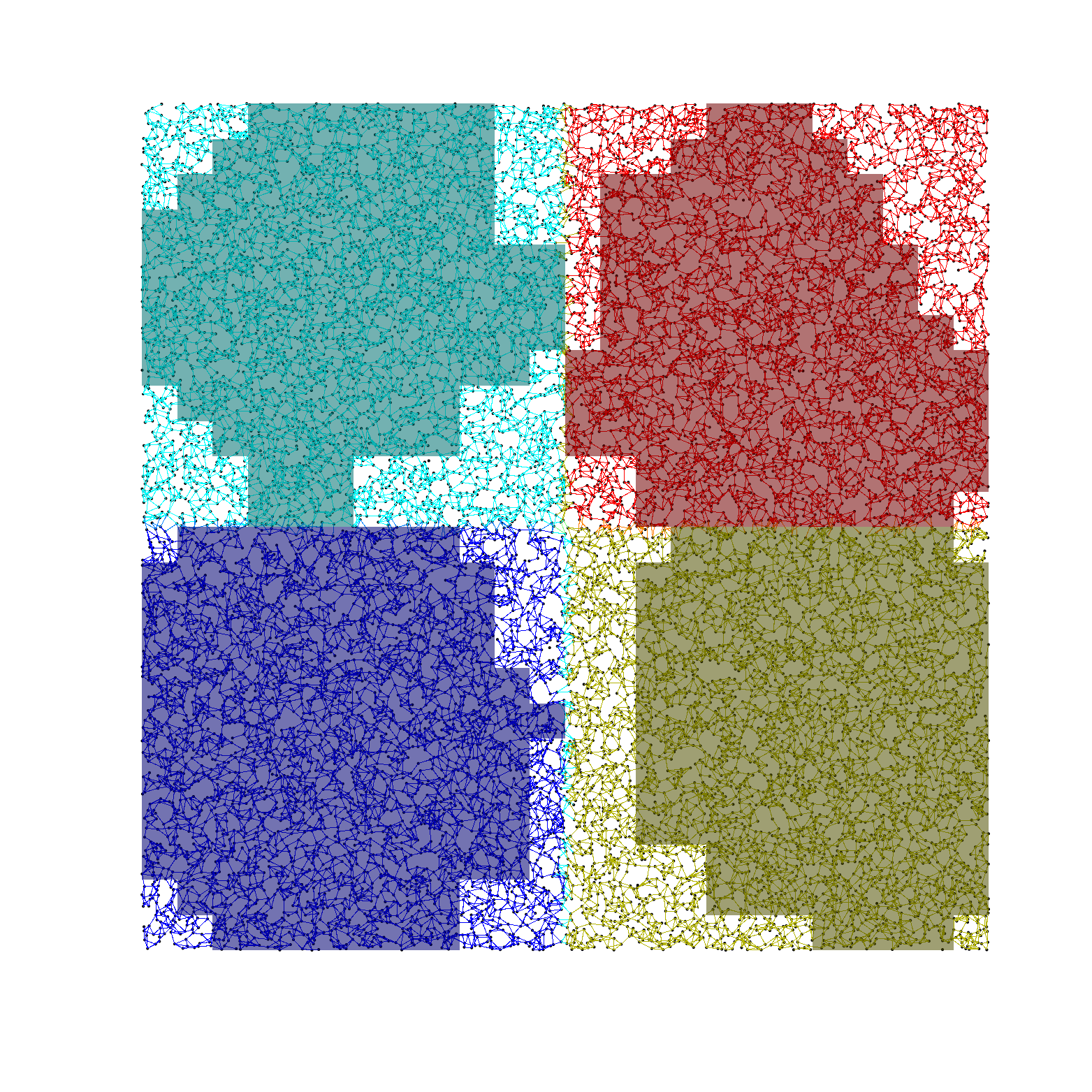}
        }%
        \subfigure[Iteration 1 (d)]{%
            \includegraphics[width=0.33\textwidth]{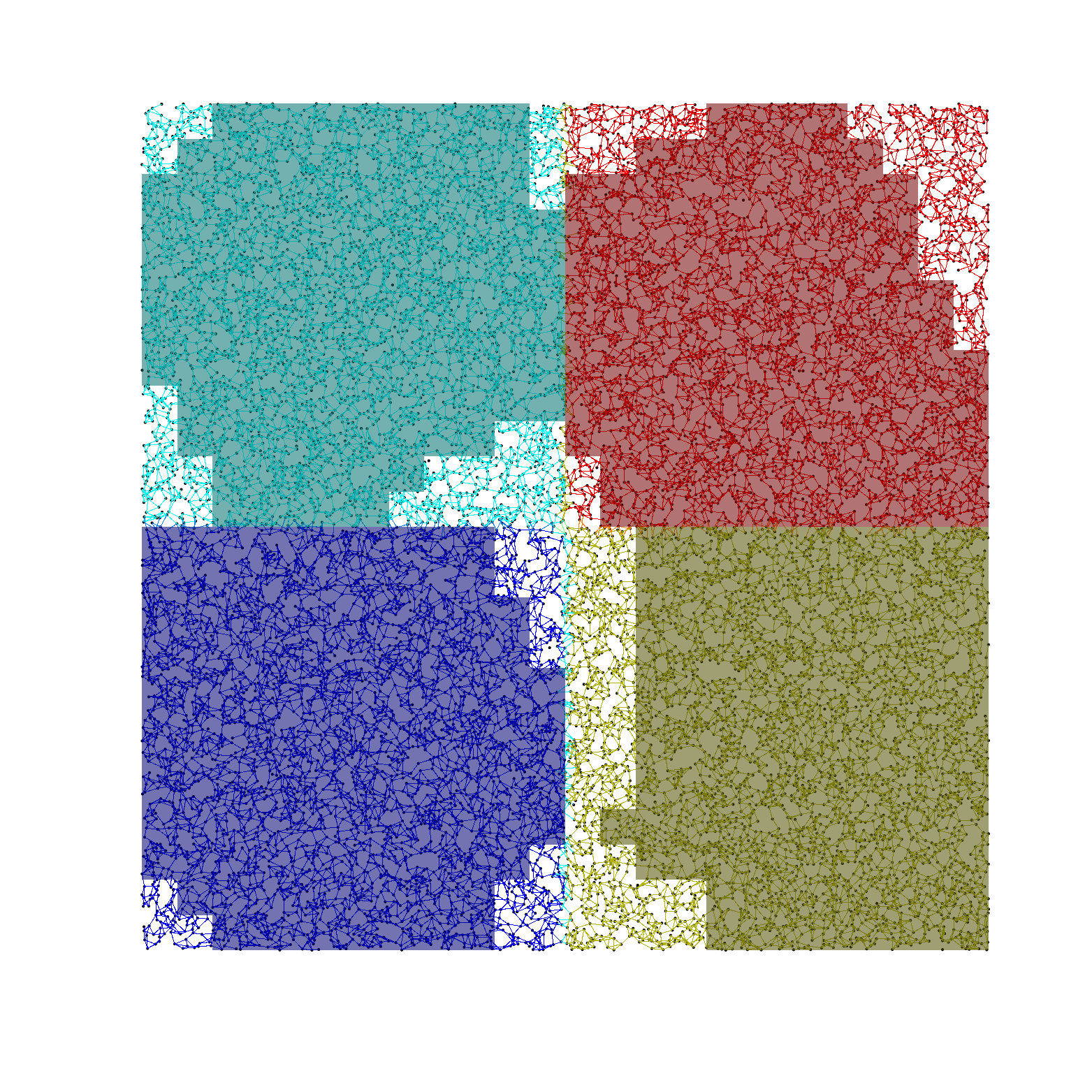}
        }\\ 
        \subfigure[Iteration 2 (a)]{%
            \includegraphics[width=0.33\textwidth]{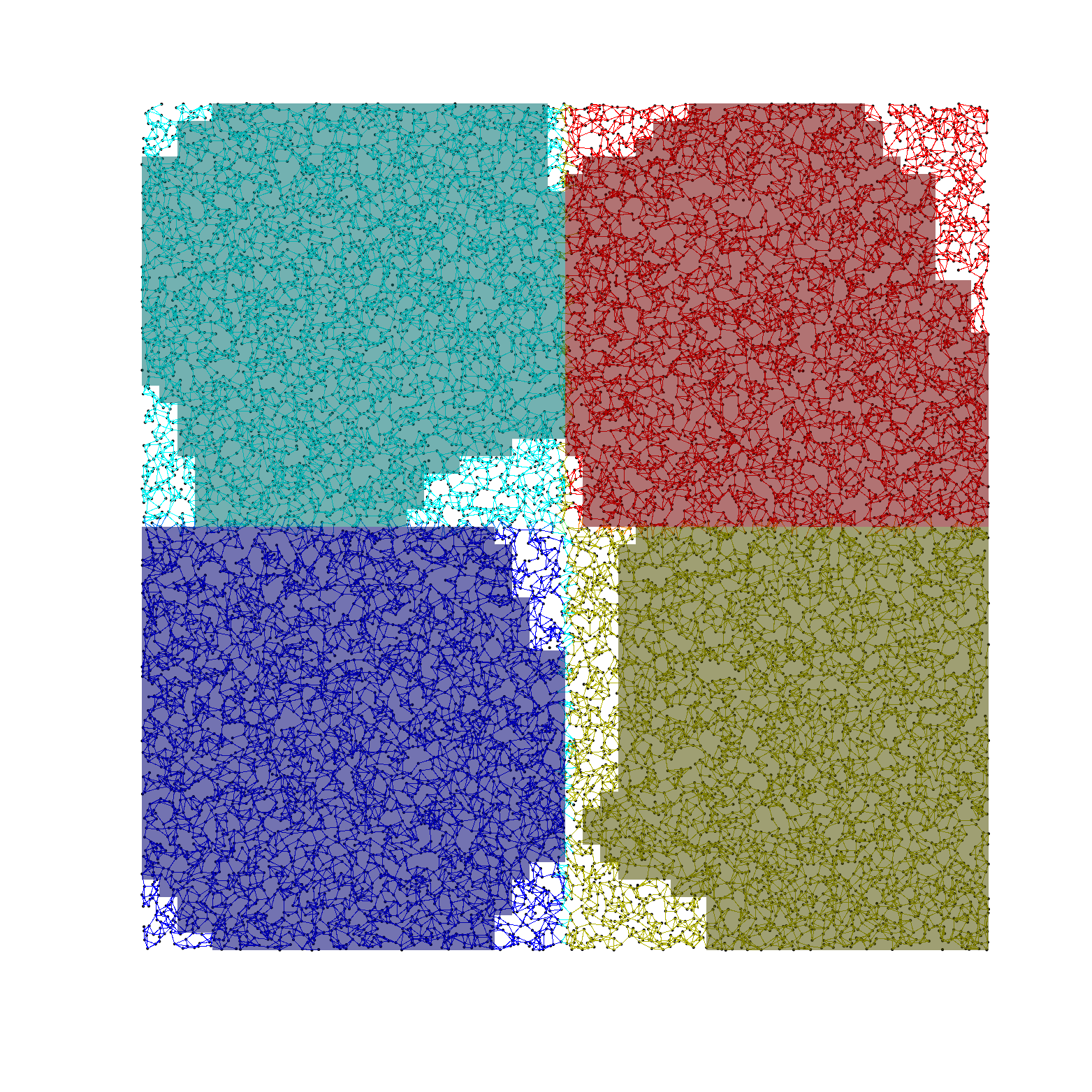}
        }%
        \subfigure[Iteration 2 (b)]{%
            \includegraphics[width=0.33\textwidth]{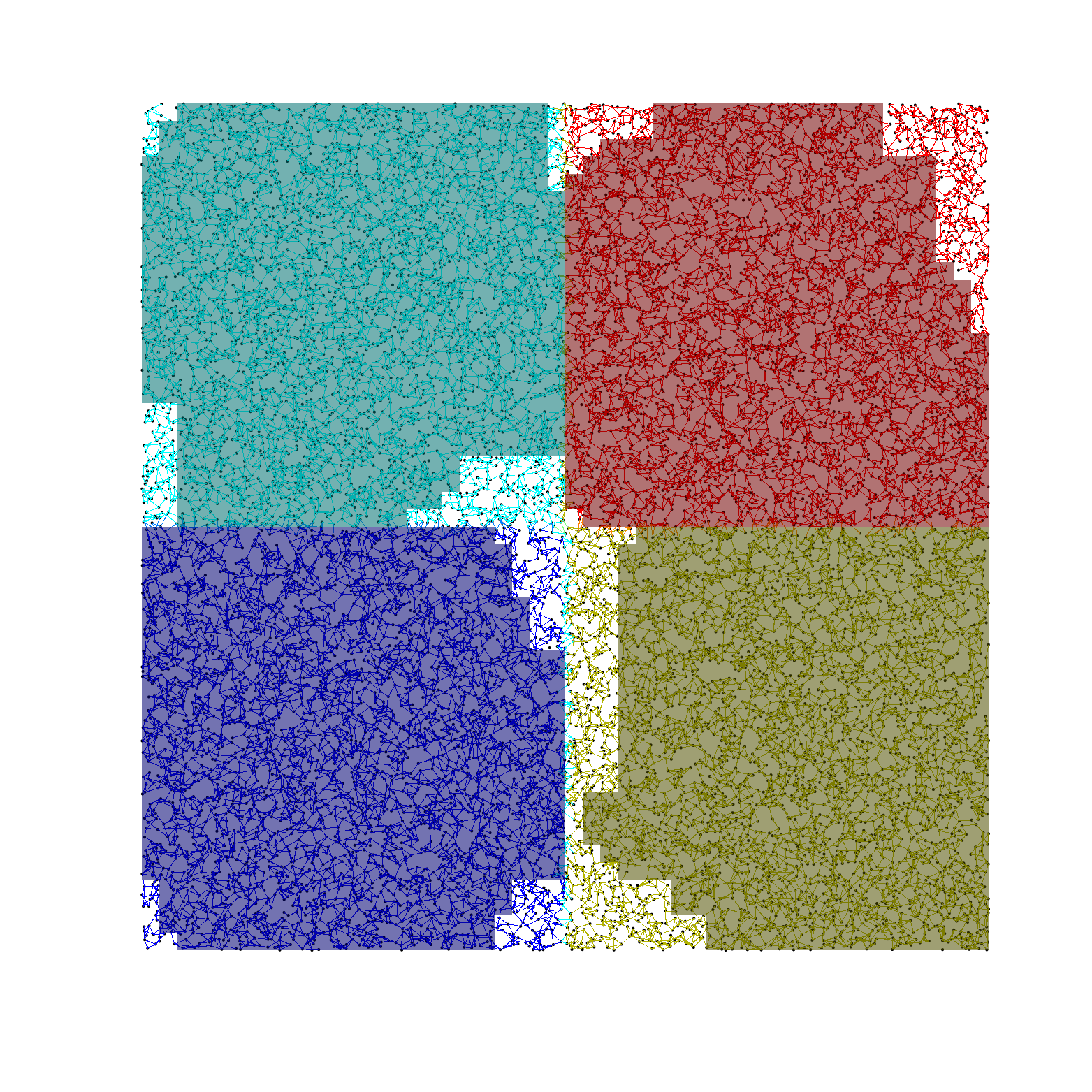}
        }%
        \subfigure[Iteration 2 (c)]{%
            \includegraphics[width=0.33\textwidth]{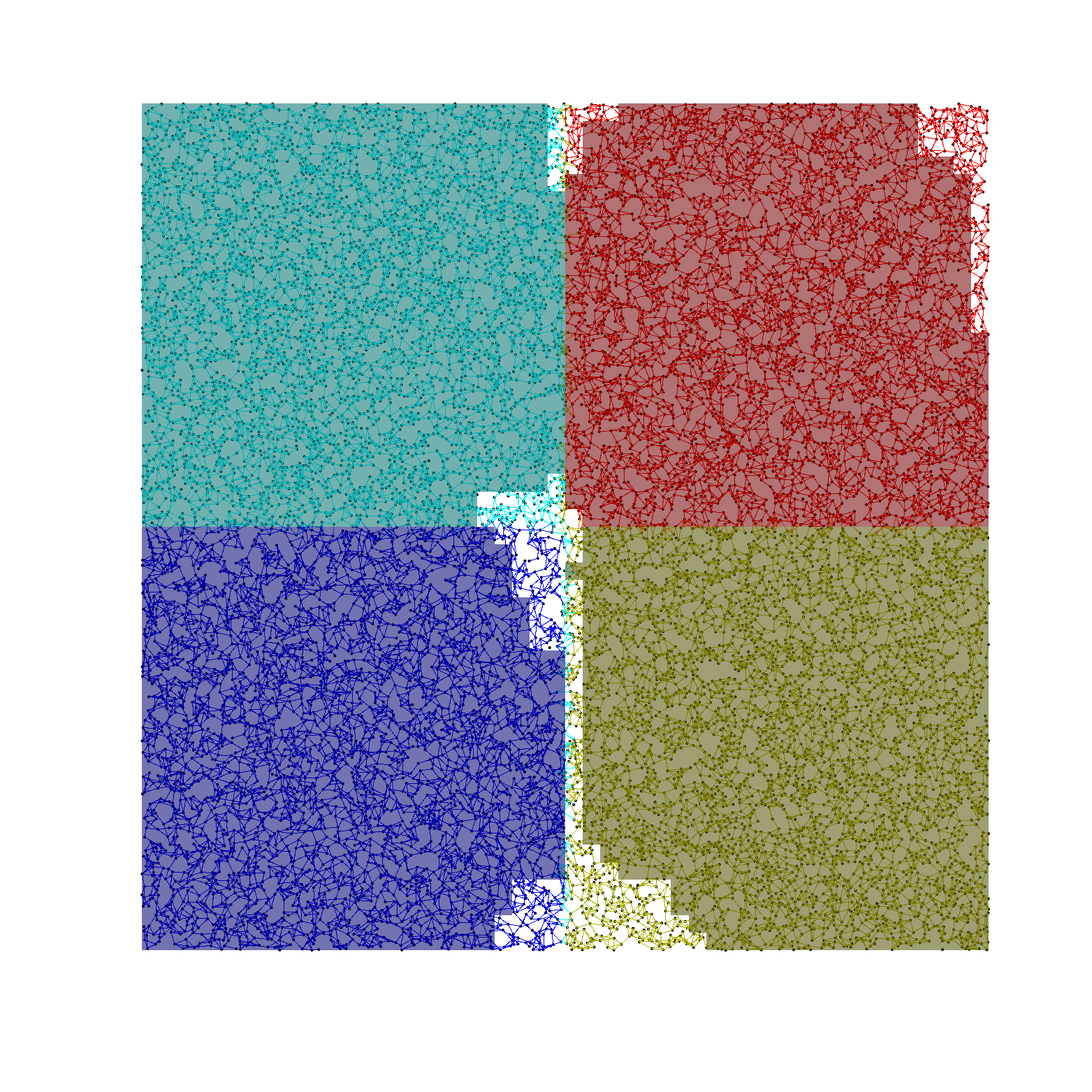}
        }%
    \caption{%
        The Convex-{\fontfamily{lmss}\selectfont GRED} with convexification on every iteration.
     }%
   \label{fig:gred_detection_iter_conv}
\end{figure}

\begin{figure}
\centering
\includegraphics[width=16cm]{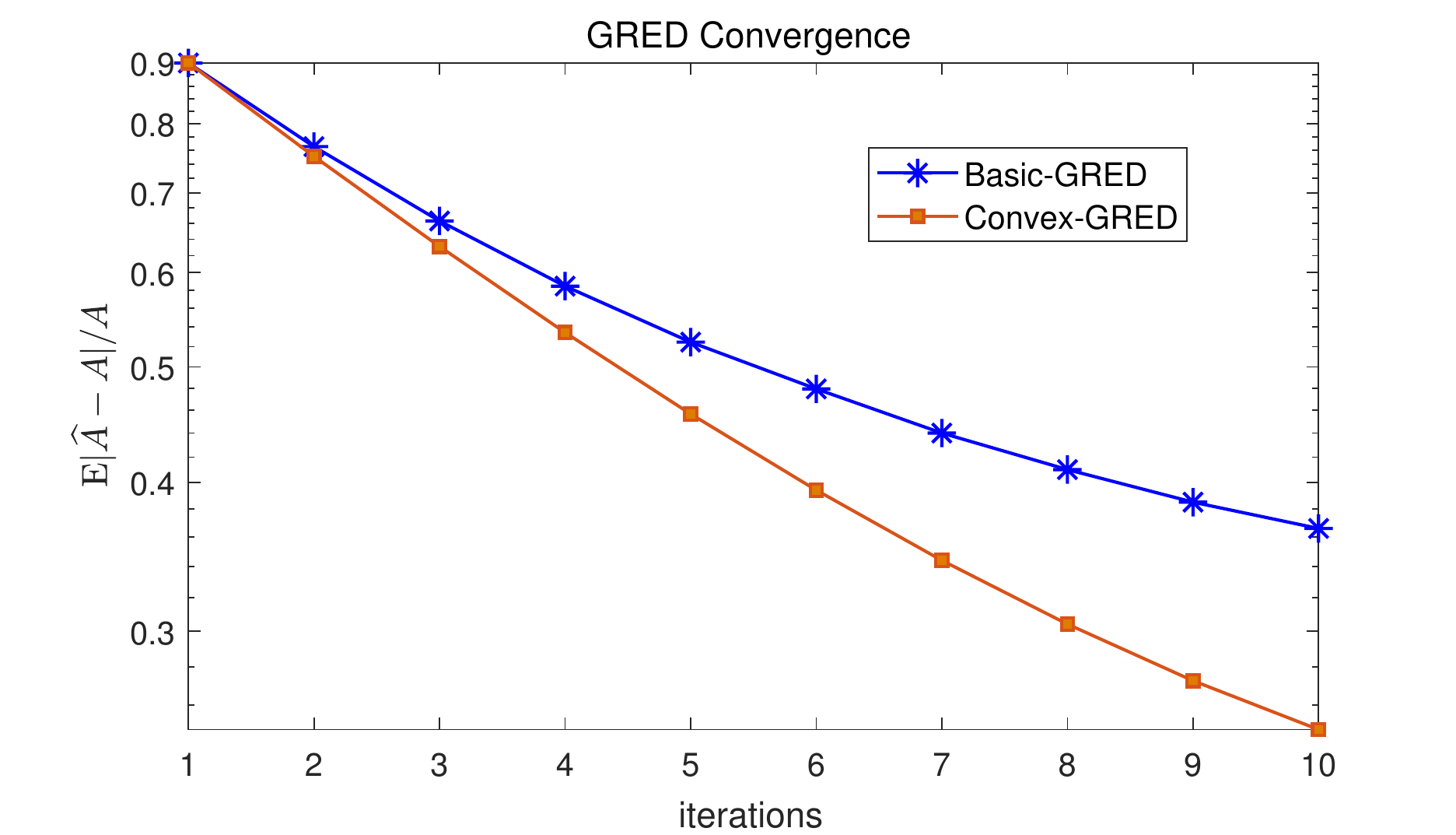}
\caption{\small Empirical expected area error, 100 iterations of the {\fontfamily{lmss}\selectfont GRED} algorithm.}
\label{fig:empir_error}
\end{figure}

\subsection{Convex-{\fontfamily{lmss}\selectfont GRED}}
\label{sec:num_test_convex}
Figure \ref{fig:gred_detection_iter_conv} illustrates the steps of the Convex-{\fontfamily{lmss}\selectfont GRED} algorithm in exactly the same setup as in the previous section. Convexification is performed on every iteration. We observe that the convergence to the true regions is much faster which is natural since the modified version of the algorithm utilizes more information of the prior knowledge. To compare the two versions of the {\fontfamily{lmss}\selectfont GRED} algorithm, in Figure \ref{fig:empir_error} we plot two graphs demonstrating the convergence of the average area of the detected regions $\hat{A}$ to its true value $A$.

In real word data analysis, the values of the threshold $\zeta$, degree $d$, and the cell sizes $\tau^{(0)}$ are usually unknown and must be estimated from the data itself. Usually, such approximations are obtained through some prior knowledge on the network at hand. The sizes $\tau^{(i)},\; i=0,\dots$ are chosen in such a way that the number of vertices inside the cells would be significantly larger than the number of vertices close to the cell boundaries. It is important to note that the precise value of $d$ is not critical for the region detection. Indeed, due to the way the parameter estimates are defined in (\ref{eq:lim_err_est}), the choice of $d$ will only scale the estimates which does not affect the region detection since the values of the threshold adapt accordingly. Therefore, we choose $d$ approximately based on some a priori known connectivity properties of the network. 
The threshold $\zeta$ is chosen based on the resolution $\delta$ we want to achieve through formula (\ref{eq:threshold_val}) or using the anticipated number of regions that we want to discover.

Remarkably, the number of i.i.d.\ snapshots utilized for graphical model selection in real world data by the authors of \cite{ravikumar2010high,anandkumar2012high} is greater than the dimension in all of their experiments, despite the fact that the dependence on the dimension is logarithmic.


\section{Conclusion}
\label{sec:concl}
In this paper, we consider the problem of model selection in Gaussian Markov fields when the number of samples is not sufficient for the consistent detection of all the edges in the graph. The classical results \cite{bresler2015efficiently, santhanam2012information, anandkumar2012high} require the number of samples to grow at least as fast as $n = c \log p$ to ensure reliable detection of the network structure. In addition, the constant of proportionality $c$ depending on the graph parameters may be prohibitively large making the application of the model selection algorithms impossible in practice. However, in many high-dimensional real world applications knowledge of the entire network is not necessary, and what is more important is the distribution of the edge parameters over the graph. By considering networks embedded into two dimensional Euclidean spaces, and assuming that they can be decomposed into a number of regular regions with similar coupling parameters, we develop a novel framework enabling learning of the region structure with less samples. Using rigorous information-theoretic approach, we derive tight necessary sample complexity bounds demonstrating that even bounded number of samples may be enough for consistent recovery of the graph regions. We also propose a simple greedy algorithm {\fontfamily{lmss}\selectfont GRED} capable of efficiently, reliably and quickly partitioning the graph into regions and rigorously analyze its performance bounds. Here too, we show that our algorithm can consistently learn the regions with bounded number of samples. 

Our current work focuses on the three dimensional generalization of the developed machinery and its application to the study of the brain activity in animals \textit{in vivo}. As mentioned earlier, whole-brain functional imaging has become available only very recently due to the pioneering works of the HHMI's Janelia Research team \cite{ahrens2013whole, dunn2016brain}. Using scarce number of available snapshots of the entire brain, our goal is to segment the latter into regions based on their connectivity properties. This will allow us to distinguish between healthy and damaged brain areas based on their functionality and will enable reliable diagnosis of human brain diseases in early stages.

\appendices
\section{Information-Theoretic Lower Bound for Model Selection}
\label{app:graph_detect_inf_thoer_bound}
Following the approach developed in \cite{anandkumar2012high}, in this Section we prove Theorem \ref{th:nec_cond_complete_graph} using Fano's inequality. We shall need a number of auxiliary results stated below.

Denote the Kullback-Leibler (KL) divergence between two probability measures $\mathbb{P}_{\theta_i}$ and $\mathbb{P}_{\theta_j}$ associated with different parameters $\theta_i \neq \theta_j$ by
\begin{equation}
D(\theta_i\;\|\;\theta_j) = \mathbb{E}_{\theta_i} \[\log \frac{\mathbb{P}_{\theta_i}}{\mathbb{P}_{\theta_j}}\].
\end{equation}
Let us also introduce a symmetrized analog of the KL-divergence,
\begin{equation}
\label{eq:sym_kl_div}
S(\theta_i\;\|\;\theta_j) = D\(\theta_i\;\|\;\theta_j\) + D\(\theta_j\;\|\;\theta_i\).
\end{equation}

\begin{definition}
\label{def:delta-unrel}
Let $\Theta = \{\theta_1,\dots,\theta_M\}$ be a family of models and $\x_i \sim \mathbb{P}_{\theta_j}$ be i.i.d.\ for some $\theta_j \in \Theta$. Denote by
\begin{equation}
\psi \colon \{\mathcal{X}^n\} \to \Theta
\end{equation}
a classification function (decoder) where $\mathcal{X}^n = \{\x_1,\dots,\x_n\}$. We say that $\psi$ is $\delta$-unreliable if
\begin{equation}
\max_m \mathbb{P}_{\theta_j} \[\psi(\mathcal{X}^n) \neq i\] \geqslant \delta - \frac{1}{M}.
\end{equation}
\end{definition}

\begin{lemma}[Fano's Inequality, \cite{santhanam2012information, yu1997assouad}]
\label{lem:fano_ineq_lemma}
In the setup of Definition \ref{def:delta-unrel}, for all $\delta \in (0,1)$ any of the following conditions implies that any decoder over $\Theta$ is $\delta$-unreliable.
\begin{enumerate}
\item \label{enum:fano_1} The number of i.i.d.\ samples is bounded as
\begin{equation}
n < (1-\delta)\frac{\log (M)}{I(\mathcal{X}^1;\mathcal{U}(\Theta))},
\end{equation}
where $\mathcal{U}(\Theta)$ is the uniform distribution over $\Theta$.
\item \label{enum:fano_2} The number of i.i.d.\ samples is bounded as
\begin{equation}
\label{eq:fano_2}
n < (1-\delta)\frac{\log (M)}{\frac{2}{M^2}\sum_{i=1}^M\sum_{j=i+1}^M S(\theta_i\;\|\;\theta_j)}.
\end{equation}
\end{enumerate}
\end{lemma}

\begin{lemma}[Upper Bound on Differential Entropy of Mixture, Lemma 20 from \cite{anandkumar2012high}]
\label{lem:upp_b_diff_e}
Using the notations from Section \ref{sec:assm},
\begin{equation}
h\(\mathcal{X}^1\) \leqslant \frac{p}{2} \log_2\(\frac{2\pi e}{1-d\bar{\theta}}\),
\end{equation}
where $h$ is the differential entropy.
\end{lemma}

\begin{lemma}[Lower Bound on Conditional Differential Entropy, Lemma 21 from \cite{anandkumar2012high}]
\label{lem:low_b_diff_e}
In the above notation,
\begin{equation}
h\(\mathcal{X}^1\) \geqslant - \frac{p}{2} \log_2 (2\pi e).
\end{equation}
\end{lemma}

\begin{lemma}[Asymptotic Enumeration of Labeled $d$-regular Graphs on $k$ Vertices, \cite{mckay1991asymptotic}]
\label{lem:gr_count}
Let $d=o(\sqrt{p})$, then the number of labeled regular graphs of degree $d$ on $k$ nodes is
\begin{equation}
\label{eq:reg_gr_num}
\mathcal{Q}_d(k) = \frac{(kd)!}{\(\frac{kd}{2}\)!2^{kd/2}(d!)^k}\exp\(-\frac{d^2-1}{4}-\frac{d^3}{12k}+O\(\frac{d^2}{k}\)\),\quad k \to \infty.
\end{equation}
\end{lemma}

To obtain a lower bound on the number of possible graphs in the family at hand defined by Assumptions [A1]-[A3], let us strengthen Assumption [A3] and assume that all the vertices inside every cell of the lattice are only connected to the vertices inside the same cell (this can only reduce the cardinality). Denote the number of graphs obtained this way by $\mathcal{T}_{k,d}^p$, where $k$ is the number of vertices inside each cell and can be expressed as
\begin{equation}
\label{eq:k_def}
k = \eta\(\rho \(\frac{p}{\eta}\)^{\xi}\)^2 = \rho^2\eta^{1-2\xi} p^{\xi}.
\end{equation}
\begin{lemma}
\label{lem:graph_lemma}
For large enough $p$, the cardinality of the set $|\mathcal{T}_{k,d}^p|$ asymptotically satisfies
\begin{equation}
\log_2|\mathcal{T}_{k,d}^p| \geqslant \frac{dp}{2}\log\(\frac{p^{\xi}}{d}\),\quad p \to \infty.
\end{equation}
\end{lemma}
\begin{proof}
Let us cover our graph by (roughly) $\floor*{\frac{p}{k}}$ squares $\mathcal{H}_i$ each containing approximately $k$ vertices. Since $k \gg d$, we get a good lower bound on the total number of possible graphs by assuming that inside each square $\mathcal{H}_i$, the induced subgraph can be chosen arbitrarily from the family of $d$-regular graphs on $k$ vertices. The number of such subgraphs is given by formula (\ref{eq:reg_gr_num}), and overall we get
\begin{equation}
\label{eq:comb_b_t1}
|\mathcal{T}_{k,d}^p| \geqslant \(\mathcal{Q}_d(k)\)^{p/k}.
\end{equation}
Using Stirling's approximation
\begin{equation}
m! \sim \sqrt{2\pi m}\(\frac{m}{e}\)^m,\quad m \to \infty,
\end{equation}
we obtain for $k \gg d$
\begin{align}
\(\mathcal{Q}_d(k)\)^{p/k} &\geqslant \frac{(kd)!}{\(\frac{kd}{2}\)!2^{kd/2}(d!)^{k}}\exp\(-\frac{d^2-1}{4}-\frac{d^3}{12k}+O\(\frac{d^2}{k}\)\) \nonumber \\
& \geqslant \frac{\sqrt{2\pi kd}(kd)^{kd}e^{kd/2}e^{dk}}{e^{kd}\sqrt{\pi kd}\(\frac{kd}{2}\)^{\frac{kd}{2}}2^{kd/2}(\sqrt{2\pi d})^{k}d^{dk}}\exp\(-\frac{d^2-1}{4}-\frac{d^3}{12k}+O\(\frac{d^2}{k}\)\) \nonumber\\
& \geqslant \(\frac{ke}{d}\)^{kd/2}\frac{1}{(2\pi d)^{k/2}}\exp\(-\frac{d^2-1}{4}-\frac{d^3}{12k}+O\(\frac{d^2}{k}\)\) \nonumber \\
&\geqslant \(\frac{ke}{d}\)^{k(d-1)/2}.
\end{align}
Plug the last inequality into (\ref{eq:comb_b_t1}) to get
\begin{equation}
\label{eq:comb_b_t}
|\mathcal{T}_{k,d}^p| \geqslant \(\mathcal{Q}_d(k)\)^{p/k} \geqslant \(\frac{ke}{d}\)^{p(d-1)/2},
\end{equation}
which completes the proof.
\end{proof}

\begin{proof}[Proof of Theorem \ref{th:nec_cond_complete_graph}]
Consider the following sequence of inequalities,
\begin{equation}
\label{eq:inf_th_main_b_unc}
\frac{p}{2} \log_2\(\frac{2 \pi e}{1-d\bar{\theta}}\) \stackrel{(i)}{\geqslant} h(\mathcal{X}^1) = I(\mathcal{X}^1; G_p) + h(\mathcal{X}^1|G_p) \stackrel{(ii)}{\geqslant} I(\mathcal{X}^1; G_p) - \frac{p}{2} \log_2\(2 \pi e\), 
\end{equation}
where (i) follows from Lemma \ref{lem:upp_b_diff_e}, and (ii) from Lemma \ref{lem:low_b_diff_e}. Derive from (\ref{eq:inf_th_main_b_unc}) the upper bound
\begin{equation}
\label{eq:inf_ineq}
I(\mathcal{X}^1; G_p) \leqslant \frac{p}{2} \log_2\(\frac{2 \pi e}{1-d\bar{\theta}}\) + \frac{p}{2} \log_2\(2 \pi e\) = \frac{p}{2}\log_2\(\frac{(2\pi e)^2}{1-d\bar{\theta}}\).
\end{equation}
The version of Fano's inequality given in Lemma \ref{lem:fano_ineq_lemma} item \ref{enum:fano_1}) together with (\ref{eq:inf_ineq}) imply that if
\begin{equation}
\label{eq:h_low_b_unc}
n < (1-\delta)\frac{\log_2 (|\mathcal{T}_{k,d}^p|)}{\frac{p}{2}\log_2\(\frac{(2\pi e)^2}{1-d\bar{\theta}}\)} \stackrel{(i)}{\leqslant} (1-\delta)\frac{\frac{dp}{2}\log\(\frac{p^{\xi}}{d}\)}{\frac{p}{2}\log\(\frac{(2\pi e)^2}{1-d\bar{\theta}}\)} \leqslant (1-\delta)\frac{d\,\log \(\frac{p^{\xi}}{d}\)}{\log\(\frac{(2\pi e)^2}{1-d\bar{\theta}}\)},
\end{equation}
any decoder will be $\delta$-unreliable, where in (i) we applied Lemma \ref{lem:graph_lemma}. Since we are interested in vanishing errors, the claim follows.
\end{proof}

\section{Large Deviation Principle}
In this section, we introduce the concept of Large Deviation Principle (LDP) which we will use below to estimate the cardinality of the model class. The latter will be plugged into Fano's equality to obtain the necessary sample complexity bounds.
\begin{definition}
A sequence $\{\mathbb{P}_m\}_{m=1}^\infty$ of probability measures on $\mathcal{P}$ satisfies a Large Deviation Principle with speed $a_n$ and rate function $I$ if
\begin{equation*}
-\inf_{b \in B^0}I(b) \leqslant \liminf_{m \to \infty} \frac{\log\, \mathbb{P}_m(B)}{a_m} \leqslant \limsup_{m \to \infty} \frac{\log\,\mathbb{P}_m(B)}{a_m} \leqslant -\inf_{b \in \bar{B}}I(b),\quad \forall B \subset \mathcal{P},
\end{equation*}
where $I : \mathcal{P} \to \overline{\mathbb{R}}_+$ is lower semi-continuous (its level sets $L(M) = \{b\in \mathcal{P} | I(b)\leqslant M\}$ are closed for any $M\geqslant 0$). If $L(M)$ are compact, we refer to $I(\cdot)$ as a good rate function.
\end{definition}

Given an element $\Gamma \in \mathcal{P}$, let $U_\varepsilon(\Gamma)$ be its $\varepsilon$-vicinity. In addition to the LDP we also formulate the so-called local LDP.
\begin{definition}
Assume that for all $\Gamma \in \mathcal{P}$,
\begin{equation}
\liminf_{\varepsilon \to 0}\liminf_{m \to \infty} \frac{\log\, \mathbb{P}_m(U_\varepsilon(\Gamma))}{a_m} =\limsup_{\varepsilon \to 0}\limsup_{m \to \infty} \frac{\log\, \mathbb{P}_m(U_\varepsilon(\Gamma))}{a_m} = -I(\Gamma),
\end{equation}
then we say that $\mathbb{P}_m$ satisfies the local LDP.
\end{definition}

The last definition can be roughly interpreted as
\begin{equation*}
\mathbb{P}_m(U_\varepsilon(\Gamma)) \sim e^{-a_mI(\Gamma)}.
\end{equation*}
We refer the reader to \cite{dembo2010large} and references therein for more details on the LDP.

\section{Enumeration of Convex Polyominoes and Polygons}
\label{app:enum_poly}
As mentioned in Section \ref{sec:model_classes}, another possible way to discretize the plane consists in using convex lattice polygons with similar restrictions, such as perimeter, area, both perimeter and area, etc. Due to the similarity of treatment of both these families of polygons, in this section we use the LDP to enumerate both the family of CPMs and the family of Convex PolyGons (CPGs) on the square lattice. However, the detection algorithm for the latter requires more technical details and due to the lack of space we postpone it to our next publication \cite{soloveychik2018polygonal}. Let us start with a number of auxiliary results.

\subsection{Convex Polyominoes}
According to the definition from Section \ref{sec:model_classes}, a polyomino on a square lattice is convex if its is both colum- and row-convex. 
\begin{lemma}[\cite{guttmann2009polygons}]
\label{lem:conv_circ}
A square lattice polyomino is convex if and only if its perimeter is equal to the perimeter of its circumscribed rectangle.
\end{lemma}
Figure \ref{fig:conv_polyomino} shows an example of a convex polyomino on a square lattice and its circumscribed rectangle. In the discrete scenario we have the following analog of the isoperimetric inequality.

\begin{lemma}[Isoperimeteric inequality for convex polyominoes]
\label{lem:isoper_ineq}
For a polyomino of area $A$ and perimeter $P$ on the square lattice,
\begin{equation}
\label{eq:isoper_ineq}
A \leqslant \frac{P^2}{16},
\end{equation}
the equality is reached when the polyomino is a square.
\end{lemma}
\begin{proof}
We should only prove (\ref{eq:isoper_ineq}) for convex polyominoes. Due to Lemma \ref{lem:conv_circ}, the perimeter of the circumscribed rectangle of a convex polyomino of perimeter $P$ is also $P$. Apparently, the area of such a polyomino is maximized when it coincides with its circumscribed rectangle. Among the rectangles of perimeter $P$, the area is maximal for the square, which completes the proof.
\end{proof}

\begin{figure}
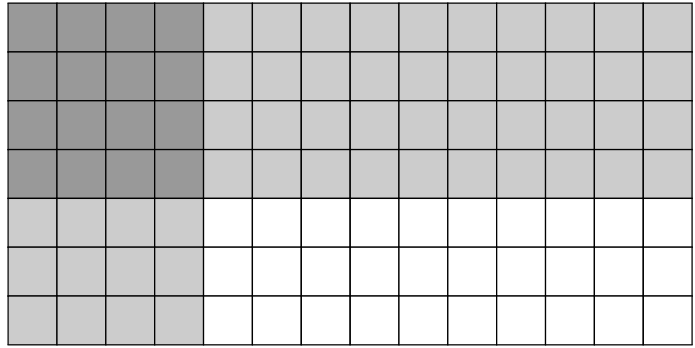

\centering
 $\ytableausetup{textmode}
\begin{ytableau}
*(black!40) & *(black!40) & *(black!40) & *(black!40) & *(black!20) & *(black!20) & *(black!20) & *(black!20) & *(black!20) & *(black!20) & *(black!20) & *(black!20) & *(black!20) & *(black!20) \\
*(black!40) & *(black!40) & *(black!40) & *(black!40) & *(black!20) & *(black!20) & *(black!20) & *(black!20) & *(black!20) & *(black!20) & *(black!20) & *(black!20) & *(black!20) & *(black!20) \\
*(black!40) & *(black!40) & *(black!40) & *(black!40) & *(black!20) & *(black!20) & *(black!20) & *(black!20) & *(black!20) & *(black!20) & *(black!20) & *(black!20) & *(black!20) & *(black!20) \\
*(black!40) & *(black!40) & *(black!40) & *(black!40) & *(black!20) & *(black!20) & *(black!20) & *(black!20) & *(black!20) & *(black!20) & *(black!20) & *(black!20) & *(black!20) & *(black!20) \\
*(black!20) & *(black!20) & *(black!20) & *(black!20) & *(black!0) & *(black!0) & *(black!0) & *(black!0) & *(black!0) & *(black!0) & *(black!0) & *(black!0) & *(black!0) & *(black!0) \\
*(black!20) & *(black!20) & *(black!20) & *(black!20) & *(black!0) & *(black!0) & *(black!0) & *(black!0) & *(black!0) & *(black!0) & *(black!0) & *(black!0) & *(black!0) & *(black!0) \\
*(black!20) & *(black!20) & *(black!20) & *(black!20) & *(black!0) & *(black!0) & *(black!0) & *(black!0) & *(black!0) & *(black!0) & *(black!0) & *(black!0) & *(black!0) & *(black!0)
\end{ytableau}$
\caption{\small Minimal inscribed square.}
\label{fig:min_inscr_sq}
\end{figure}

\begin{lemma}
\label{lem:min_inscr_sq}
Let a convex polyomino have area $A$ and perimeter $P$, then it contains a square with the side length at least $\frac{2A}{P}$.
\end{lemma}
\begin{proof}
Figure \ref{fig:min_inscr_sq} illustrates the extreme case of a polyomino described in the statement with the minimal possible inscribed square. Note that the specific configuration of the circumscribed rectangle does not affect the reasoning.
\end{proof}

\subsection{Enumeration of Convex Polyominoes}
Our main goal in this section is to count the number of convex polyominoes (CPMs) without restrictions and with specific restrictions, such as fixed perimeter, or area, or both perimeter and area. There exists a large body of literature addressing the problem of polyomino counting according to their perimeter and/or area \cite{guttmann2009polygons, delest1988generating, bousquet1992convex, bousquet1996method, bousquet1995generating}. However, in all these works the desired numbers are given implicitly as coefficients of the corresponding terms in the series expansions of the generating functions derived therein. These series are usually too complicated and bulky to be analyzed directly and the sought for coefficients cannot be easily extracted. Moreover, even the asymptotic behavior of these coefficients is by no means obvious to derive. Below we use the LDP results derived in \cite{soloveychik2018large} to count the cardinalities of different CPM families.

Consider the plane $\mathbb{R}^2$ with the standard basis and fixed origin. Assume we are given a closed piece-wise differentiable (more generally, it may be continuous) curve $\Gamma \subset \mathbb{R}^2$ which is unimodal in both vertical and horizontal directions. In other words, every horizontal and vertical line intersects the curve in at most two points. Denote the region embraced by $\Gamma$ by $\mathcal{F}$ and its area by
\begin{equation}
\text{area}(\mathcal{F}) = A.
\end{equation}
For convenience, let us assume that the barycenter of $\mathcal{F}$ coincides with the coordinate origin. Given two curves $\Gamma_1$ and $\Gamma_2$, the distance between them is defined as
\begin{equation}
d(\Gamma_1,\Gamma_2) = \text{area}(\mathcal{F}_1\Delta \mathcal{F}_2).
\end{equation}

For every $m \in \mathbb{N}$, we consider the integer lattice centered at the origin scaled by $\frac{1}{\sqrt{m}}$ so that the area of every elementary cell becomes $\frac{1}{m}$. Consider the set of polyominoes in the $\varepsilon$-vicinity of $\Gamma$, which we denote by
\begin{equation}
\mathbb{Q}_m = \mathbb{M}_m \cap U_{\varepsilon}(\Gamma),
\end{equation}
where $\mathbb{M}_m$ is the set of all convex polyominoes on the $\frac{1}{\sqrt{m}}$-grid.

Next we count the polyominoes in $\mathbb{Q}_m$ satisfying different conditions mentioned earlier. For example, the polyominoes in $\mathbb{Q}_m$ having fixed area $\mathbb{Q}_{A}$\footnote{We suppress the $m$ index to simply the notation.}, fixed perimeter $\mathbb{Q}_{P}$, or both fixed area and perimeter $\mathbb{Q}_{A,P}$, etc. Denote
\begin{equation}
Q_X = |\mathbb{Q}_{X}|,\quad X \in \{A,L,\{A,L\}\}.
\end{equation}

\begin{rem}
By convention, below we write
\begin{equation}
\int_{\Gamma}f(\Gamma)(|dx|+|dy|) = \int_{\Gamma}f(\Gamma(s))(|\sin(\theta)|+|\cos(\theta)|)ds,
\end{equation}
where $\Gamma(s)$ is the natural parameterization of the curve by its arc length and $\theta = \arctan(y')$ is the angle between the tangent line at a point and the horizontal axis.
\end{rem}

\begin{lemma}[\cite{soloveychik2018large}]
\label{lem:conv_polyOMINOES}
Let $\varepsilon = \varepsilon(m)$ be such that $\varepsilon = o(m)$ and $\varepsilon \sqrt{m} \to \infty$, then the number $Q_X(\Gamma;\varepsilon),\; X \in \{A,L,\{A,L\}\}$ of convex polyominoes on the square lattice of width $\frac{1}{\sqrt{m}}$ in the $\varepsilon$-vicinity of the unimodal curve $\Gamma$ satisfies the following bound,
\begin{equation}
\left|\frac{\log\, Q_X(\Gamma;\varepsilon)}{\sqrt{m}} - I_M(\Gamma)\right| = O(\varepsilon),
\end{equation}
where
\begin{equation}
I_M(\Gamma) = \int_{\Gamma}H\(\frac{|y'|}{1+|y'|}\)(|dx|+|dy|),
\end{equation}
and $y(x)$ is the local parametrization of $\Gamma$.
\end{lemma}
\begin{proof}
The proof follows from Lemmas 5 and 6 from \cite{soloveychik2018large}.
\end{proof}

Interestingly, we observe that the bound in Lemma \ref{lem:conv_polyOMINOES} does not depend on the type of possible constraints on the family of models, such as perimeter, area, or both. For more details we refer the reader to our article \cite{soloveychik2018large}.

\subsection{Enumeration of Convex Polygons}
\label{app:enum_cpgs}
Unlike the previous section, here we consider the family of truly convex polygons on square lattice. Let us mention, that in the case of convex lattice polygons, the isoperimetric inequality receives its standard form.
\begin{lemma}[Isoperimeteric inequality for convex polyominoes]
\label{lem:isoper_ineq_polygon}
For a polygon of area $A$ and perimeter $P$ on the square lattice,
\begin{equation}
\label{eq:isoper_ineq_polygon}
A < \frac{P^2}{4\pi}.
\end{equation}
\end{lemma}

The large deviation principle for this class of polygons was developed by A.M. Vershik and O. Zeitouni in \cite{vershik1994limit, vershik1999large}. The first of these two fundamental works exploits a combinatorial approach, while the second one uses the Gauss-Minkovskii transformation to allow a more analytic treatment. Let $\Gamma \in \mathbb{R}^2$ be a closed convex curve and $\mathcal{F}$ its embraced region. We use the same metric over the space of curves and for every $m \in \mathbb{N}$ construct the $\frac{1}{\sqrt{m}}$-scaled interger lattice. Consider the set of convex polygons in the $\varepsilon$-vicinity of $\Gamma$, which we denote by
\begin{equation}
\mathbb{F}_m = \mathbb{D}_m \cap U_{\varepsilon}(\Gamma),
\end{equation}
where $\mathbb{D}_m$ is the set of all convex polygons on the $\frac{1}{\sqrt{m}}$-grid. Similarly to the above setting we count convex polygons with different fixed parameters, such as area, perimeter or both and introduce an analogous notation for them. Below $\zeta(l)$ stands for the Riemannian zeta function,
\begin{equation}
\zeta(l) = \sum_{k=1}^\infty \frac{1}{k^l}.
\end{equation}

Denote
\begin{equation}
F_X = |\mathbb{F}_{X}|,\quad X \in \{A,L,\{A,L\}\}.
\end{equation}
\begin{lemma}[Corollary from Theorem 2.3 from \cite{vershik1994limit} and Theorems 1, 2 from \cite{vershik1999large}]
\label{lem:conv_polyGONS}
Let $\varepsilon = \varepsilon(m)$ be such that $\varepsilon = O(m)$ and $\varepsilon \sqrt{m} \to \infty$, then the number $F_X(\Gamma;\varepsilon)$ of convex polygons on the square lattice of width $\frac{1}{\sqrt{m}}$ in the $\varepsilon$-vicinity of the convex curve $\Gamma$ satisfies the bound
\begin{equation}
\left|\frac{\log\, F_X(\Gamma;\varepsilon)}{m^{1/3}} - I_{G}(\Gamma)\right| = O(\varepsilon), \quad X \in \{A,L,\{A,L\}\},
\end{equation}
where
\begin{equation}
I_{G}(\Gamma) = \frac{3}{2^{2/3}}\(\frac{\zeta(3)}{\zeta(2)}\)^{1/3}\int_{\Gamma}k(s)^{1/3}ds,
\end{equation}
and $\Gamma$ is parametrized by its Euclidean arc length $s$ and $k(s)$ is its affine curvature defined as
\begin{equation}
k = -\frac{1}{2}\(\frac{1}{\(y''\)^{2/3}}\)'',
\end{equation}
where $y(x)$ is the local parametrization of $\Gamma$ and the differentiation is w.r.t.\ $x$.
\end{lemma}

\begin{rem}
Interestingly, we can also consider polyominoes inscribed into convex polygons. The enumeration of the former follows directly from that for the latter. The only minor difference will involve the calculation of their perimeter, but it is straightforward to achieve. Therefore, the result stated in Lemma \ref{lem:conv_polyGONS} applies to the polyominoes inscribed into the convex polygons.
\end{rem}

\begin{proof}
The proof follows that of the aforementioned results in \cite{vershik1994limit, vershik1999large} through calculating the remainder terms depending on $\varepsilon$. It is important to note that the scaling used here is different from the one in \cite{vershik1994limit, vershik1999large}. Indeed, in these works the original integer grid is multiplied by $\frac{1}{m}$ in both directions, while we use the $\frac{1}{\sqrt{m}}$ scaling, which results in the speed factor $m^{1/3}$ in our case instead of $m^{2/3}$ there.
\end{proof}

Table \ref{tab:table_polyg} summarizes the known LDP results for various families of CPMs and CPGs. Remarkably, the speed of the LDP for CPMs is always higher than that for the CPGs, reflecting the fact the number of CPMs is much higher than the number of CPGs. This happens because the convexity of the latter is a much stronger restriction.

\begin{landscape}
{\renewcommand{\arraystretch}{2}
\renewcommand{\figurename}{Table.}
\setcounter{figure}{0}   
\begin{figure}[h]
\scalebox{.88}{
\begin{tabular}{ c || c | c | c }	
  \hline		
  Type of polygons & Speed & Rate Function $I(\Gamma)$ & Limiting Shape \\
  \hline\hline
  Young (Ferrers) diagrams (Ferrers) & $m^{1/2}$ & $\pi\sqrt{\frac{2}{3}} - \int_{\Gamma}(1-y')H\(\frac{-y'}{1-y'}\)dx$ \cite{blinovskii1999large, dembo1998large} & $e^{\frac{-\pi x}{\sqrt{6}}}+e^{\frac{-\pi y}{\sqrt{6}}}=1\quad$ (1) \cite{vershik1987statistical} \\
  \hline
  Young diagrams of height $a$ & $m^{1/2}$ & $C(a) - \int_{\Gamma}(1-y')H\(\frac{-y'}{1-y'}\)dx$ \cite{soloveychik2018large} & scaled segments of (1) \cite{petrov2009two} \\
  \hline
  Young diagrams in $[a,b]$ box & $m^{1/2}$ & $C(a,b) - \int_{\Gamma}(1-y')H\(\frac{-y'}{1-y'}\)dx$ \cite{soloveychik2018large} & scaled segments of (1) \cite{petrov2009two} \\
  \hline
  Strict Young diagrams & $m^{1/2}$ & $\pi\sqrt{\frac{1}{3}} - \int_{\Gamma} H\(-y'\)dx$ \cite{dembo1998large} & $e^{\frac{-\pi x}{\sqrt{12}}}+e^{\frac{-\pi y}{\sqrt{12}}}=1$ \cite{dembo1998large} \\
  \hline
  CPMs of area $A$ & $m^{1/2}$ & $C(A) - \int_{\Gamma}H\(\frac{|y'|}{1+|y'|}\)(|dx|+|dy|)$ \cite{soloveychik2018large} & concat. of scaled segments of (1) \cite{soloveychik2018large} \\
  \hline
  CPMs of area $A$ and perimeter $P$ & $m^{1/2}$ & $C(A,P) - \int_{\Gamma}H\(\frac{|y'|}{1+|y'|}\)(|dx|+|dy|)$ \cite{soloveychik2018large} & concat. of scaled segments of (1) \cite{soloveychik2018large} \\
  \hline
  CPMs & $m^{1/2}$ & $C - \int_{\Gamma}H\(\frac{|y'|}{1+|y'|}\)(|dx|+|dy|)$ \cite{soloveychik2018large} & concat. of scaled segments of (1) \cite{soloveychik2018large} \\
  \hline
  CPGs in a square of perimeter $P=4L$ & $m^{1/3}$ & $\frac{3}{2^{2/3}}\(\frac{\zeta(3)}{\zeta(2)}\)^{1/3}\[(4P)^{2/3}-\int_{\Gamma}k(s)^{1/3}ds\]$ \cite{vershik1994limit} & $\sqrt{L-|x|} + \sqrt{L-|y|}=\sqrt{L}$ \cite{vershik1994limit} \\
  \hline
  CPGs of area $A$ and perimeter $P$ & $m^{1/3}$ & $\frac{3}{2^{2/3}}\(\frac{\zeta(3)}{\zeta(2)}\)^{1/3}\[\(8\pi^2\min\[A,\frac{P^2}{4\pi}\]\)^{1/3}-\int_{\Gamma}k(s)^{1/3}ds\]$ \cite{vershik1999large} & ellipses of area $A$ and perimeter $P$ \cite{vershik1999large} \\
  \hline
  CPGs of perimeter $P$ & $m^{1/3}$ & $\frac{3}{2^{2/3}}\(\frac{\zeta(3)}{\zeta(2)}\)^{1/3}\[\(2\pi P^2\)^{1/3}-\int_{\Gamma}k(s)^{1/3}ds\]$ \cite{vershik1999large}& circle of radius $\frac{P}{2\pi}$ \cite{vershik1999large} \\
  \hline
  \hline
\end{tabular}}
\caption{\small Parameters, speeds, rate functions, and limiting curves of known Young diagrams, CPM and CPG ensembles on the $\frac{1}{\sqrt{m}}\times\frac{1}{\sqrt{m}}$ grid. \textbf{Legend:} $y(x)$ is the local parametrization of $\Gamma$ and $k(s)$ is the affine curvature of $\Gamma,\; k = -\frac{1}{2}\(\frac{1}{\(y''\)^{2/3}}\)''$.}
\label{tab:table_polyg}
\end{figure}}

\setcounter{figure}{8}   

To make the enumeration of convex lattice polygons compatible with the enumeration of polyominoes, the scaling of the lattice must be chosen as $\frac{1}{\sqrt{m}}\mathbb{Z} \times \frac{1}{\sqrt{m}}\mathbb{Z}$, or in other words the size of the elementary cell should be $\frac{1}{\sqrt{m}} \times \frac{1}{\sqrt{m}}$ as in random partitions unlike the usually utilized in the CPG enumeration literature $\frac{1}{m}$-scaling.
\end{landscape}

\section{Information-Theoretic Lower Bound for Region Detection}
\label{app:area_detect_proof}
\begin{lemma}
\label{lem:gauss_area_diff}
Given a graph $G_p$, let $\theta_1$ and $\theta_2$ be two coupling parameter vectors corresponding to two Gaussian graphical models over the same graph. Denote their respective precision matrices by $\J_1$ and $\J_2$, then
\begin{equation}
\label{eq:trace_lemma}
S(\theta_1\;\|\;\theta_2) = \frac{1}{2}\Tr{\(\J_1-\J_2\)\(\J_2^{-1}-\J_1^{-1}\)},
\end{equation}
where $S(\theta_1\;\|\;\theta_2)$ is defined in (\ref{eq:sym_kl_div}).
\end{lemma}

\begin{proof}
Denote the corresponding covariance matrices by $\bm\Sigma_1=\J_1^{-1}$ and $\bm\Sigma_2=\J_2^{-1}$. The KL-divergence between two normal distributions $\mathbb{P}_1 = \mathcal{N}(\bm{0},\bm\Sigma_2)$ and $\mathbb{P}_2 = \mathcal{N}(\bm{0},\bm\Sigma_1)$ reads as
\begin{align}
D(\theta_1\;\|\;\theta_2) & = \int\log\frac{\mathbb{P}_1}{\mathbb{P}_2} d\mathbb{P}_1 + \int\log\frac{\mathbb{P}_2}{\mathbb{P}_1} d\mathbb{P}_2 \\
&= \int \log\[\frac{\sqrt{\det{\bm\Sigma_2}}\exp\(-\frac{1}{2}\x^\top\bm\Sigma_1^{-1}\x\)}{\sqrt{\det{\bm\Sigma_1}}\exp\(-\frac{1}{2}\x^\top\bm\Sigma_2^{-1}\x\)}\] d\mathbb{P}_1 + \int \log\[\frac{\sqrt{\det{\bm\Sigma_1}}\exp\(-\frac{1}{2}\x^\top\bm\Sigma_2^{-1}\x\)}{\sqrt{\det{\bm\Sigma_2}}\exp\(-\frac{1}{2}\x^\top\bm\Sigma_1^{-1}\x\)}\] d\mathbb{P}_2 \nonumber \\
& = \frac{1}{2}\[\mathbb{E}_1\[\Tr{\x\x^\top\bm\Sigma_2^{-1}}-\Tr{\x\x^\top\bm\Sigma_1^{-1}}\] + \mathbb{E}_2\[\Tr{\x\x^\top\bm\Sigma_1^{-1}}-\Tr{\x\x^\top\bm\Sigma_2^{-1}}\]\] \nonumber \\
& = \frac{1}{2}\[\Tr{\bm\Sigma_1\(\bm\Sigma_2^{-1}-\bm\Sigma_1^{-1}\)} + \Tr{\bm\Sigma_2\(\bm\Sigma_1^{-1}-\bm\Sigma_2^{-1}\)}\] \nonumber \\
& = \frac{1}{2}\Tr{\(\bm\Sigma_2^{-1}-\bm\Sigma_1^{-1}\)\(\bm\Sigma_1-\bm\Sigma_2\)} = \frac{1}{2}\Tr{\(\J_1-\J_2\)\(\J_2^{-1}-\J_1^{-1}\)}.\nonumber
\end{align}
\end{proof}

\begin{corollary}
\label{cor:gauss_area_diff}
Assume we are given two regions $\mathcal{F}_s$ and $\mathcal{F}_t$ with the coupling parameters $\theta_s$ and $\theta_t$ respectively, sharing a boundary. Let us deform the regions only along the shared boundary so that their new shapes are $\mathcal{F}_s'$ and $\mathcal{F}_t'$, accordingly, then
\begin{equation}
S(\theta_s\;\|\;\theta_t) \leqslant \frac{1}{2}\(\frac{\theta_s-\theta_t}{1-d\bar{\theta}}\)^2\frac{\eta d\norm{\mathcal{F}_s\Delta \mathcal{F}_s'}}{2}.
\end{equation}
\end{corollary}
\begin{proof}
Let $\B$ and $\C$ be two invertible matrices of the same size. Consider the identity
\begin{equation}
\B^{-1}-\C^{-1} = \B^{-1}\(\C-\B\)\C^{-1},
\end{equation}
which can be easily checked by multiplying by $\B$ on the left and by $\C$ on the right. Using this identity, we rewrite the right-hand side of (\ref{eq:trace_lemma}) as
\begin{equation}
\Tr{\(\J_s-\J_t\)\(\J_t^{-1}-\J_s^{-1}\)} = \Tr{\(\J_s-\J_t\)\J_t^{-1}\(\J_s-\J_t\)\J_s^{-1}}.
\end{equation}
Denote by $\A_G$ the adjacency matrix of the graph at hand, then we can bound the norm of $\J_s^{-1}$ as
\begin{equation}
\label{eq:norm_bound}
\norm{\J_s^{-1}} = \norm{\(\I+\theta_s\A_G\)^{-1}} \leqslant \frac{1}{1-\theta_s\norm{\A_G}} \leqslant \frac{1}{1-d\bar{\theta}}.
\end{equation}
Similar bound holds for $\norm{\J_t^{-1}}$, as well. Finally, we conclude,
\begin{equation}
\Tr{\(\J_s-\J_t\)\(\J_t^{-1}-\J_s^{-1}\)} \leqslant \frac{1}{\(1-d\bar{\theta}\)^2} \Tr{\(\J_s-\J_t\)^2}.
\end{equation}
Taking into account that the number of edges in the subgraph $\mathcal{F}_s\Delta \mathcal{F}_s'$ is
\begin{equation}
E\(\mathcal{F}_s\Delta \mathcal{F}_s'\) = \frac{\eta d\norm{\mathcal{F}_s\Delta \mathcal{F}_s'}}{2},
\end{equation}
we get the desired statement.
\end{proof}

Unlike the proof of Theorem \ref{th:nec_cond_complete_graph}, below we use the second inequality from Lemma \ref{lem:fano_ineq_lemma}. Recall that the areas of the regions and their boundary lengths are $A_s=\frac{p_s}{\eta}$ and $l_s = \beta_s\sqrt{A_s}$, respectively and due to Assumption [A1], the lengths of the sides of the boundaries must be multiples of $r = \rho \sqrt{\frac{p_s}{\eta}}$, respectively for each $s=1,\dots,S$.

\newtheorem*{customtheorem}{Theorem 2a}

\begin{customtheorem}[Generalization of Theorem \ref{th:nec_cond_region}]
\label{th:nec_cond_region_a}
Suppose that Assumptions [A1] - [A4] hold and that a graph $G_p \in R_{p}$ is chosen uniformly from the class $R_p$ which is in turn chosen uniformly from the family $\mathcal{R}_p$. The number of i.i.d.\ samples from $G_p$ necessary for $\mathbb{P}_e^R$ to vanish asymptotically is
\begin{equation}
\label{eq:inf_th_bound_main}
n \geqslant \frac{1}{p^{1/2+\xi-\phi(1-2\xi)}}\[\frac{\rho^{1-2\phi}}{\eta^{1/2-\xi+\phi(1-2\xi)} d}\min_s \min\limits_{\partial \mathcal{F}_s \cap \partial \mathcal{F}_t \neq \emptyset} \frac{C(\beta_s)}{\(\frac{\theta_s-\theta_t}{1-d\bar{\theta}}\)^2\frac{\eta d\norm{\mathcal{F}_s\Delta \mathcal{F}_s'}}{4}\beta_s \nu_s^{1/2}}\],\quad p \to \infty,
\end{equation}
where $C(\beta_s)$ is a constant depending only on $\beta_s$ and
\begin{equation}
\phi = \begin{cases} \frac{1}{2}, &\text{for CPMs}, \\ \frac{1}{3}, & \text{for CPGs}. \end{cases}
\end{equation}
\end{customtheorem}

Interestingly, unlike the CPM case where the decay is proportional to $\frac{1}{p^{2\xi}}$ and becomes very slow for small $\xi$, in the CPG case it scales as $\frac{1}{p^{1/6+5\xi/3}}$. Therefore, in the CPG case even for vanishing $\xi$, the necessary number of required samples vanishes asymptotically. This can be explained by the higher level of boundary regularity, making the family of models much smaller, or equivalently making the detection with the same amount of i.i.d.\ snapshots easier.

\begin{proof}[Proof of Theorem 2a]
We use Fano's inequality in the form (\ref{eq:fano_2}) to prove the statement, therefore, we need to estimate the numerator and the denominator of the ratio in its right-hand side. The strategy of the proof will be as follows. For every region $\mathcal{F}_s$, we will compute the number of samples necessary to distinguish its shape from the shapes of its neighbors $\mathcal{F}_t$ and then will take the minimum over $s$, which corresponds to the worst case bound. We denote the coupling parameter of the aforementioned regions by $\theta_s$ and $\theta_t$, respectively.

Let us start with the numerator. Using the relation between the area of the domain at hand and the area of a single cell in our setting, we need to determine the scaling $m = m(p)$ that we will use in Lemma \ref{lem:conv_polyOMINOES} to count the number of admissible models. Indeed, in our case,
\begin{equation}
m = \frac{1}{\rho^2}\(\frac{p}{\eta}\)^{1-2\xi}.
\end{equation}
Given a family $\Psi$ of admissible continuous curves (unimodal in the CPM case and convex in the CPG case) satisfying specific conditions, such as restrictions on their perimeter $P$, area $A$, etc., we can easily estimate the cardinality of the set $Q(\Psi)$ of CPMs (CPGs) on the $\frac{1}{\sqrt{m}}$-grid in the $\varepsilon$-vicinity of $\Psi$ using the result stated in Lemma \ref{lem:conv_polyOMINOES} (Lemma \ref{lem:conv_polyGONS}),
\begin{equation}
\label{eq:pse_card_approx}
\log\, Q(\Psi) \approx m^{\phi} \max_{\Gamma \in \Psi}I_{\phi}(\Gamma),
\end{equation}
where
\begin{equation}
\(\phi, I\) = \begin{cases} \(\frac{1}{2},I_M\), &\text{for CPMs}, \\ \(\frac{1}{3},I_G\), & \text{for CPGs}. \end{cases}
\end{equation}

Let us now treat the denominator. Using Corollary \ref{cor:gauss_area_diff}, we can write the quantity
\begin{equation}
\frac{2}{M^2}\sum_{i=1}^M\sum_{j=i+1}^M S(\theta_i\;\|\;\theta_j) \leqslant \(\frac{\theta_s-\theta_t}{1-d\bar{\theta}}\)^2\frac{\eta d\,\mean{\norm{\mathcal{F}_s\Delta \mathcal{F}_s'}}}{4}.
\end{equation}
As mentioned in Lemmas \ref{lem:conv_polyOMINOES} and \ref{lem:conv_polyGONS}, their statements are valid if we make $\varepsilon$ decrease in such way that $\sqrt{m}\varepsilon \to \infty$. Set $\varepsilon = m^{-1/2+a}$, where $a>0$, then we can write that
\begin{equation}
\mean{\mathcal{F}_s\Delta \mathcal{F}_s'} \approx l_s \varepsilon \sqrt{m}\rho\(\frac{p}{\eta}\)^{\xi} = \beta_s \rho \(\frac{\nu_s p}{\eta}\)^{1/2} \(\frac{p}{\eta}\)^{\xi}m^a = \frac{\beta \nu^{1/2}}{\rho^{1+2a}}\(\frac{p}{\eta}\)^{1/2+\xi+a(1-2\xi)}.
\end{equation}
Taking into account Lemma \ref{lem:gauss_area_diff} and plugging the last two formulas into (\ref{eq:fano_2}), we conclude that if
\begin{align}
n &< (1-\delta)\min\limits_{\partial \mathcal{F}_s \cap \partial \mathcal{F}_t \neq \emptyset} \frac{m^{\phi} \max\limits_{\Gamma \in \Psi}I_\phi(\Gamma)}{\(\frac{\theta_s-\theta_t}{1-d\bar{\theta}}\)^2\frac{\eta d\,\mean{\norm{\mathcal{F}_s\Delta \mathcal{F}_s'}}}{4}} \nonumber \\ 
&= 4(1-\delta)\min\limits_{\partial \mathcal{F}_s \cap \partial \mathcal{F}_t \neq \emptyset} \frac{\[\frac{1}{\rho^2}\(\frac{p}{\eta}\)^{1-2\xi}\]^\phi \max\limits_{\Gamma \in \Psi}I_\phi(\Gamma)}{\(\frac{\theta_s-\theta_t}{1-d\bar{\theta}}\)^2\eta d\frac{\beta_s \nu_s^{1/2}}{\rho^{1+2a}}\(\frac{p}{\eta}\)^{1/2+\xi+a(1-2\xi)}} \nonumber \\ 
&= 4(1-\delta)\min\limits_{\partial \mathcal{F}_s \cap \partial \mathcal{F}_t \neq \emptyset} \frac{\(\frac{p}{\eta}\)^{\phi(1-2\xi) - (1/2+\xi+a(1-2\xi))} \max\limits_{\Gamma \in \Psi}I_\phi(\Gamma)}{\(\frac{\theta_s-\theta_t}{1-d\bar{\theta}}\)^2\eta d \beta_s \nu_s^{1/2}\rho^{2\phi -1-2a}}
\end{align}
then any detector of $\mathcal{F}_s$ must be $\delta$-unreliable. Since we need to able to detect the worst (in terms of sample complexity) region $\mathcal{F}_s$ and $a>0$ can be chosen arbitrarily small, the necessary condition on $n$ reads as
\begin{equation}
\label{eq:main_eq_theorem2}
n \geqslant \min_s \min\limits_{\partial \mathcal{F}_s \cap \partial \mathcal{F}_t \neq \emptyset} \frac{\rho^{1-2\phi}\max\limits_{\Gamma \in \Psi}I_\phi(\Gamma)}{\(\frac{\theta_s-\theta_t}{1-d\bar{\theta}}\)^2\eta d\beta_s \nu_s^{1/2}\(\frac{p}{\eta}\)^{1/2+\xi-\phi(1-2\xi)}},
\end{equation}
which concludes the proof. The only remaining point is to mention that in the setting of the theorem, 
\begin{equation}
C(\beta_s) = \max\limits_{\Gamma \in \Psi}I_\phi(\Gamma),
\end{equation}
where $\Psi$ is the family of curves defined by the parameter $\beta_s$. In order to get the statement of Theorem \ref{th:nec_cond_region}, plug $\phi = \frac{1}{2}$.
\end{proof}

\subsection{Computation of $C(\beta)$}
\label{sec:const_calc}
In this section, we explain how to compute the constant $C(\beta)$ on a specific example. Consider the family $\Psi$ of polyminoes with fixed relation between the boundary length and the square root of the area, or in other words with fixed $\beta$. According to (\ref{eq:main_eq_theorem2}), we need to find such a closed curve $\Gamma$ that $\frac{P(\Gamma)}{\sqrt{A(\Gamma)}} = \beta$ and $I(\Gamma)$ is maximized (here $P(\Gamma)$ is clearly the perimeter of the circumscribed rectangle since we are interested in the perimeters of polyominoes). There may be a number of such curves (see \cite{vershik1999large} for details), but among them at least one will be symmetric w.r.t. to the $x$ and $y$ axes. Therefore, it is enough for us to consider its north-eastern quarter $\Gamma_{NE}$. Due to the additivity of integral in $I(\Gamma)$, the shape $\Gamma_{NE}$ has to maximize the latter under the same restriction
\begin{equation}
\label{eq:quater_cond}
\frac{P(\Gamma_{NE})}{\sqrt{A(\Gamma_{NE})}} = \beta,
\end{equation}
where $A(\Gamma_{NE})$ is the area under the graph. As suggested in \cite{petrov2009two}, the solution is obtained in the following manner. Consider Vershik's curve $\Gamma_V$ (defined as (1) in Table \ref{tab:table_polyg}) in Figure \ref{fig:versh_max} on the left and let us find points $a$ and $b$ on the $x$ axis satisfying the following two conditions:
\begin{enumerate}
\item the circumscribed rectangle of the curve segment $\Gamma_V([a,b])$ supported on $[a,b]$ is a square,
\item the relation of perimeter of the obtained square to the area of one of the curved triangles inside the square is $\beta$ (in the figure it is the upper triangle without loss of generality).
\end{enumerate}
As explained in \cite{petrov2009two} such $a$ and $b$ always exist. After we have found the points $a$ and $b$, we get the desired shape $\Gamma_V([a,b])$. Now the maximizing closed curve $\Gamma$ is obtained by gluing together four rotated copies of $\Gamma_V([a,b])$ as in Figure \ref{fig:versh_max} on the right. The value of $I(\cdot)$ on this curve is the desired values of the constant $C(\beta)$.
\begin{figure}[t!]
\centering
\includegraphics[width=15cm]{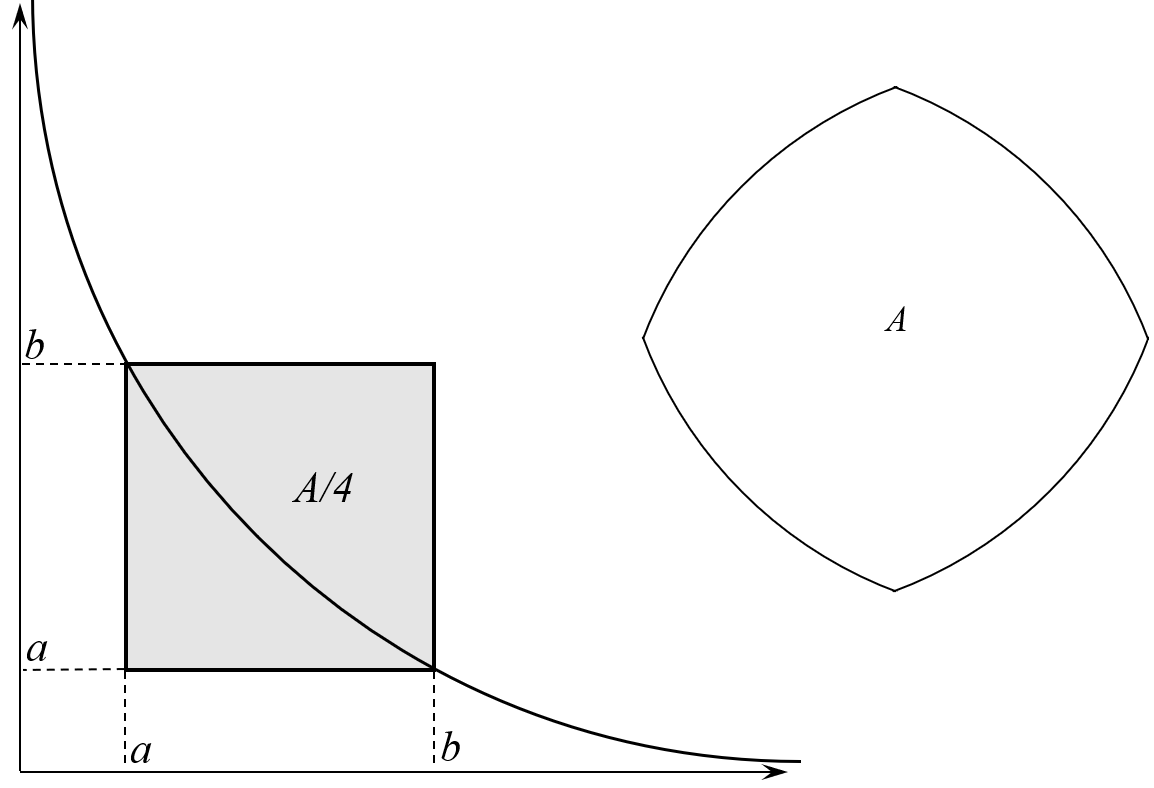}
\caption{\small Vershik's curve segment maximizing the rate function $I(\Gamma)$ under (\ref{eq:quater_cond}) and the closed curve build from its four rotated copies.}
\label{fig:versh_max}
\end{figure}

\section{Parameter Estimation and Consistency}
\label{app:algo}
\begin{proof}[Proof of Lemma \ref{lem:trace_est_lemma}]
Let $\M$ be a symmetric matrix, then the Maclaurin expansion of the function $\(\I+\M\)^{-1}$, for $\norm{\M}<1$ reads as
\begin{equation}
\label{eq:mac_exp}
\(\I+\M\)^{-1} = \I-\M+\M^2-\M^3+o(\M),\quad \norm{\M} \to 0.
\end{equation}
Partition the precision matrix as
\begin{equation}
\J = \begin{pmatrix} \J_{A} & \J_{AB} \\ \J_{AB}^\top & \J_{B} \end{pmatrix} = \begin{pmatrix} \I_k + \theta_s\E_k & \J_{AB} \\ \J_{AB}^\top & \I_{p-k}+\D_{p-k} \end{pmatrix}.
\end{equation}
Use the expansion in (\ref{eq:mac_exp}) to get the following chain of equalities
\begin{align}
\label{eq:trace_j_eq_chain}
&\Tr{\(\J_{A} - \J_{AB} \J_{B}^{-1} \J_{AB}^\top\)^{-1}} = \Tr{\(\I_k + \theta_s\E_k - \J_{AB}\(\I_{p-k}+\D_{p-k}\)^{-1}\J_{AB}^\top\)^{-1}} \\
&\quad= \Tr{\(\I_k + \theta_s\E_k - \J_{AB}\(\I_{p-k}-\D_{p-k}+o(\bar{\theta})\)\J_{AB}^\top\)^{-1}} \nonumber\\
&\quad= \Tr{\(\I_k + \theta_s\E_k - \J_{AB}\J_{AB}^\top +\J_{AB}\D_{p-k}\J_{AB}^\top + o(\bar{\theta})\J_{AB}\J_{AB}^\top\)^{-1}} \nonumber\\
&\quad= \Tr{\I_k - \theta_s\E_k + \J_{AB}\J_{AB}^\top + \theta_s^2\E_k^2 - \J_{AB}\D_{p-k}\J_{AB}^\top - \theta_s\[\J_{AB}\J_{AB}^\top\E_k+\E_k\J_{AB}\J_{AB}^\top\]- \theta_s^3\E_k^3} \nonumber \\
&\qquad+o(\bar{\theta}^3). \nonumber
\end{align}
Since the graph is $d$-regular,
\begin{equation}
\label{eq:d_matr_eq}
\norm{\D_{p-k}} \leqslant \bar{\theta}d.
\end{equation}
It is easy to check that
\begin{equation}
\Tr{\J_{AB}\J_{AB}^\top+\theta_s^2\E_k^2} = \sum_{\substack{i \in A, j \in V,\\ i \neq j}}\J_{ij}^2 =  \Tr{\J_{AV}\J_{AV}^\top} - k = \theta_s^2dk.
\end{equation}
Due to Assumption [A3] (equation (\ref{eq:assm2_eq})),
\begin{equation}
\Tr{\J_{AB}\J_{AB}^\top} \leqslant \theta_s^2d\sqrt{k},
\end{equation}
which together with (\ref{eq:d_matr_eq}) implies
\begin{equation}
\label{eq:trace_rel_aux}
\Tr{\J_{AB}\D_{p-k}\J_{AB}^\top} \leqslant \bar{\theta}\theta_s^2d^2\sqrt{k}.
\end{equation}
Recall that $\E_k$ is the adjacency matrix of a graph, and therefore has zero trace. Using relations (\ref{eq:d_matr_eq})-(\ref{eq:trace_rel_aux}), we conclude from (\ref{eq:trace_j_eq_chain}),
\begin{align}
\label{eq:q_param_est}
&\left|\Tr{\(\J_{A} - \J_{AB} \J_{B}^{-1} \J_{AB}^\top\)^{-1}} - k+\theta_s^2dk \right| \nonumber\\ 
&\qquad\qquad\qquad \leqslant \left|\Tr{\theta_s\E_k \[\J_{AB}\J_{AB}^\top + \theta_s^2\E_k^2\]+\theta_s\J_{AB}\J_{AB}^\top\E_k  + \J_{AB}\D_{p-k}\J_{AB}^\top+o(\bar{\theta}^3)}\right| \nonumber\\
&\qquad\qquad\qquad \leqslant \theta_s^3d^2k + \theta_s^3d^2\sqrt{k} + \bar{\theta}\theta_s^2d^2\sqrt{k} +o(\bar{\theta}^3).
\end{align}
The proof is concluded by dividing (\ref{eq:q_param_est}) by $\theta_s^2dk$.
\end{proof}

\begin{lemma}[Corollary 1.8 from \cite{guionnet2000concentration}]
\label{lem:sigma_concentr}
Let $\widehat{\bm\Sigma} = \frac{1}{n}\sum_{i=1}^n \x_i\x_i^\top$ be the sample covariance of $n$ i.i.d.\ copies of a $k$ dimensional $\x \sim \mathcal{N}(0,\bm\Sigma)$, where $\bm\Sigma$ is positive definite and has largest eigenvalue $\overline{\lambda}$, then for any $t>0$,
\begin{equation}
\label{eq:init_est_concentr21}
\mathbb{P}\[\left|\Tr{\widehat{\bm\Sigma}} - \Tr{\bm\Sigma}\right| \geqslant t\] \leqslant 2\exp\(-\frac{n^2t^2}{\overline{\lambda}}\)
\end{equation}
\end{lemma}

Lemma \ref{lem:init_est_concentr} is a direct corollary of this statement.

\begin{proof}[Proof of Lemma \ref{lem:init_est_concentr}]
Plug (\ref{eq:lim_err_est_1}) into (\ref{eq:init_est_concentr21}) from the previous lemma to get
\begin{equation}
\label{eq:init_est_concentr11}
\mathbb{P}\[\left|\hat{\theta}^2 - \(\theta^*\)^2\right| \geqslant t\] \leqslant 2\exp\(-\frac{(nkd)^2}{\overline{\lambda}}t^2\).
\end{equation}
Let us use (\ref{eq:norm_bound}) to bound the norm of the covariance,
\begin{equation}
\overline{\lambda} = \norm{\bm\Sigma_A} \leqslant \norm{\bm\Sigma} = \norm{\J^{-1}} \leqslant \frac{1}{1-d\bar{\theta}}.
\end{equation}
Now use the inequality
\begin{equation}
\left|\hat{\theta}^2 - \(\theta^*\)^2\right| = \left|\(\hat{\theta} - \theta^*\)\right|\(\hat{\theta} + \theta^*\) \geqslant 2\ubar{\theta} \left|\(\hat{\theta} - \theta^*\)\right|,
\end{equation}
to conclude the desired statement.
\end{proof}

\bibliographystyle{IEEEtran}
\bibliography{ilya_bib}

\begin{thebibliography}{10}
\providecommand{\url}[1]{#1}
\csname url@samestyle\endcsname
\providecommand{\newblock}{\relax}
\providecommand{\bibinfo}[2]{#2}
\providecommand{\BIBentrySTDinterwordspacing}{\spaceskip=0pt\relax}
\providecommand{\BIBentryALTinterwordstretchfactor}{4}
\providecommand{\BIBentryALTinterwordspacing}{\spaceskip=\fontdimen2\font plus
\BIBentryALTinterwordstretchfactor\fontdimen3\font minus
  \fontdimen4\font\relax}
\providecommand{\BIBforeignlanguage}[2]{{%
\expandafter\ifx\csname l@#1\endcsname\relax
\typeout{** WARNING: IEEEtran.bst: No hyphenation pattern has been}%
\typeout{** loaded for the language `#1'. Using the pattern for}%
\typeout{** the default language instead.}%
\else
\language=\csname l@#1\endcsname
\fi
#2}}
\providecommand{\BIBdecl}{\relax}
\BIBdecl

\bibitem{lauritzen1996graphical}
S.~L. Lauritzen, ``Graphical models,'' \emph{Oxford University Press}, 1996.

\bibitem{farasat2015probabilistic}
A.~Farasat, A.~Nikolaev, S.~N. Srihari, and R.~H. Blair, ``Probabilistic
  graphical models in modern social network analysis,'' \emph{Social Network
  Analysis and Mining}, vol.~5, no.~1, p.~62, 2015.

\bibitem{knorr1998modelling}
L.~Knorr-Held and J.~Besag, ``Modelling risk from a disease in time and
  space,'' \emph{Statistics in Medicine}, vol.~17, no.~18, pp. 2045--2060,
  1998.

\bibitem{karger2001learning}
D.~Karger and N.~Srebro, ``Learning {M}arkov networks: Maximum bounded
  tree-width graphs,'' \emph{Proceedings of the 12th ACM-SIAM Symposium on
  Discrete Algorithms}, pp. 392--401, 2001.

\bibitem{bogdanov2008complexity}
A.~Bogdanov, E.~Mossel, and S.~Vadhan, ``The complexity of distinguishing
  {M}arkov random fields,'' \emph{Lecture Notes in Computer Science}, vol.
  5171, pp. 331--342, 2008.

\bibitem{chow1968approximating}
C.~Chow and C.~Liu, ``Approximating discrete probability distributions with
  dependence trees,'' \emph{IEEE Transactions on Information Theory}, vol.~14,
  no.~3, pp. 462--467, 1968.

\bibitem{bresler2015efficiently}
G.~Bresler, ``Efficiently learning {I}sing models on arbitrary graphs,'' pp.
  771--782, 2015.

\bibitem{santhanam2012information}
N.~P. Santhanam and M.~J. Wainwright, ``Information-theoretic limits of
  selecting binary graphical models in high dimensions,'' \emph{IEEE
  Transactions on Information Theory}, vol.~58, no.~7, pp. 4117--4134, 2012.

\bibitem{anandkumar2012high}
A.~Anandkumar, V.~Y.~F. Tan, F.~Huang, and A.~S. Willsky, ``High-dimensional
  {G}aussian graphical model selection: Walk summability and local separation
  criterion,'' \emph{Journal of Machine Learning Research}, vol.~13, no. Aug,
  pp. 2293--2337, 2012.

\bibitem{bresler2008reconstruction}
G.~Bresler, E.~Mossel, and A.~Sly, ``Reconstruction of {M}arkov random fields
  from samples: Some observations and algorithms,'' \emph{Lecture Notes in
  Computer Science}, vol. 5171, pp. 343--356, 2008.

\bibitem{meinshausen2006high}
N.~Meinshausen and P.~B{\"u}hlmann, ``High-dimensional graphs and variable
  selection with the lasso,'' \emph{The Annals of Statistics}, pp. 1436--1462,
  2006.

\bibitem{ravikumar2011high}
P.~Ravikumar, M.~J. Wainwright, G.~Raskutti, and B.~Yu, ``High-dimensional
  covariance estimation by minimizing $\ell_1$-penalized log-determinant
  divergence,'' \emph{Electronic Journal of Statistics}, vol.~5, pp. 935--980,
  2011.

\bibitem{friedman2008sparse}
J.~Friedman, T.~Hastie, and R.~Tibshirani, ``Sparse inverse covariance
  estimation with the graphical lasso,'' \emph{Biostatistics}, vol.~9, no.~3,
  pp. 432--441, 2008.

\bibitem{yuan2007model}
M.~Yuan and Y.~Lin, ``Model selection and estimation in the {G}aussian
  graphical model,'' \emph{Biometrika}, vol.~94, no.~1, pp. 19--35, 2007.

\bibitem{ji1996consistent}
C.~Ji and L.~Seymour, ``A consistent model selection procedure for {M}arkov
  random fields based on penalized pseudolikelihood,'' \emph{The Annals of
  Applied Probability}, pp. 423--443, 1996.

\bibitem{ahrens2013whole}
M.~B. Ahrens, M.~B. Orger, D.~N. Robson, J.~M. Li, and P.~J. Keller,
  ``Whole-brain functional imaging at cellular resolution using light-sheet
  microscopy,'' \emph{Nature Methods}, vol.~10, no.~5, p. 413, 2013.

\bibitem{dunn2016brain}
T.~W. Dunn, Y.~Mu, S.~Narayan, O.~Randlett, E.~A. Naumann, C.-T. Yang, A.~F.
  Schier, J.~Freeman, F.~Engert, and M.~B. Ahrens, ``Brain-wide mapping of
  neural activity controlling zebrafish exploratory locomotion,'' \emph{Elife},
  vol.~5, 2016.

\bibitem{gustafsson2000adaptive}
F.~Gustafsson, ``Adaptive filtering and change detection,'' \emph{Wiley New
  York}, vol.~1, 2000.

\bibitem{aminikhanghahi2017survey}
S.~Aminikhanghahi and D.~J. Cook, ``A survey of methods for time series change
  point detection,'' \emph{Knowledge and Information Systems}, vol.~51, no.~2,
  pp. 339--367, 2017.

\bibitem{mccoy2014two}
B.~M. McCoy and T.~T. Wu, ``The two-dimensional {I}sing model,'' \emph{Courier
  Corporation}, 2014.

\bibitem{haralick1985image}
R.~M. Haralick and L.~G. Shapiro, ``Image segmentation techniques,''
  \emph{Computer Vision, Graphics, and Image Processing}, vol.~29, no.~1, pp.
  100--132, 1985.

\bibitem{fortunato2010community}
S.~Fortunato, ``Community detection in graphs,'' \emph{Physics reports}, vol.
  486, no.~3, pp. 75--174, 2010.

\bibitem{eckmann2002curvature}
J.-P. Eckmann and E.~Moses, ``Curvature of co-links uncovers hidden thematic
  layers in the world wide web,'' \emph{Proceedings of the National Academy of
  Sciences}, vol.~99, no.~9, pp. 5825--5829, 2002.

\bibitem{clauset2005finding}
A.~Clauset, ``Finding local community structure in networks,'' \emph{Physical
  Review E}, vol.~72, no.~2, p. 026132, 2005.

\bibitem{dasgupta1999learning}
S.~Dasgupta, ``Learning polytrees,'' \emph{Proceedings of the 15th Conference
  on Uncertainty in Artificial Intelligence}, pp. 134--141, 1999.

\bibitem{srebro2001maximum}
N.~Srebro, ``Maximum likelihood bounded tree-width {M}arkov networks,''
  \emph{Proceedings of the 17th Conference on Uncertainty in Artificial
  Intelligence}, pp. 504--511, 2001.

\bibitem{netrapalli2010greedy}
P.~Netrapalli, S.~Banerjee, S.~Sanghavi, and S.~Shakkottai, ``Greedy learning
  of {M}arkov network structure,'' \emph{Allerton Conference on Communication,
  Control, and Computing}, pp. 1295--1302, 2010.

\bibitem{montanari2009graphical}
A.~Montanari and J.~A. Pereira, ``Which graphical models are difficult to
  learn?'' \emph{Advances in Neural Information Processing Systems}, pp.
  1303--1311, 2009.

\bibitem{gamarnik2013correlation}
D.~Gamarnik, ``Correlation decay method for decision, optimization, and
  inference in large-scale networks,'' \emph{Theory Driven by Influential
  Applications}, pp. 108--121, 2013.

\bibitem{dobrushin1970prescribing}
R.~L. Dobrushin, ``Prescribing a system of random variables by conditional
  distributions,'' \emph{Theory of Probability and Its Applications}, vol.~15,
  no.~3, pp. 458--486, 1970.

\bibitem{malioutov2006walk}
D.~M. Malioutov, J.~K. Johnson, and A.~S. Willsky, ``Walk-sums and belief
  propagation in {G}aussian graphical models,'' \emph{Journal of Machine
  Learning Research}, vol.~7, no. Oct, pp. 2031--2064, 2006.

\bibitem{ravikumar2010high}
P.~Ravikumar, M.~J. Wainwright, and J.~D. Lafferty, ``High-dimensional {I}sing
  model selection using $\ell_1$-regularized logistic regression,'' \emph{The
  Annals of Statistics}, vol.~38, no.~3, pp. 1287--1319, 2010.

\bibitem{guttmann2009polygons}
A.~J. Guttmann, ``Polygons, polyominoes and polycubes,'' \emph{Springer}, vol.
  775, 2009.

\bibitem{soloveychik2018polygonal}
I.~Soloveychik and V.~Tarokh, ``Polygonal region detection in markov random
  fields,'' \emph{in progress}, 2018.

\bibitem{borrelli2003angular}
V.~Borrelli, F.~Cazals, and J.-M. Morvan, ``On the angular defect of
  triangulations and the pointwise approximation of curvatures,''
  \emph{Computer Aided Geometric Design}, vol.~20, no.~6, pp. 319--341, 2003.

\bibitem{penrose2003random}
M.~Penrose, ``Random geometric graphs,'' \emph{Oxford University Press}, no.~5,
  2003.

\bibitem{cover2012elements}
T.~M. Cover and J.~A. Thomas, ``Elements of information theory,'' \emph{John
  Wiley and Sons}, 2012.

\bibitem{zhang2006schur}
F.~Zhang, ``The {S}chur complement and its applications,'' \emph{Springer
  Science and Business Media}, vol.~4, 2006.

\bibitem{yu1997assouad}
B.~Yu, ``Assouad, {F}ano, and {L}e {C}am,'' \emph{Festschrift for Lucien Le
  Cam}, vol. 423, p. 435, 1997.

\bibitem{mckay1991asymptotic}
B.~D. McKay and N.~C. Wormald, ``Asymptotic enumeration by degree sequence of
  graphs with degrees $o(\sqrt{n})$,'' \emph{Combinatorica}, vol.~11, no.~4,
  pp. 369--382, 1991.

\bibitem{dembo2010large}
A.~Dembo and O.~Zeitouni, ``Large deviations techniques and applications,''
  \emph{Stochastic Modelling and Applied Probability}, vol.~38, 2010.

\bibitem{delest1988generating}
M.-P. Delest, ``Generating functions for column-convex polyominoes,''
  \emph{Journal of Combinatorial Theory, Series A}, vol.~48, no.~1, pp. 12--31,
  1988.

\bibitem{bousquet1992convex}
M.~Bousquet-M{\'e}lou, ``Convex polyominoes and heaps of segments,''
  \emph{Journal of Physics A: Mathematical and General}, vol.~25, no.~7, p.
  1925, 1992.

\bibitem{bousquet1996method}
------, ``A method for the enumeration of various classes of column-convex
  polygons,'' \emph{Discrete Mathematics}, vol. 154, no. 1-3, pp. 1--25, 1996.

\bibitem{bousquet1995generating}
M.~Bousquet-M{\'e}lou and J.-M. F{\'e}dou, ``The generating function of convex
  polyominoes: the resolution of a $q$-differential system,'' \emph{Discrete
  Mathematics}, vol. 137, no. 1-3, pp. 53--75, 1995.

\bibitem{soloveychik2018large}
I.~Soloveychik and V.~Tarokh, ``Large deviations of convex polyominoes,''
  \emph{arXiv:1802.03849}, 2018.

\bibitem{vershik1994limit}
A.~M. Vershik, ``The limit shape of convex lattice polygons and related
  topics,'' \emph{Functional Analysis and Its Applications}, vol.~28, no.~1,
  pp. 13--20, 1994.

\bibitem{vershik1999large}
A.~M. Vershik and O.~Zeitouni, ``Large deviations in the geometry of convex
  lattice polygons,'' \emph{Israel Journal of Mathematics}, vol. 109, no.~1,
  pp. 13--27, 1999.

\bibitem{blinovskii1999large}
V.~M. Blinovskii, ``Large deviation principle for the border of a random
  {Y}oung diagram,'' \emph{Problemy Peredachi Informatsii}, vol.~35, no.~1, pp.
  61--74, 1999.

\bibitem{dembo1998large}
A.~Dembo, O.~Zeitouni, and A.~M. Vershik, ``Large deviations for integer
  partitions,'' \emph{Technical Report}, 1998.

\bibitem{vershik1987statistical}
A.~M. Vershik, ``A statistical sum associated with {Y}oung diagrams,''
  \emph{Zapiski Nauchnykh Seminarov POMI, St. Petersburg Department of Steklov
  Institute of Mathematics, Russian Academy of Sciences}, vol. 164, pp. 20--29,
  1987.

\bibitem{petrov2009two}
F.~Petrov, ``Limit shapes of young diagrams. two elementary approaches,''
  \emph{Zapiski Nauchnykh Seminarov POMI, St. Petersburg Department of Steklov
  Institute of Mathematics, Russian Academy of Sciences}, vol. 370, pp.
  111--131, 2009.

\bibitem{guionnet2000concentration}
A.~Guionnet and O.~Zeitouni, ``Concentration of the spectral measure for large
  matrices,'' \emph{Electronic Communications in Probability}, vol.~5, pp.
  119--136, 2000.

\end{thebibliography}
\end{document}